%% file: arxiv.tex
\documentclass{zgroup/zgroup} %
\usepackage[utf8]{inputenc}
\usepackage{times}

\usepackage{algorithm,algorithmic}
\usepackage{amsmath,amsfonts,bm}
\usepackage{microtype}
\usepackage{bbm}
\usepackage{eqparbox}
\usepackage{adjustbox}
\renewcommand{\algorithmiccomment}[1]{\eqparbox{COMMENT}{// #1}}
\usepackage{graphicx, xcolor}
\usepackage{hyperref}
\usepackage{wrapfig}
\usepackage{url}
\usepackage{multicol,multirow}
\usepackage{xfrac}
\usepackage{float}
\usepackage{colortbl}
\usepackage{enumitem}
\usepackage{thmtools, thm-restate}

\title[]{Analyzing Monotonic Linear Interpolation\\in Neural Network Loss Landscapes}

\zgroupauthor{\Name{James Lucas} \Email{jlucas@cs.toronto.edu}\\
  \addr University of Toronto; Vector Institute\\
  \Name{Juhan Bae} \Email{jbae@cs.toronto.edu}\\
  \addr University of Toronto; Vector Institute\\
  \Name{Michael R. Zhang} \Email{michael@cs.toronto.edu}\\
  \addr University of Toronto; Vector Institute\\
  \Name{Stanislav Fort} \Email{sfort1@stanford.edu}\\
  \addr Stanford University\\
  \Name{Richard Zemel} \Email{zemel@cs.toronto.edu}\\
  \addr University of Toronto; Vector Institute; Canadian Institute for Advanced Research\\
  \Name{Roger Grosse} \Email{grosse@cs.toronto.edu}\\
  \addr University of Toronto; Vector Institute; Canadian Institute for Advanced Research\\
}

\input{macros}
\arxivtrue
\begin{document}
\maketitle
\input{sections/abstract}
\input{sections/introduction}

\input{sections/related}

\input{sections/mli_property}
\input{sections/experiments/experiments}

\input{sections/conclusion}

\bibliography{references}

\newpage
\appendix
\input{appendix/appendix.tex}

\end{document}

%% file: macros.tex
\newif\ifcomments
\newif\ifarxiv

\commentsfalse

\ifcomments
  \newcommand{\colornote}[3]{{\color{#1}\bf{#2: #3}\normalfont}}
\else
  \newcommand{\colornote}[3]{}
\fi

\newcommand{\ba}{\mathbf{a}}

\newcommand{\bv}{\mathbf{v}}
\newcommand{\bx}{\mathbf{x}}
\newcommand{\by}{\mathbf{y}}
\newcommand{\bz}{\mathbf{z}}
\newcommand{\btheta}{\boldsymbol{\theta}}

\newcommand{\vecop}[1]{\mathrm{vec}(#1)}

\newcommand{\calL}{\mathcal{L}}
\newcommand{\calN}{\mathcal{N}}

\newcommand{\bbR}{\mathbb{R}}
\newcommand{\bbE}{\mathbb{E}}
\newcommand{\bbN}{\mathbb{N}}

\DeclareMathOperator*{\argmin}{arg\,min}

%% file: sections/abstract.tex
\begin{abstract}
Linear interpolation between initial neural network parameters and converged parameters after training with stochastic gradient descent (SGD) typically leads to a monotonic decrease in the training objective. This Monotonic Linear Interpolation (MLI) property, first observed by~\citet{goodfellow2014qualitatively}, persists in spite of the non-convex objectives and highly non-linear training dynamics of neural networks. Extending this work, we evaluate several hypotheses for this property that, to our knowledge, have not yet been explored. Using tools from differential geometry, we draw connections between the interpolated paths in function space and the monotonicity of the network --- providing sufficient conditions for the MLI property under mean squared error. While the MLI property holds under various settings (e.g.~network architectures and learning problems), we show in practice that networks violating the MLI property can be produced systematically, by encouraging the weights to move far from initialization. The MLI property raises important questions about the loss landscape geometry of neural networks and highlights the need to further study their global properties.
\end{abstract}

%% file: sections/introduction.tex
\section{Introduction}
\label{sec:intro}

A simple and lightweight method to probe neural network loss landscapes is to linearly interpolate between the parameters at initialization and the parameters found after training. More formally, consider a neural network with parameters $\btheta \in \bbR^d$ trained with respect to loss function $\calL \colon \bbR^d \rightarrow \bbR$ on a dataset $\mathcal{D}$. Let the neural network be initialized with some parameters $\btheta_0$. Then, using a gradient descent optimizer, the network converges to some final parameters $\btheta_T$. A linear path is then constructed between these two parameters denoted $\btheta_{\alpha}=(1-\alpha)\btheta_{0}+\alpha\btheta_T$. A surprising phenomenon, first observed by \citet{goodfellow2014qualitatively}, is that the function $\calL(\btheta_{\alpha})$ typically monotonically decreases on the interval $\alpha \in [0,1]$. We call this effect the \emph{Monotonic Linear Interpolation (MLI) property} of neural networks.

The MLI property is illustrated in Figure~\ref{fig:title_figure_landscape}. The interpolated path ($\btheta_\alpha$) exhibits the MLI property as the training loss monotonically decreases along this line. Even more surprising, linear interpolation between an unrelated random initialization and the same converged parameters also satisfies the MLI property.

\citet{goodfellow2014qualitatively} showed that the MLI property persists on various architectures, activation functions, and training objectives in neural network training. They conclude their study by stating that ``the reason for the success of SGD on a wide variety of tasks is now clear: these tasks are relatively easy to optimize.'' In our work, we observe that networks violating the MLI property can be produced systematically and are also trained without significant difficulty. Moreover, since the publication of their research, there have been significant developments both in terms of the neural network architectures that we train today \citep{he2016deep, vaswani2017attention,huang2017densely} and our theoretical understanding of them \citep{amari2020does, jacot2018neural, draxler2018essentially, frankle2018lottery, fort2019emergent}. Hence, with a wider lens that addresses these developments, we believe that further investigation of this phenomenon is likely to yield new insights into neural network optimization and their loss landscapes.

\begin{figure}[!t]
    \center{\includegraphics[width=1.0\linewidth]
    {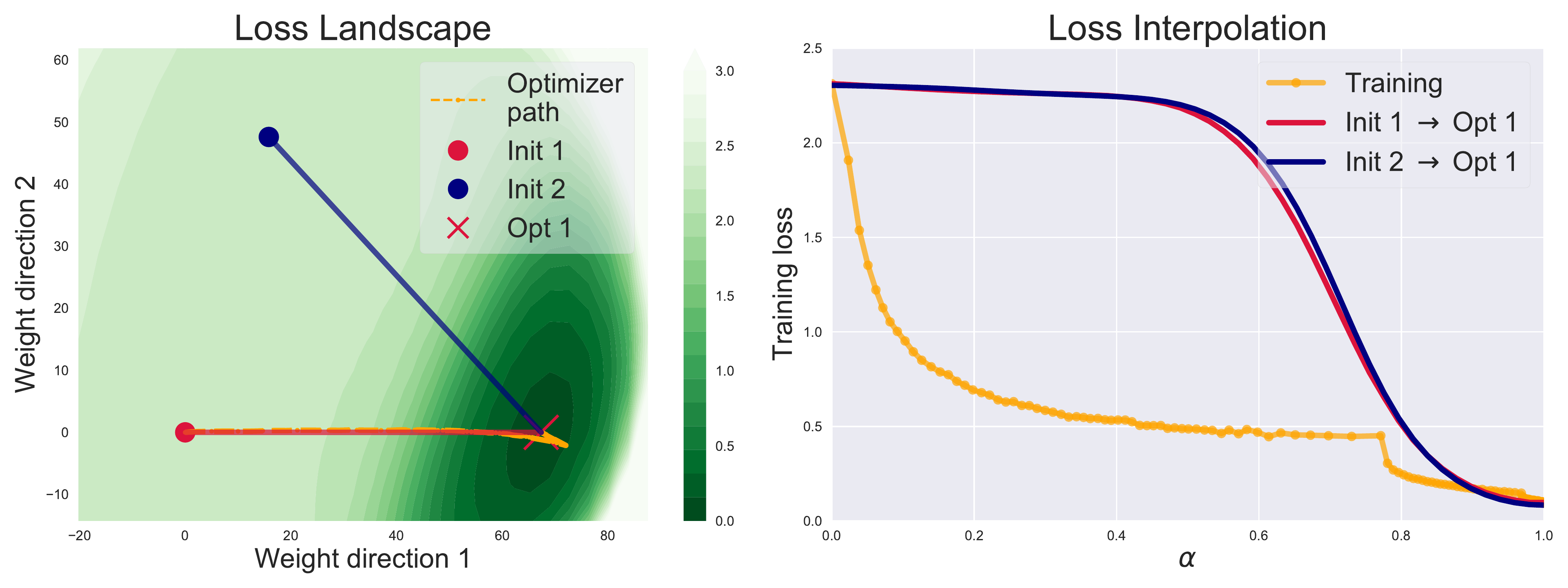}}
    \vspace{-0.6cm}
    \caption{Monotonic linear interpolation for a ResNet-20 trained on CIFAR-10 from initialization to an optimum (\textbf{\textcolor{red}{red}}) and from an unrelated initialization to the same optimum (\textbf{\textcolor{blue}{blue}}). On the left, we show a 2D slice of the loss landscape, defined by the two initializations and optimum, along with the optimization trajectory projected onto the plane (\textbf{\textcolor{orange}{orange}}). On the right, we show the interpolated loss curves, with training loss shown relative to the proportion of distance travelled to the optimum.}
    \label{fig:title_figure_landscape}
    \vspace{-0.6cm}
\end{figure}

We study three distinct questions surrounding the MLI property. 1) How persistent is the MLI property? 2) Why does the MLI property hold? 3) What does the MLI property tell us about the loss landscape of neural networks? To address these questions, we provide an expanded empirical and theoretical study of this phenomenon.

\begin{figure*}[!t]
    \begin{minipage}{0.24\linewidth}
    \centering
    \includegraphics[width=\linewidth]{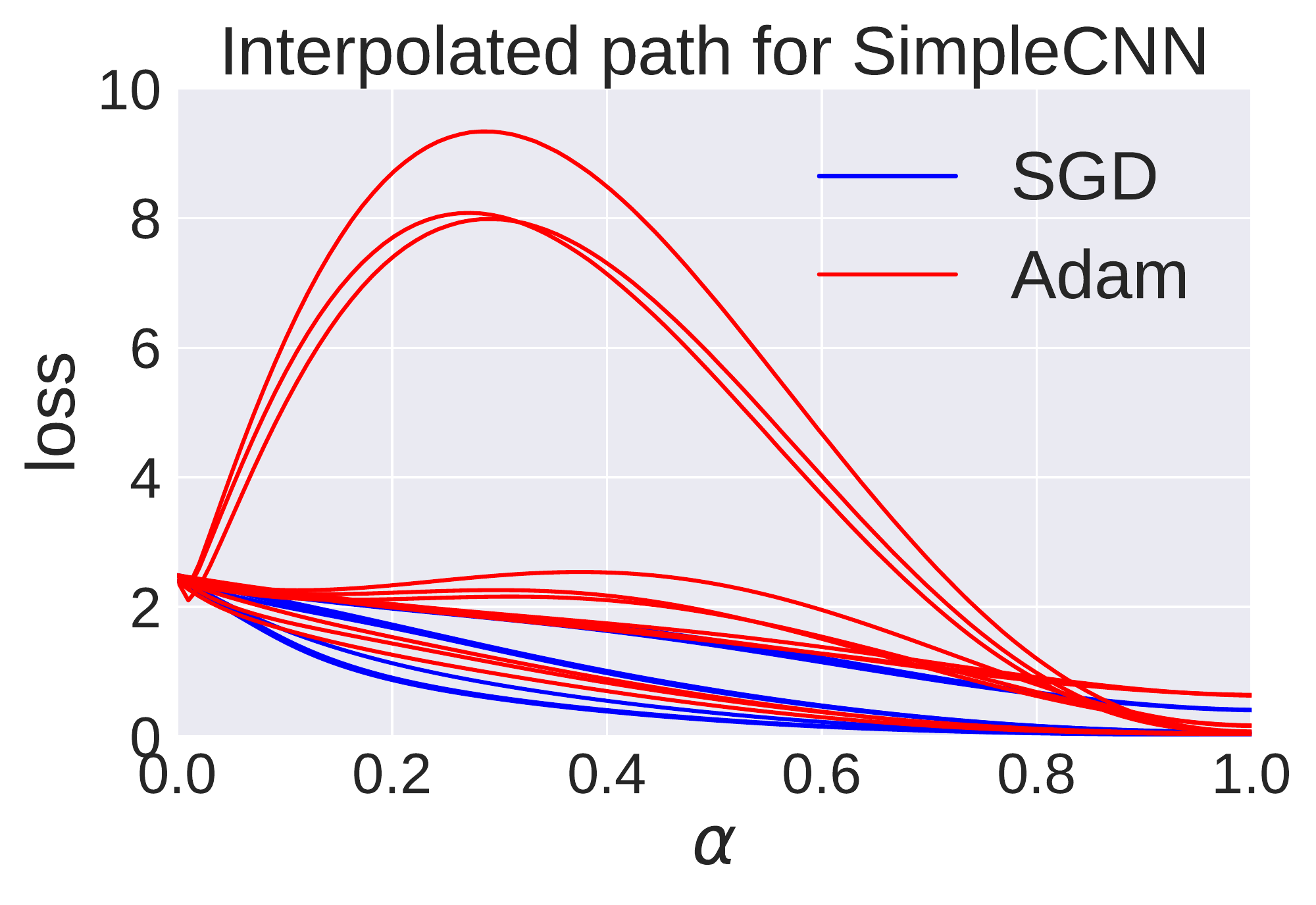}
    \end{minipage}\hfill%
    \begin{minipage}{0.24\linewidth}
    \centering
    \includegraphics[width=\linewidth]{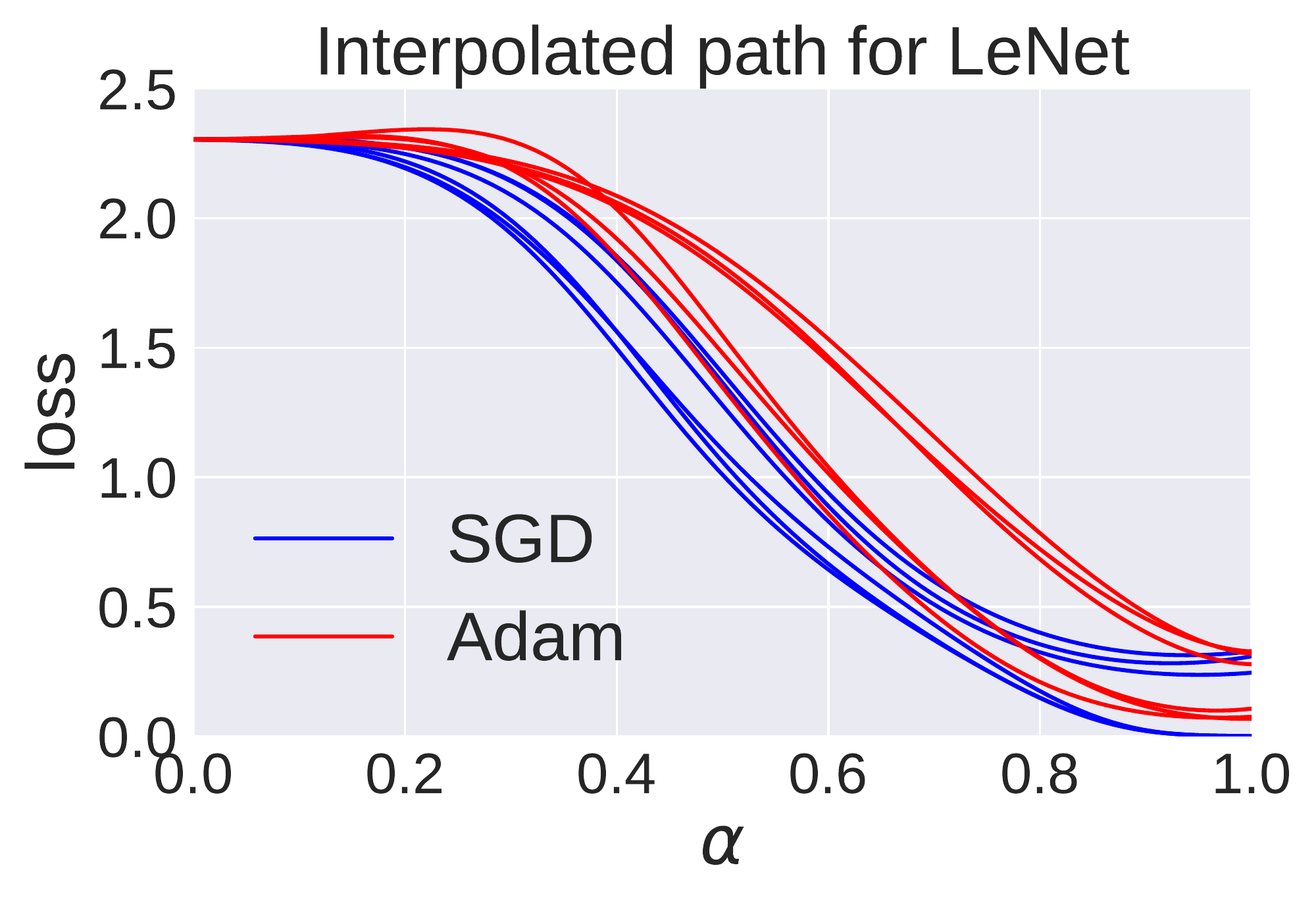}
    \end{minipage}\hfill%
    \begin{minipage}{0.24\linewidth}
    \centering
    \includegraphics[width=\linewidth]{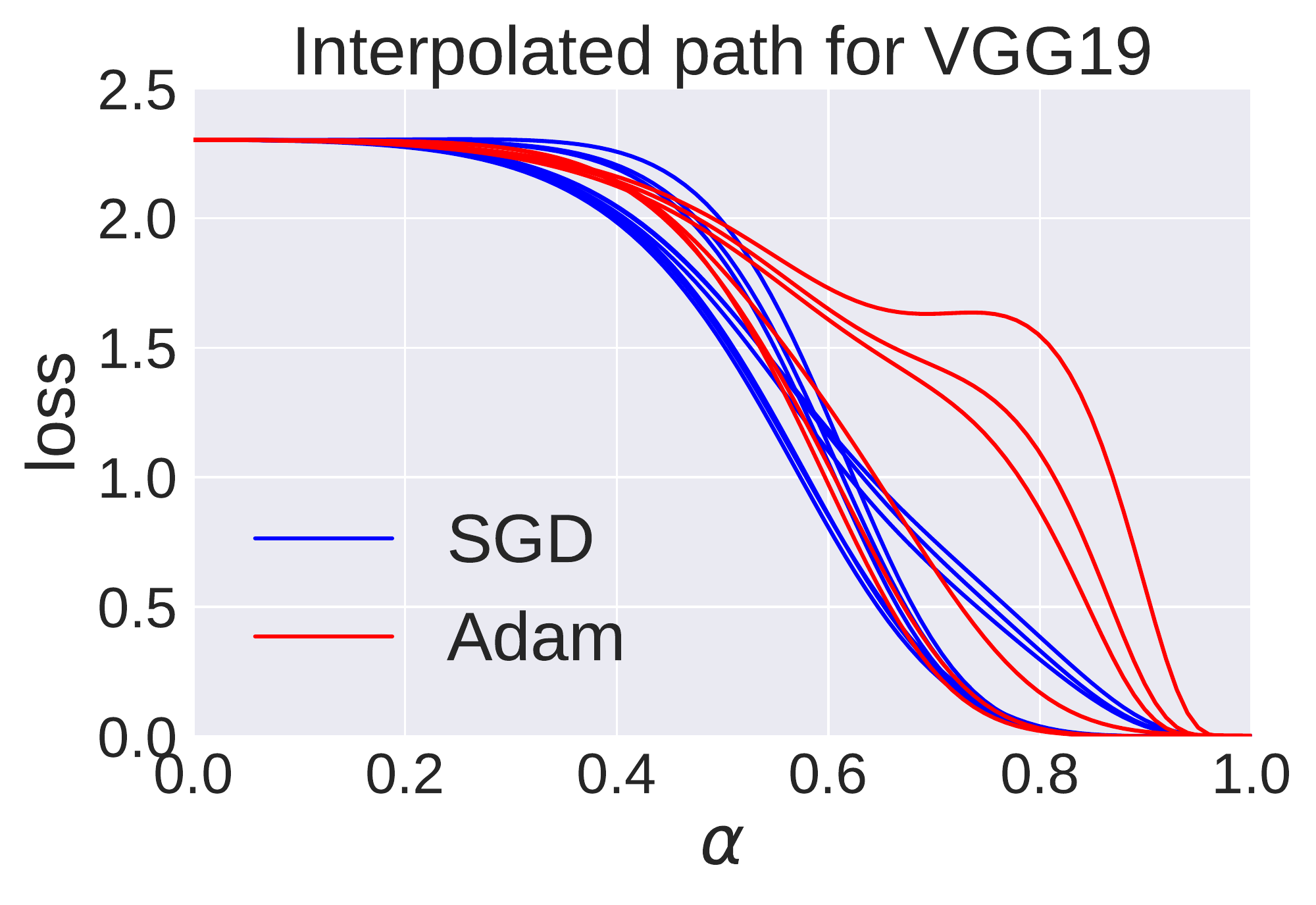}
    \end{minipage}
    \begin{minipage}{0.24\linewidth}
    \centering
    \includegraphics[width=\linewidth]{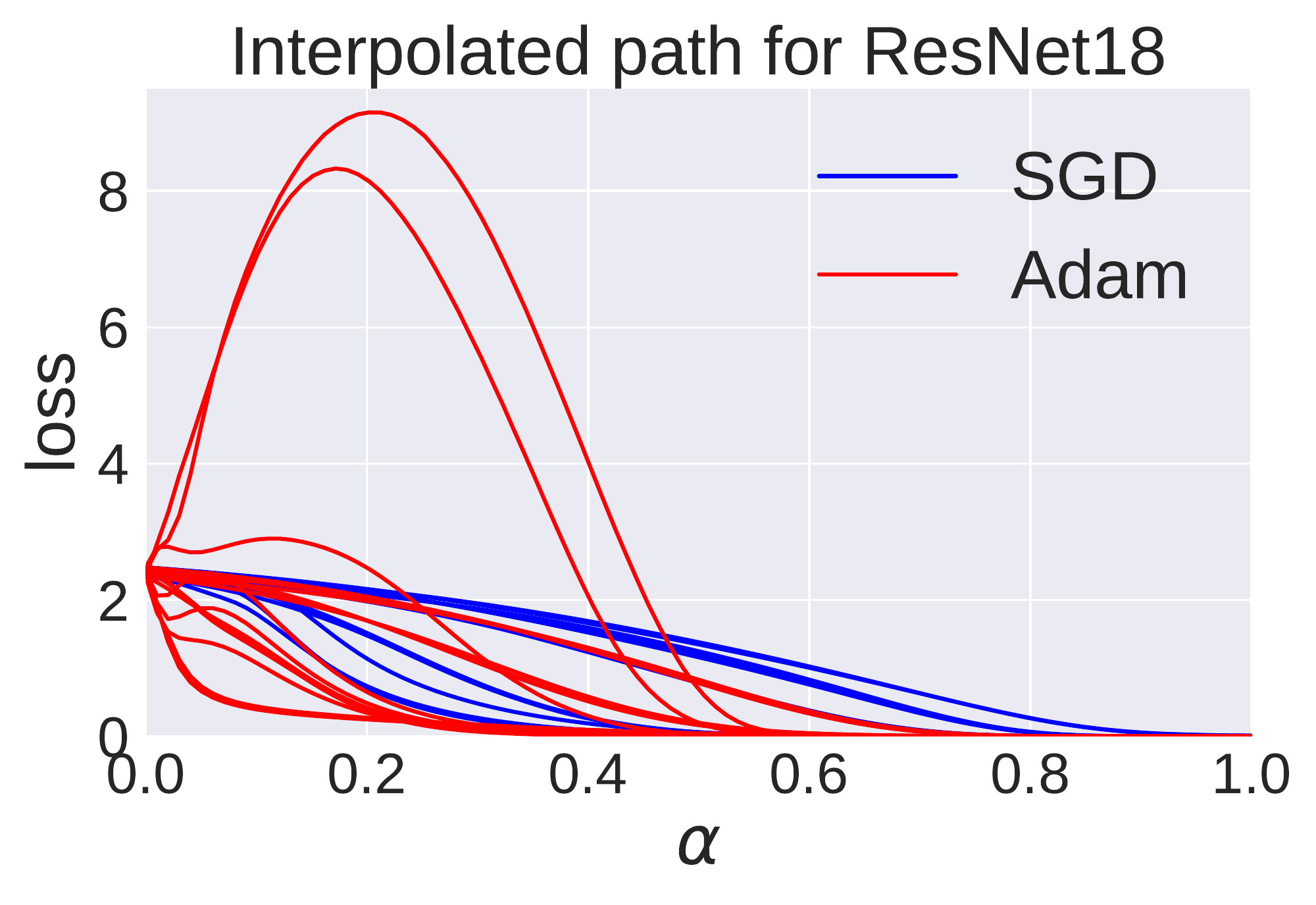}
    \end{minipage}
    \vspace{-0.1cm}
    \caption{Training loss over the linear interpolation connecting initial and final parameters. Each curve represents a network trained on CIFAR-10 with different hyperparameter configurations (achieving at least 1.0 training loss). The MLI property holds for networks trained with SGD, but often fails for networks trained with Adam.}
    \vspace{-0.4cm}
    \label{fig:cnn_vary_opt}
\end{figure*}

To evaluate the persistence of the MLI property, we train neural networks with varying architectures, optimizers, datasets, initialization methods, and training mechanisms (e.g.~batch normalization~\citep{ioffe2015batch}). We find that the MLI property persists for the majority of these settings but can be consistently broken through mechanisms that encourage the weights to move far from initialization. As far as we know, ours is the first work to observe that MLI is not a stable property of the network architecture.

One hypothesis for the MLI property is that the networks are close to linear along the interpolation path. We formalize this notion using tools from differential geometry and provide sufficient conditions for neural networks trained under the MSE loss to satisfy the MLI property. In particular, we prove that if the length under the Gauss map (which we refer to as the \emph{Gauss length}) of the interpolation trajectory in function space is small, then the network is guaranteed to have the MLI property. While the converse does not hold in general, we show that this quantity is correlated with monotonicity in practice. We connect this explanation to our prior observation that large distances moved in weight space encourage non-monotonic interpolations through a surprising power-law relationship between the distance moved and the average Gauss length.

Finally, we investigate the loss landscape of the neural networks we trained by evaluating the MLI property over alternative linear paths. For example, we examine the interpolation path connecting different initializations and final parameters (as in Figure~\ref{fig:title_figure_landscape}). Surprisingly, when the MLI property holds for an initialization $\to$ final solution pair, the MLI property also holds for unrelated initializations to the same solution.

In summary, our primary contributions include:
\begin{itemize}
    \item We prove a sufficient condition for neural networks minimizing MSE to satisfy the MLI property.
    \item We show that the MLI property does not always hold and that we can systematically control for/against it.
    \item We identify several common training mechanisms that provide this control and connect them to our novel theoretical results.
\end{itemize}

%% file: sections/related.tex
\section{Related Work}
\label{sec:related-work}

\paragraph{Monotonic linear interpolation.}~\citet{goodfellow2014qualitatively} were the first to observe that the MLI property persists on various architectures, activation functions, and training objectives in deep learning. In addition to their empirical evaluation, they provided a qualitative analysis of the MLI property in a toy model where they argued that the MLI property holds despite negative curvature about initialization and disconnected optima. Concurrent research \citep{frankle2020revisiting} extends the original work of \citet{goodfellow2014qualitatively} with evaluations on modern architectures trained with SGD.

In this work, we provide an expanded study of the MLI property. We first investigate the persistence of the MLI property on various tasks, including settings with modern architectures and techniques that were not invented at the time of the original investigation. Further, we show that despite the original work's claim, we can train networks that violate the MLI property without significant training difficulty. Our experiments yield new insights into neural networks' loss landscapes and uncover aspects of neural network training that correlate with the MLI property.

\vspace{-0.3cm}\paragraph{Linear connectivity.} This work is connected to empirical and theoretical advancements in understanding the loss landscape of neural networks. Much of this recent work has involved characterizing mode connectivity of neural networks. In general, linear paths between modes cross regions of high loss~\citep{goodfellow2014qualitatively}. However, \citet{garipov2018loss,draxler2018essentially} show that local minima found by stochastic gradient descent (SGD) can be connected via piecewise linear paths. \citet{frankle2019linear} further show that linearly connected solutions may be found if networks share the same initialization. \citet{fort2020deep} demonstrate the connection between linear connectivity and the advantage nonlinear networks enjoy over their linearized version. \citet{kuditipudi2019explaining} posit \textit{dropout stability} as one possible explanation for mode connectivity, with \citet{shevchenko2019landscape} extending these result to show that the loss landscape becomes increasingly connected and more dropout stable with increasing network depth. Finally, \citet{nguyen2019connected} shows that every sublevel set of an overparameterized network is connected, implying that all global minima are connected.

Note that the MLI property we study is distinct from mode connectivity, where paths are drawn between different final solutions instead of initialization $\to$ solution pairs. As far as we are aware, no prior work has explored connections between the MLI property and mode connectivity. This would make for exciting future work.

\paragraph{Loss landscape geometry.} Recent analysis argues that there exists a small subspace at initialization in which the network converges \citep{gur2018gradient, fort2019emergent, papyan2020traces}. \citet{li2018measuring} show that some of these spaces can be identified by learning in a random affine subspace of low dimension. \citet{fort2019goldilocks} show that the success of these random spaces is related to the \emph{Goldilocks zone} that depends on the Hessian at initialization. In a loose sense, the MLI can be considered a special case of these results, wherein a 1D space is sufficient for training to succeed. However, this is not the only mechanism in which neural network training can succeed --- the solutions that violate the MLI property can have good generalization capability and are found without difficulty.

It has long been argued that flatter minima lead to better generalization~\citep{hochreiter1997flat} with some caveats~\citep{dinh2017sharp}. Recent work has shown that (full-batch) gradient descent with a large learning rate is able to find flatter minima by overcoming regions of initial high curvature~\citep{lewkowycz2020large}. Intuitively, gradient descent breaks out of one locally convex region of the space and into another --- suggesting that a barrier in the loss landscape has been surpassed. In this paper, we show that training with larger learning rates can lead to failure of the MLI property. And in doing so, identify a high loss barrier between the initial and converged parameters. Moreover, we show that these barriers do not appear when training with smaller learning rates.

\paragraph{Neural tangent kernel.} Recent research has shown that over-parameterized networks appreciate faster and, in some cases, more linear learning dynamics~\citep{lee2019wide,matthews2018gaussian}. The Neural Tangent Kernel (NTK)~\citep{jacot2018neural} describes the learning dynamics of neural networks in their function space. Existing work argues that the NTK is near-constant in the infinite width setting \citep{sun2019optimization}, however recent work challenges this view in general~\citep{liu2020linearity}. \citet{fort2020deep} recently showed that the NTK evolves quickly early on during training but the rate of change decreases dramatically during training. In Appendix~\ref{app:wide_nets}, we draw connections between the NTK literature and the MLI property and show that sufficiently wide fully-connected networks exhibit the MLI property with high probability.

\paragraph{Optimization algorithms.} In this work, we investigate the role that optimization algorithms have on the MLI property (and thus the explored loss landscape more generally). \citet{amari2020does} recently showed that for linear regression, natural gradient descent~\citep{amari1998natural} travels further in parameter space, as measured by Euclidean distance, compared to gradient descent. We verify this claim empirically for larger networks trained with adaptive optimizers and observe that this co-occurs with non-monotonicity along the interpolating path $\btheta_\alpha$.

%% file: sections/mli_property.tex
\section{The Monotonic Linear Interpolation Property}
\label{sec:mli_property}

The Monotonic Linear Interpolation (MLI) property states that when a network is randomly initialized and then trained to convergence, the linear path connecting the initialization and converged solution is monotonically decreasing in the training loss. Specifically, we say that a network has the MLI property if, for all $\alpha_1, \alpha_2 \in [0,1]$ with $\alpha_1 < \alpha_2$,
\begin{equation}
\calL(\btheta_{\alpha_1}) \geq \calL(\btheta_{\alpha_2}),\textrm{ where } \btheta_\alpha = \btheta_0 + \alpha (\btheta_T - \btheta_0).
\end{equation}
Here, $\btheta_0$ denotes the parameters at initialization and $\btheta_T$ denotes the parameters at convergence.

\subsection{$\Delta$-Monotonicity}
\citet{goodfellow2014qualitatively} found that the MLI property holds for a wide range of neural network architectures and learning problems. They provided primarily qualitative evidence of this fact by plotting $\calL(\btheta_\alpha)$ with discretizations of $[0,1]$ using varying resolutions. We instead propose a simple quantitative measure of non-monotonicity.

\begin{definition}{($\Delta$-monotonicity)}\label{def:delta_mono}
Consider a linear parameter interpolation parameterized by $\alpha, \btheta_0$, and $\btheta_T$ with corresponding loss function $\calL$. The path is $\Delta$-monotonic for $\Delta \geq 0$ if for all $\alpha_1, \alpha_2 \in [0,1]$ with $\alpha_1 < \alpha_2$, we have $\calL(\alpha_2) - \calL(\alpha_1) < \Delta$. 
\end{definition}

Intuitively, the above definition states that any bump due to increasing loss over the interpolation path should have a height upper-bounded by $\Delta$. We are interested in the smallest $\Delta \geq 0$ for which this definition holds. Notably, this minimum $\Delta$ can be approximated well numerically by stepping along the interpolation path in fixed intervals to find $\alpha_1$ and $\alpha_2$ giving the largest positive gap $\calL(\alpha_2) - \calL(\alpha_1)$. 

\subsection{Weight-space perspective}

It is natural to attempt to reason about the MLI property in terms of the parameters of the neural network. Intuitively, the MLI property suggests that, during optimization, the parameters move into a nearby basin of low loss without encountering any high-loss barriers in their path.

We can formalize this intuition for ``Lazy Training'' \citep{chizat2018lazy}, where the weights find a minimum near their initial value. Consider the second-order Taylor series expansion about the converged minimum $\btheta^*$,
\begin{equation}
    \calL(\btheta_0) \approx \calL(\btheta^*) + (\btheta_0 - \btheta^*)^\top \nabla_{\btheta}^2 \calL(\btheta^*)(\btheta_0 - \btheta^*).
\end{equation}
If the difference between the initial and converged parameters, $\Vert \btheta_0 - \btheta^* \Vert$, is sufficiently small, then this quadratic approximation holds well throughout the linear interpolation. In this case, the linear interpolation yields a monotonic decrease in the loss (Lemma~\ref{lemma:linear_convex_mli}, Appendix~\ref{app:additional_theory}).

Experimentally, we investigate the connection between the distance moved in weight space and the monotonicity of the resulting interpolation. We find that networks that move further in weight space during training are significantly more likely to produce non-monotonic initialization$\rightarrow$optimum interpolations. Theoretically, we investigate the MLI property for wide neural networks where lazy training occurs provably \citep{lee2019wide}. In this setting, we prove that the MLI property holds with high probability for networks of sufficient width (Theorem~\ref{thm:inf_width_mli}, Appendix~\ref{app:additional_theory}).

\subsection{Function-space perspective}

We typically train neural networks with a convex loss function applied to the network's output. While the parameter space of neural networks is extremely high-dimensional and exhibits symmetries, the function space is generally simpler and easier to reason about \citep{jacot2018neural}. To that end, we let
\begin{equation}\label{eqn:logit_interpolation}
    \bz(\alpha; \bx) = f(\bx; \btheta_\alpha) \in \bbR^k,~~\alpha \in [0,1]
\end{equation}
denote the \emph{logit interpolation} of a neural network $f$ evaluated on data point $\bx$ with parameters $\btheta_\alpha=\btheta_0 + \alpha(\btheta_T - \btheta_0)$.

One special case that guarantees the MLI property is that of linear functions, $f(\bx; \btheta) = \btheta^\top \bx$ (with $\calL(\btheta_0) > \calL(\btheta_T)$). In this case, the logit interpolations are also linear and, under a convex loss function, $f$ will satisfy the MLI property \citep{boyd2004convex}. In practice, we work with non-linear neural networks that have non-linear logit interpolations. However, we observed that the logit interpolations are often close to linear (in a sense that we formalize soon) and that this coincides with the MLI property (Figure~\ref{fig:784_logit_viz}). Therefore, we raise the question: Can we guarantee the MLI property for logit interpolations that are \emph{close} to linear?

\begin{figure}[!hpt]
\centering
\ifarxiv
\includegraphics[width=0.8\linewidth]{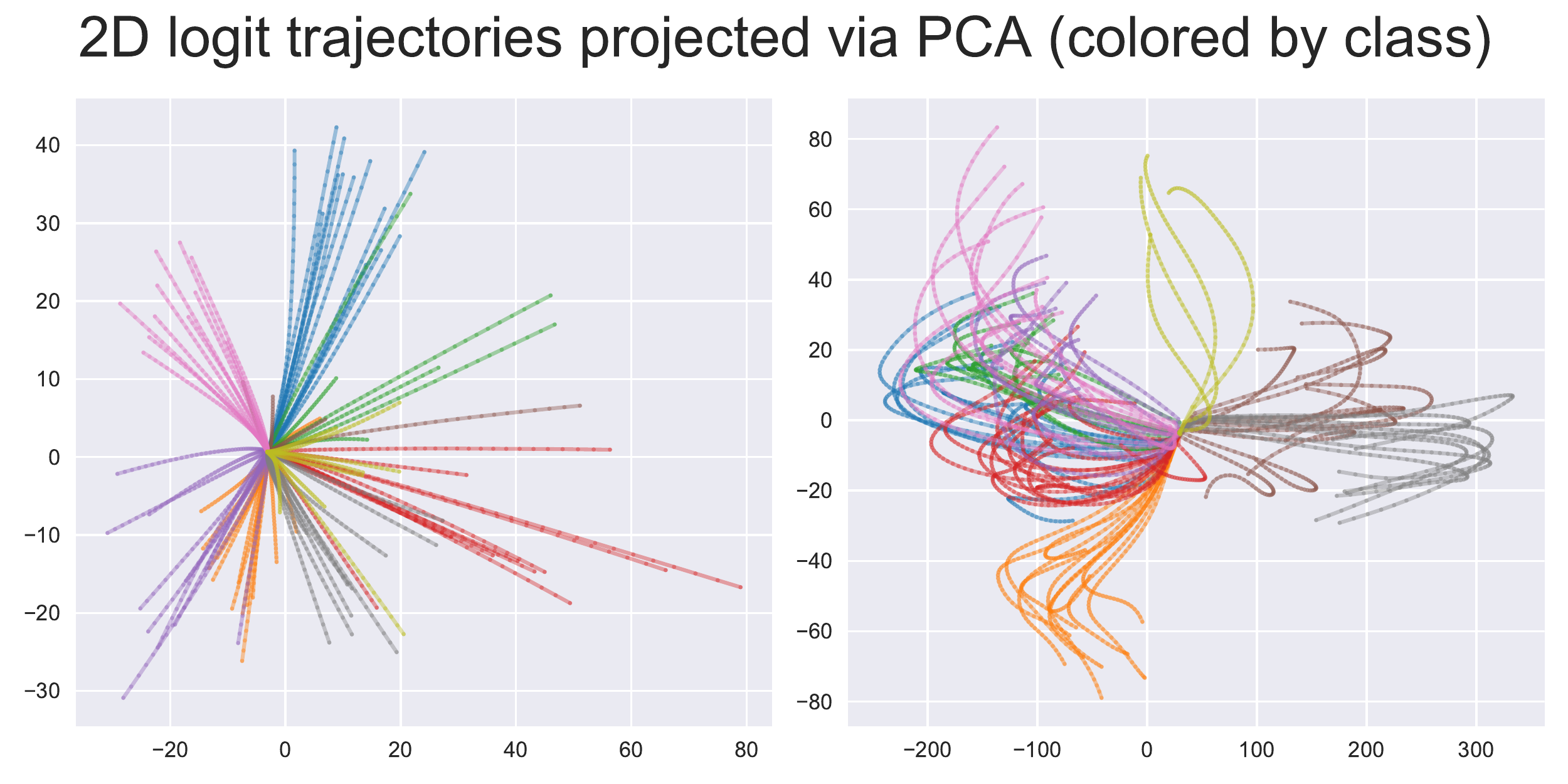}
\else
\includegraphics[width=\linewidth]{figures/784_search/2D_PCA_logit_paths.pdf}
\fi
\vspace{-0.5cm}
\caption{2D projections (computed with PCA) of logit interpolations for fully-connected networks trained on Fashion-MNIST. Both networks achieve near-perfect final training accuracy. However, the first one (left) interpolates monotonically while the second one (right) does not. The only difference between these two networks is that the second was trained using batch normalization while the first was not.}
\vspace{-0.5cm}
\label{fig:784_logit_viz}
\end{figure}

\paragraph{Measuring logit linearity.} There is no standard method to measure the linearity of a curve, but there are several tools from differential geometry that are applicable. In this work, we focus on the length under the Gauss map, which we refer to as the \emph{Gauss length}, a unit-free measure that is related to the curvature. In the case of curves, the Gauss length is computed by mapping the normalized tangent vectors of the curve onto the corresponding projective space (through the so-called Gauss map), and then measuring the length of the curve in this space. This is described formally in the following definition.
\begin{restatable}[Gauss length]{definition}{gausslength}\label{def:gausslength}
Given a curve $\bz: (0,1) \rightarrow \bbR^d$. Let $\hat{\bv}(\alpha) = \frac{\partial\bz}{\partial\alpha} / \Vert\frac{\partial\bz}{\partial\alpha}\Vert_2$ denote the normalized tangent vectors. The length under the Gauss map (Gauss length) is given by:
\[\int_{0}^{1} \sqrt{\langle \partial_\alpha \hat{\bv}(\alpha), \partial_\alpha \hat{\bv}(\alpha)  \rangle} d\alpha,\]
where $\partial_\alpha \hat{\bv}(\alpha)$ denotes the pushforward of the Gauss map acting on the acceleration vector.
\end{restatable}

We refer readers to \citet{lee2006riemannian} or \citet{poole2016exponential} for a more thorough introduction to these concepts. Intuitively, the Gauss length measures how much the curve bends along its path, with a Gauss length of zero indicating a linear path. In Theorem~\ref{thm:small_gauss_mse_mono}, we prove that a sufficiently small Gauss length guarantees the MLI property for MSE loss.

\begin{restatable}[Small Gauss length gives monotonicity]{theorem}{smallgaussmono}\label{thm:small_gauss_mse_mono}
Let $\calL(\bz) = \Vert \bz - \bz^* \Vert_2^2$ for $\bz^* \in \bbR^d$, and let $\bz: (0,1) \rightarrow \bbR^d$ be a smooth curve in $\bbR^d$ with $\bz(1) = \bz^*$ and $\calL(\bz(0)) > 0$. If the Gauss length of $\bz$ is less than $\pi / 2$, then $\calL \circ \bz(\alpha)$ is monotonically decreasing in $\alpha$.
\end{restatable}

See Appendix~\ref{app:gauss-length} for the proof. Informally, this theorem can be understood through a simple physical analogy. Imagine that you are standing on the inside surface of a uniform bowl and wish to increase your height before reaching the bottom. To do so, you must walk at an angle that is at least $\sfrac{\pi}{2}$ relative to the line connecting you to the bottom. Now, the smallest total rotation that guarantees your return to the bottom is at least $\sfrac{\pi}{2}$ radians.

Importantly, Theorem~\ref{thm:small_gauss_mse_mono} applies to arbitrary smooth curves including those produced in the function space of neural networks when we interpolate in the weight space ($\bz(\alpha; \bx)$ above). As an application of Theorem~\ref{thm:small_gauss_mse_mono}, in Appendix~\ref{app:two_layer_linear}, we give sufficient conditions for the MLI property to hold for two-layer linear models (whose loss landscape is non-convex with disconnected globally optimal manifolds \citep{pmlr-v97-kunin19a}). Furthermore, we prove that these sufficient conditions hold almost surely for models satisfying the \emph{tabula rasa} assumptions of \citet{saxe2019mathematical}.

One notable departure from the theory in our experiments is that we consider the average loss over the dataset. In this case, individual logit trajectories may be non-monotonic while the network satisfies the MLI property. Nonetheless, we find the average Gauss length to be a good indicator for the monotonicity of the network as a whole.

%% file: sections/experiments/experiments.tex
\section{Exploring \& Explaining the MLI Property}
\label{section:experiments}

In this section, we present our empirical investigation of the following questions:  1) How persistent is the MLI property? 2) Why does the MLI property hold? 3) What does the MLI property tell us about the loss landscape of neural networks?

For all experiments, unless specified otherwise, we discretize $\alpha$ in the interval $[0, 1]$ using 50 uniform steps. Here we report statistics from the training set throughout but note that the same observations hold for held-out datasets. Many additional results can be found in Appendix~\ref{app:experiments}.

\paragraph{A note on batch normalization.} We experiment with networks that use batch normalization during training. These networks require additional care when interpolating network parameters as the running statistics will not align with the activation statistics during interpolation. Therefore, we opt to reset and \emph{warm up} the running statistics for each interpolated set of parameters. This warm-up consists of computing the activation statistics over an epoch of the training data, meaning that each interpolation curve requires an additional 50 epochs (the number of discretizations of $\alpha$) of data consumption to get accurate loss/accuracy estimates. Note that the learned affine transformation is interpolated as usual.

\paragraph{Experiment settings.}  We summarize the main settings here with full details of our experimental procedure given in Appendix~\ref{app:exp_details}. We trained neural networks for reconstruction, classification, and language modeling. For the reconstruction tasks, we trained fully-connected deep autoencoders on MNIST~\citep{lecun2010mnist}. For the classification tasks, we trained networks on MNIST, Fashion-MNIST~\citep{xiao2017fashion}, CIFAR-10, and CIFAR-100~\citep{krizhevsky2009learning}. On these datasets, we explored fully-connected networks, convolutional networks, and residual architectures \citep{he2016deep}. In the above cases, we provide substantial exploration over varying architectures and optimization. We also provide a short study on the language modeling setting by training RoBERTa~\citep{liu2019roberta} on the Esperanto~\citep{conneau2019unsupervised} dataset. There, we verify the MLI property and visualize the loss landscape.

\input{sections/experiments/how}
\input{sections/experiments/why}

\input{sections/experiments/what}

%% file: sections/experiments/how.tex
\subsection{How persistent is the MLI property?}
\label{sec:exp:how_persistent}

We first investigate the persistence of the MLI property. \citet{goodfellow2014qualitatively} showed that the MLI property persists in classification and language modeling tasks (with LSTMs \citep{hochreiter1997long}) when trained with SGD. However, several modern advances in neural network training remain unaddressed and the limits of the MLI property have not been characterized. We provide a secondary investigation of the MLI property on reconstruction, classification, and language modelling tasks using modern architectures and methods.

In summary, we found that the MLI property is persistent over most standard neural network training knobs, including (but not limited to): layer width, depth, activation function, initialization method and regularization. However, there were three mechanisms through which we regularly observed the failure of the MLI property: the use of large learning rates, the use of adaptive optimizers such as Adam~\citep{kingma2014adam}, and the use of batch normalization~\citep{ioffe2015batch}. For the remainder of this section, we focus on the effect of these mechanisms but refer readers to Appendix~\ref{app:experiments} for a wider view of our study. We defer further analysis of explanations for the MLI property to Section~\ref{sec:exp:why_mli}.

\paragraph{Using large learning rates.} We found throughout that large learning rates were necessary to train networks that violated the MLI property. However, large learning rates alone were not always sufficient. In Table~\ref{tab:mnist_lr}, we show the proportion of networks with non-monotonic interpolations over varying learning rate (including only those models that achieved better than 0.1 training loss). Models trained with SGD using smaller learning rates always exhibited the MLI property. On the other hand, models trained with SGD with larger learning rates often violated the MLI property. For example, 71\% of the configurations with a learning rate of 1.0 were found to be non-monotonic. One hypothesis for this behaviour is due to the so-called catapult phase \citep{lewkowycz2020large, Jastrzebski2020The}, where large learning rates encourage the parameters to overcome a barrier in the loss landscape. Additional results on the effect of using larger learning rates can be found in Appendix~\ref{app:experiments_lr}.

\begin{table*}[]
    \centering
    \small
    \begin{tabular}{|l|r|l l l l l l l l|}\hline
        & LR: &  0.001 & 0.003 & 0.01 & 0.03 & 0.1 & 0.3 & 1.0 & 3.0\\ \hline
\multirow{2}{*}{\rotatebox[origin=c]{90}{SGD}} & BN & 0.00 (20) & 0.00 (24) & 0.00 (24) & 0.00 (24) & 0.00 (24) & 0.17 (24) & 0.83 (24) & 1.00 (16)\rule{0pt}{2.6ex}\rule[-1.2ex]{0pt}{0pt}\\
 & No BN & 0.00 (4) & 0.00 (8) & 0.00 (12) & 0.00 (20) & 0.20 (20) & 0.00 (12) & 0.00 (4) & 0.00 (4)\rule[-1.2ex]{0pt}{0pt}\\ \hline
\multirow{2}{*}{\rotatebox[origin=c]{90}{Adam}} & BN & 0.17 (24) & 0.68 (22) & 0.83 (24) & 1.00 (24) & 1.00 (16) & 1.00 (16) & 1.00 (4) & -\rule{0pt}{2.6ex}\rule[-1.2ex]{0pt}{0pt}\\
 & No BN & 0.00 (24) & 0.20 (20) & 0.00 (12) & 0.00 (4) & - & - & - & -\rule[-1.2ex]{0pt}{0pt}\\ \hline
    \end{tabular}
    \caption{Proportion of trained MNIST \& Fashion-MNIST classifiers (achieving better than 0.1 training loss) that had non-monotonic interpolations from initialization to final solution. The total number of runs with less than 0.1 training loss is displayed in parentheses next to the proportion. A dashed line indicates that no networks achieved 0.1 loss.}
    \label{tab:mnist_lr}
\end{table*}

\subsubsection{The effect of adaptive optimizers}

Prior work has only investigated the MLI property when training with SGD. To address this gap, we trained a wide variety of networks with adaptive optimizers (RMSProp~\citep{hinton2012neural}, Adam~\citep{kingma2014adam}, and K-FAC~\citep{martens2015optimizing}). Across all settings, we found that adaptive optimizers with large learning rates frequently led to models violating the MLI property.

\paragraph{MNIST autoencoders.} For image reconstruction, we evaluated the MLI property for deep fully-connected autoencoders trained on MNIST. We trained autoencoders with SGD and Adam, with varying learning rates and with a varying number of hidden layer size.
\begin{figure}[!htb]
    \centering
    \includegraphics[width=\linewidth]{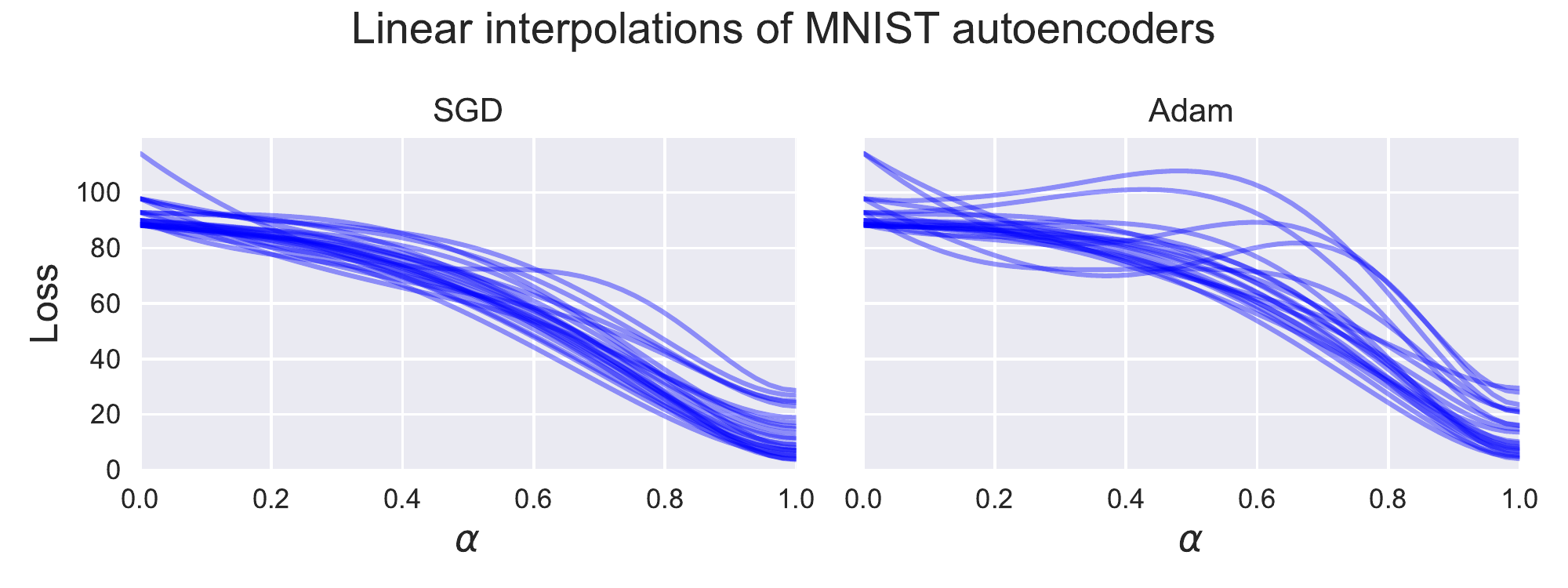}
    \vspace{-0.6cm}
    \caption{Training loss over linear interpolation of deep autoencoders trained on MNIST using SGD and Adam. Each interpolation line is for a training configuration with different hyperparameters (achieving better than 30 training loss).}
    \label{fig:ae_interpolations}
    \vspace{-0.5cm}
\end{figure}

In Figure~\ref{fig:ae_interpolations}, we show the training loss over the interpolated path for autoencoders with final loss (MSE) lower than 30. The majority of the networks trained with SGD retained the MLI property (with few failures at large learning rates). However, when trained with the Adam optimizer, a larger proportion of converged networks exhibited non-monotonic interpolations.

\paragraph{MNIST \& Fashion-MNIST classifiers.} On the MNIST and Fashion-MNIST datasets, we explored varying dataset size, network size (depth/width of hidden layers), activation function, choice of optimizer, optimization hyperparameters, initialization methods, and the use of batch normalization. In Table~\ref{tab:mnist_lr}, we compare two-layer networks trained with SGD and Adam. Models trained with SGD typically retained the MLI property but those trained with Adam frequently did not.
In Appendix~\ref{app:experiments_opt}, we show additional results for models trained with RMSProp and K-FAC~\citep{martens2015optimizing} (whose behaviour is qualitatively close to Adam) along with the interpolated loss curves.

\paragraph{CIFAR-10 \& CIFAR-100 classifiers.} On CIFAR-10 and CIFAR-100 datasets, we trained two-layer convolutional neural networks (SimpleCNN), LeNet~\citep{lecun1989backpropagation}, AlexNet~\citep{krizhevsky2012imagenet}, VGG16, VGG19~\citep{simonyan2014very}, and ResNets~\citep{he2016deep} with different choices of optimizer and learning rates. In Figure~\ref{fig:cnn_vary_opt}, we show a broad overview of the interpolation paths for different architectures and optimizers. Across all models, the average $\min \Delta$ was 0.016 and 0.626 for SGD and Adam respectively. Overall, Adam-trained models violated the MLI property $3.2$ times more often than SGD.

\subsubsection{The effect of batch normalization}\label{sec:exp:how_persistent:bn}
Batch normalization's invention and subsequent ubiquity postdate the initial investigation of the MLI property. Even now, the relationship between the MLI property and the use of batch normalization has not been investigated. We provide the first such study in this section. We found that the use of batch normalization greatly increased the rate at which trained networks failed to satisfy the MLI property.

\paragraph{MNIST \& Fashion-MNIST classifers.} Table~\ref{tab:mnist_lr} shows the effect of batch normalization on the MLI property for fully connected classifiers trained on MNIST \& Fashion-MNIST. The networks trained with batch normalization failed to satisfy the MLI property more frequently than those without. This is more pronounced with large learning rates and with Adam.

\paragraph{CIFAR-10 \& CIFAR-100 classifers.} Next, we trained ResNet models on CIFAR-10 \& CIFAR-100 classification tasks. We evaluated ResNet-\{20,32,44,56\} trained with Adam and SGD and with varying learning rates. We also varied the distribution over initial parameters and whether or not batch normalization was applied. The results for CIFAR-10 are displayed in Table~\ref{tab:cifar10_resnets} (CIFAR-100 results are similar, and are presented in Appendix~\ref{app:experiments}). The column headers, ``BN'' and ``NBN'' indicate batch normalization and no batch normalization respectively. The suffices ``I'' and ``F'' indicate two alternative initialization schemes, block-identity initialization~\citep{goyal2017accurate} and Fixup initialization~\citep{zhang2019fixup}. For each configuration, we report the percentage of models violating the MLI property and the average minimum $\Delta$ such that the model is $\Delta$-monotonic (conditioning on $\Delta > 0$). Batch normalization led to significantly more networks with non-monotonic interpolations. We also observed that the initialization of the residual blocks plays an important role in shaping the loss landscape.

\begin{table}[]
    \centering
    \begin{adjustbox}{max width=\linewidth}
    \begin{tabular}{|c|c|c|c|c|c|}\hline
        & & BN & BN-I & NBN-I & NBN-F\\\hline
\multirow{2}{*}{\rotatebox[origin=c]{90}{ SGD }} & \% (total)  & 0.54 (26) & 0.00 (26) & 0.00 (23) & 0.11 (27)\rule{0pt}{2.6ex}\rule[-1.2ex]{0pt}{0pt}\\
 & $\min \Delta$ & 0.794 & 0.000 & 0.000 & 0.076 \rule[-1.2ex]{0pt}{0pt}\\\hline
\multirow{2}{*}{\rotatebox[origin=c]{90}{Adam}} & \% (total) & 0.77 (22) & 0.27 (30) & 0.20 (20) & 0.04 (23)\rule{0pt}{2.6ex}\rule[-1.2ex]{0pt}{0pt}\\
 & $\min \Delta$ & 0.351 & 0.054 & 0.033 & 0.332 \rule[-1.2ex]{0pt}{0pt}\\\hline
    \end{tabular}
    \end{adjustbox}
    \caption{Evaluation of effect of batch normalization, initialization, and choice of optimizer for residual networks trained on CIFAR-10 (achieving better than 1.0 training loss). We display the proportion of networks with non-monotonic interpolation and average $\min \Delta$ such that the network is $\Delta$-monotonic over varying training settings. Full explanation of table is given in main text.}
    \label{tab:cifar10_resnets}
    \vspace{-0.3cm}
\end{table}

%% file: sections/experiments/why.tex
\subsection{Why does MLI hold?}
\label{sec:exp:why_mli}

\begin{figure*}[!h]
\begin{minipage}{0.5\linewidth}
\centering
\ifarxiv
\includegraphics[width=\linewidth]{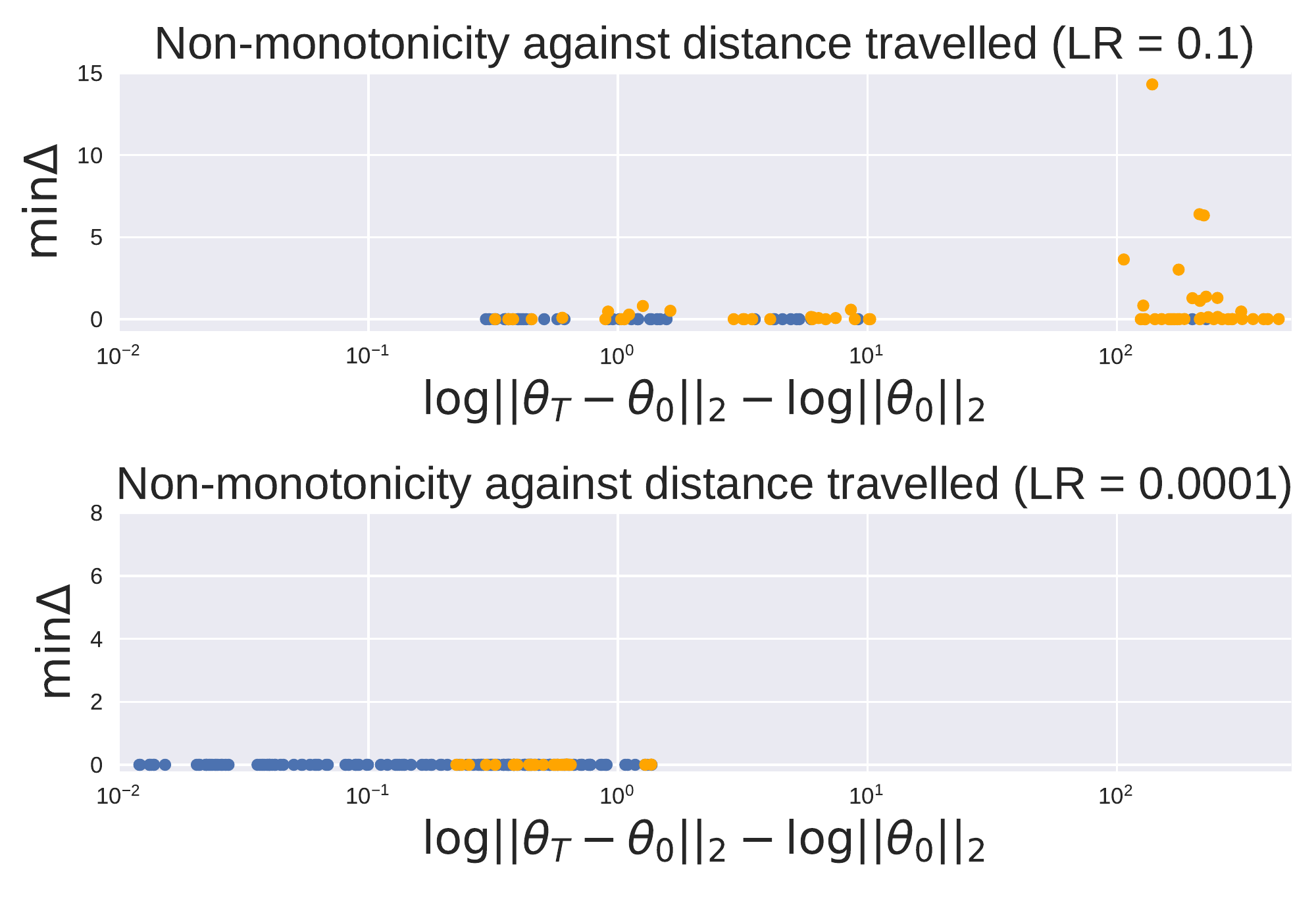}
\else
\includegraphics[width=0.9\linewidth]{figures/mnist_opt/mono_vs_dis_lr2.pdf}
\fi
\end{minipage}\hfill%
\begin{minipage}{0.5\linewidth}
\centering
\ifarxiv
\includegraphics[width=\linewidth]{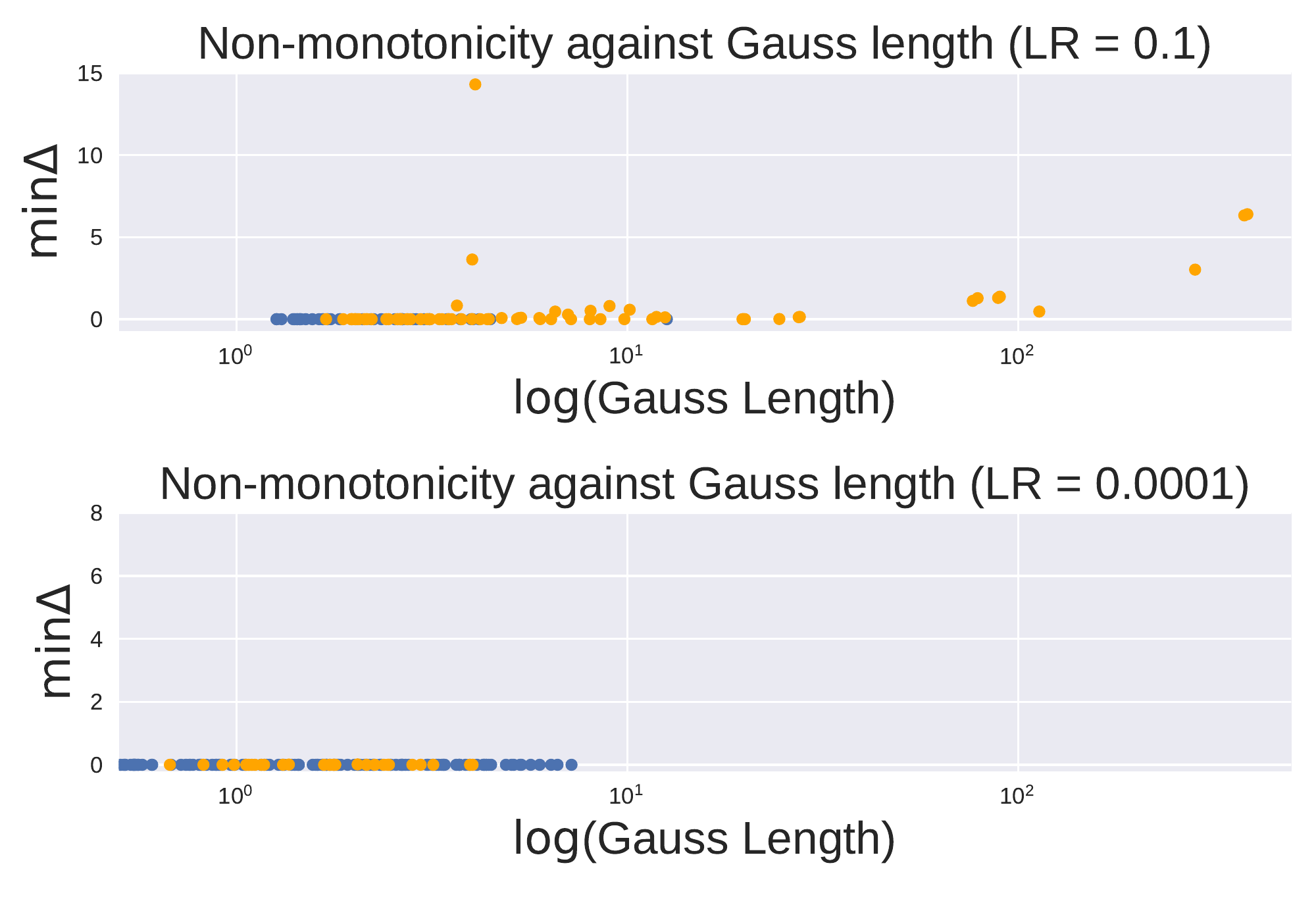}
\else
\includegraphics[width=0.9\linewidth]{figures/mnist_opt/mono_vs_gl_lr.pdf}
\fi
\end{minipage}
\caption{For each MNIST \& Fashion-MNIST classifier, we compute the minimum $\Delta$ such that the interpolated loss is $\Delta$-monotonic. We plot models trained with a learning rate of 0.1 and 0.0001 in the top and bottom rows respectively. On the left, we compare the distance moved in the weight space. On the right, we compare the Gauss length of the interpolated network outputs. \textbf{\textcolor{blue}{Blue}} points represent networks where the MLI property holds and \textbf{\textcolor{orange}{orange}} points are networks where the MLI property fails.}
\label{fig:nm_gl_ds}
\end{figure*}

In Section~\ref{sec:mli_property}, we discussed the parameter- and function-space perspectives of the MLI property. In our experiments, we explore these two perspectives on reconstruction and classification tasks. We computed the average Gauss length of the logit interpolations and the weight distance travelled. In both cases, these measures are predictive of MLI in practice, even for values exceeding the limits of our theory.

In Appendix~\ref{app:experiments_gen}, we provide the full set of results for all settings we explored. Additionally, we provide an investigation of the relationship between the MLI property and generalization. In summary, we did not find a clear relationship between the success of the MLI property and the generalization ability of the neural network.

\subsubsection{Weight distance vs. monotonicity}

Throughout our experiments, we found that weight distance was negatively correlated with the monotonicity of the interpolated network. In Figure~\ref{fig:nm_gl_ds} (left), we show the relationship between the (normalized) distance travelled in weight space and the minimum $\Delta$ such that fully-connected classifiers are $\Delta-$monotonic.

First, we note that larger learning rates encourage greater movement in weight space --- a finding that also extends to batch normalization and the use of adaptive optimizers. Second, we observed that the networks that travelled short distances during optimization consistently satisfied the MLI property. Conversely, networks with larger distances travelled in weight space were more likely to exhibit non-monotonic loss interpolations. In Appendix~\ref{app:experiments_weight_dis}, we show similar results for the autoencoders, CIFAR-10 \& CIFAR-100 classifiers, and comparisons over batch normalization and adaptive optimizers.

\begin{figure*}[!h]
\begin{minipage}{0.5\linewidth}
\centering
\ifarxiv
\includegraphics[width=\linewidth]{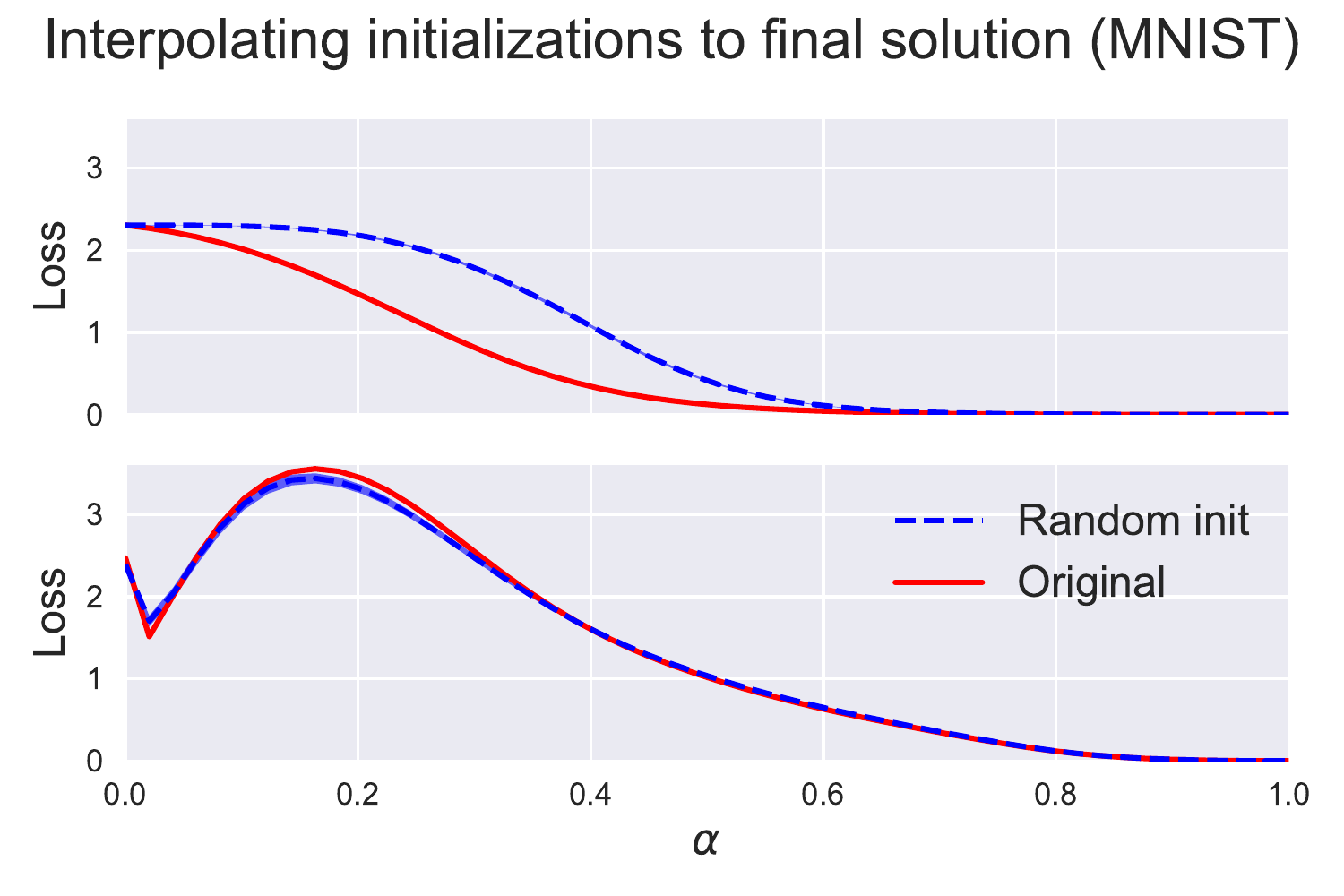}
\else
\includegraphics[width=0.9\linewidth]{figures/mnist_interp_nway/rand_init_compare.pdf}
\fi
\end{minipage}\hfill%
\begin{minipage}{0.5\linewidth}
\centering
\ifarxiv
\includegraphics[width=\linewidth]{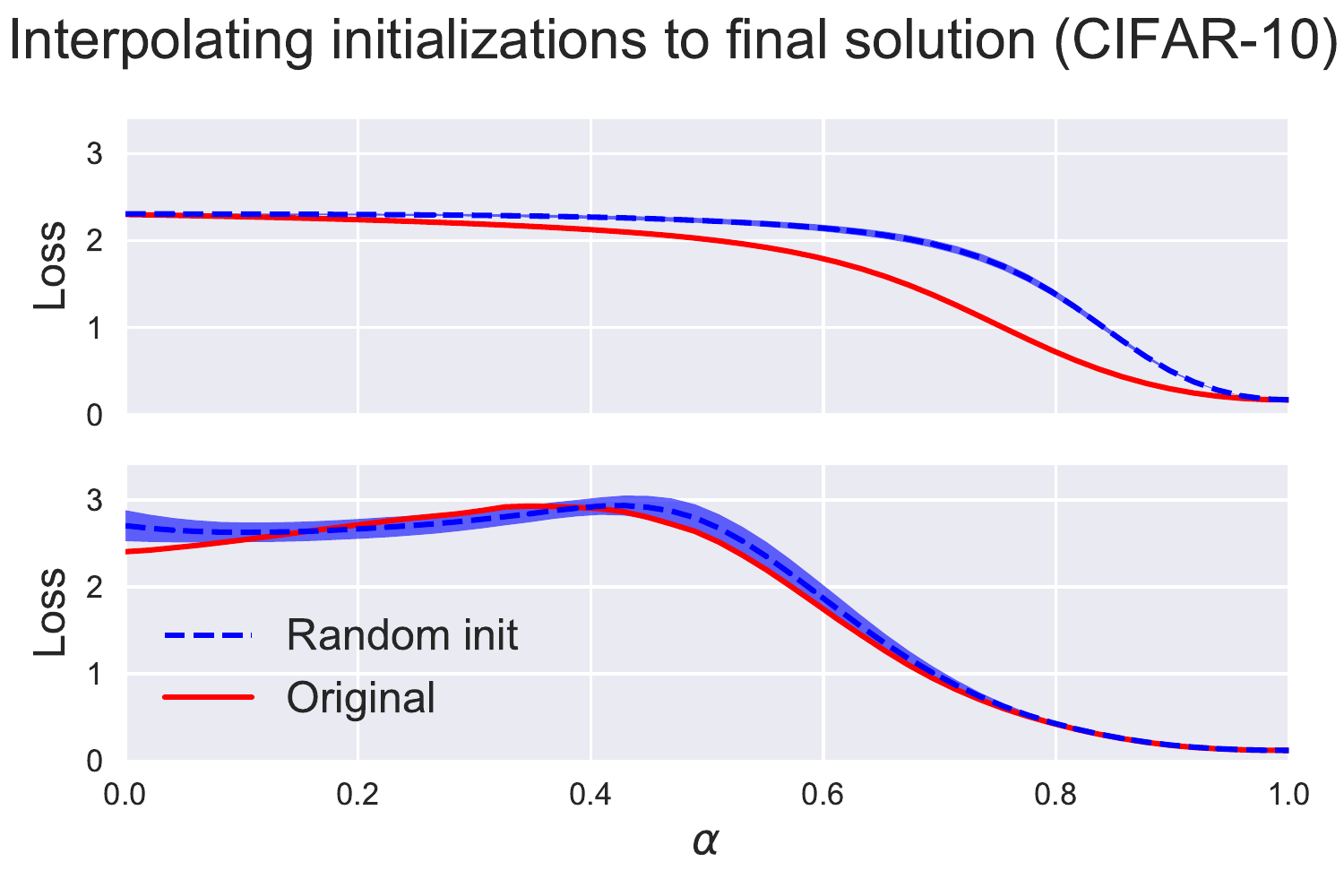}
\else
\includegraphics[width=0.9\linewidth]{figures/CIFAR_search/CIFAR10/rand_init_compare.pdf}
\fi
\end{minipage}\vspace{-0.5cm}
\caption{Classifier interpolation loss on training set between 15 different random initializations and an optimum. The top row shows interpolation towards a final solution that is monotonic with its original initialization. The bottom row shows this interpolation for a non-monotonic original pair. For the random initializations, mean loss is shown with standard deviation ($\pm 1$) as filled region.}
\label{fig:rand_init_to_optimum}
\vspace{-0.3cm}
\end{figure*}

\ifarxiv
\begin{wrapfigure}[13]{R}{0.5\linewidth}
    \centering
    \vspace{-0.5cm}
    \includegraphics[width=\linewidth]{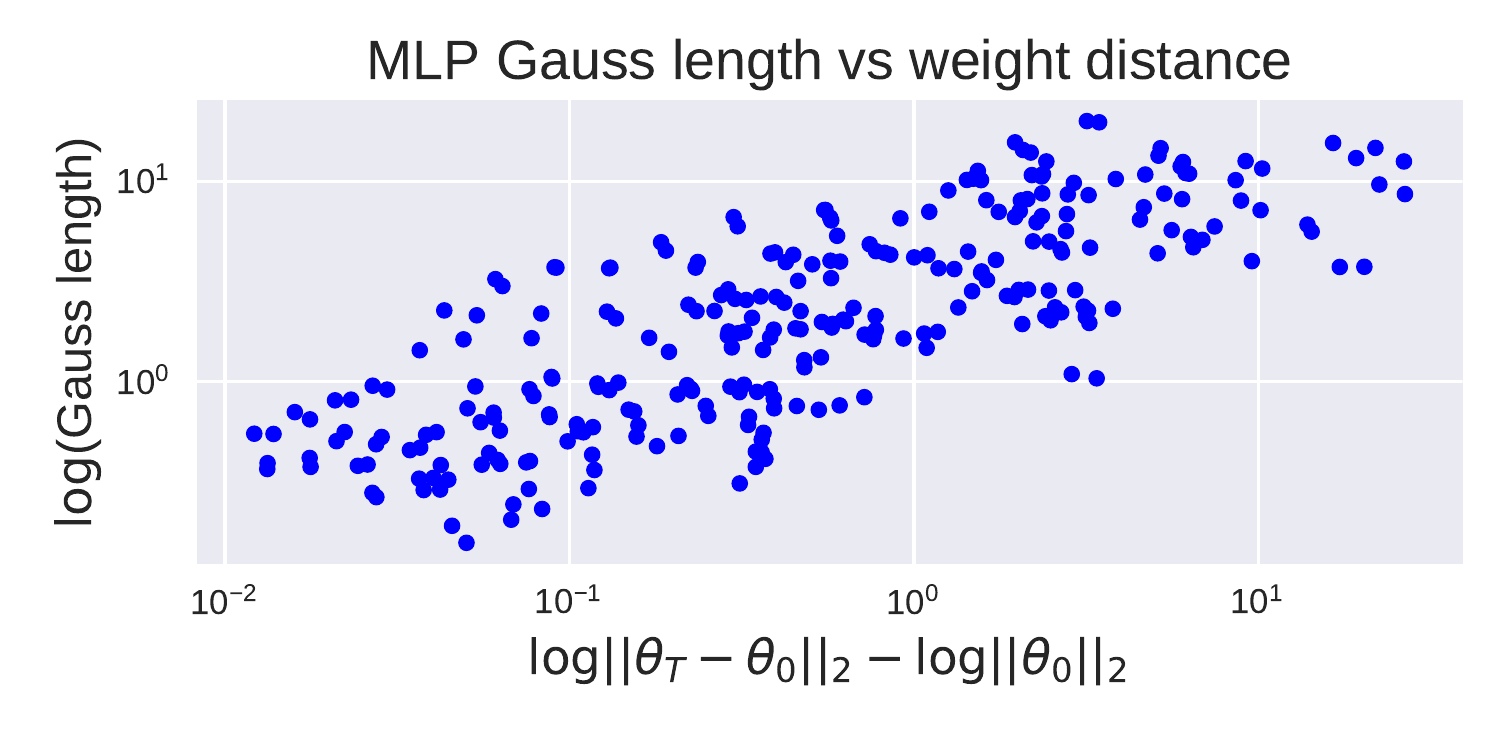}\vspace{-0.8cm}
    \caption{Power law relationship between Gauss length and weight distance travelled for MLP \& Fashion-MNIST experiments. ($R^2 = 0.616$)}
    \label{fig:mlp_power_law}
    \vspace{-1cm}
\end{wrapfigure}
\else
\begin{figure}[H]
    \centering
    \includegraphics[width=\linewidth]{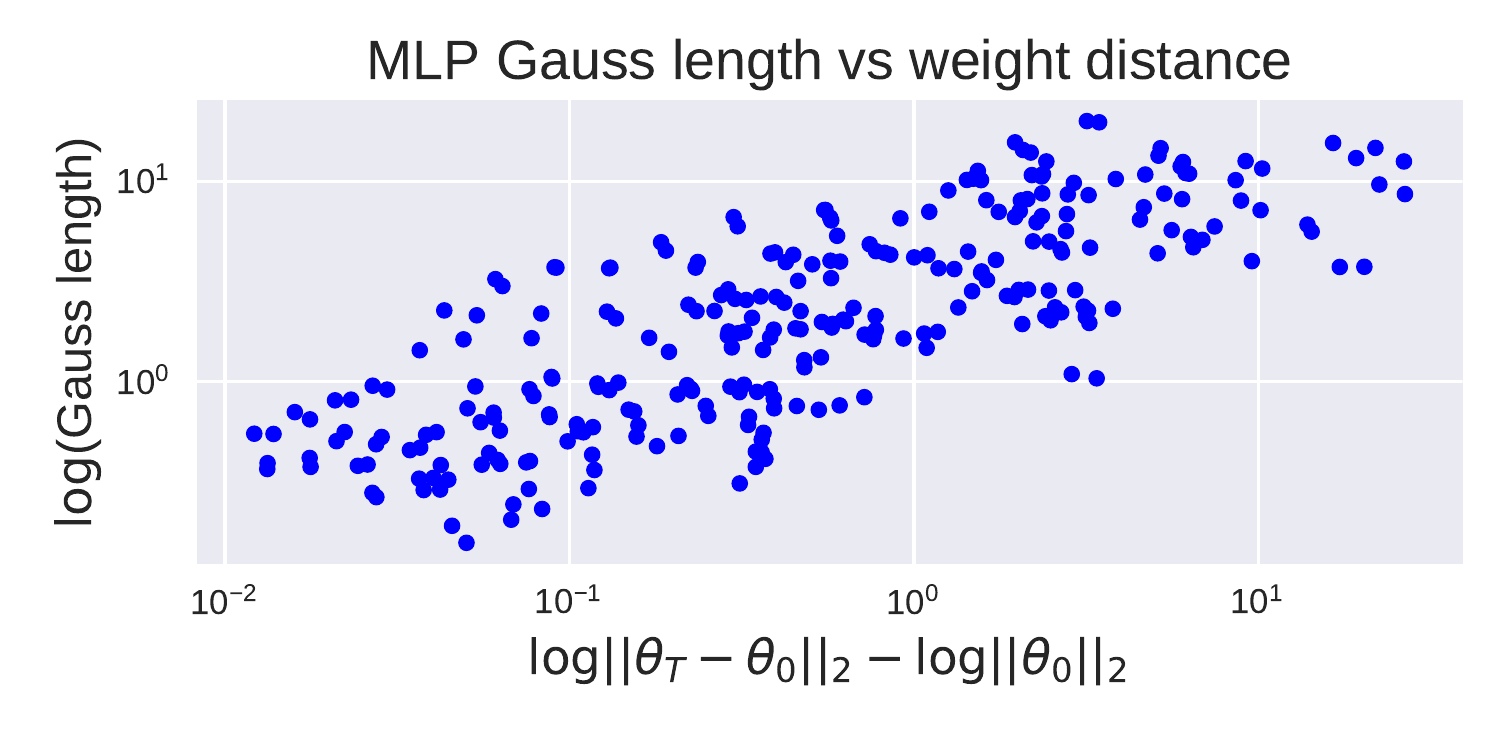}\vspace{-0.5cm}
    \caption{Power law relationship between Gauss length and weight distance travelled for MLP \& Fashion-MNIST experiments. ($R^2 = 0.616$)}
    \label{fig:mlp_power_law}
    \vspace{-0.4cm}
\end{figure}
\fi

\subsubsection{Gauss length vs. monotonicity}
We also observed a negative correlation between the Gauss length of the logit interpolations and the minimum $\Delta$ such that the loss interpolation is $\Delta$-monotonic. In Figure~\ref{fig:nm_gl_ds} (right), we make this comparison for classifiers trained on MNIST \& Fashion-MNIST. As our analysis predicts, small Gauss lengths lead to monotonic interpolations. And beyond the strict limits of our theoretical analysis, we find that as the Gauss length increases, the non-monotonicity also increases. 

We also observed that larger learning rates lead to much larger Gauss lengths. As with the weight distance, this finding extends to batch normalization and the use of adaptive optimizers too (see Appendix~\ref{app:experiments_gl}). In Appendix~\ref{app:experiments_opt_abl}, we conduct an ablation study to investigate the relationship between Gauss length and the choice of optimizer by changing the optimizer in the middle of training (SGD $\to$ Adam and Adam $\to$ SGD). Switching to Adam at any point during training leads to large Gauss length and weight distance without a significant spike in the training loss --- with little variation due to the time of the optimizer switch.

\subsubsection{Gauss length vs weight distance}

When the distance moved in weight space is small, we would expect a small Gauss length as a linearization of the network provides a good approximation. However, it is not obvious what relationship (if any) should be expected more generally. Surprisingly, we consistently observed a power-law relationship between the average Gauss length and the distance moved in weight space (Figure~\ref{fig:mlp_power_law}). We observed this relationship across all of the experimental settings that we explored. Full results are presented in Appendix~\ref{app:experiments_gl_wd}.

%% file: sections/experiments/what.tex
\subsection{What does MLI say about loss landscapes?}
\label{sec:exp:what_landscape}

Thus far, we focused on the monotonicity of paths connecting the initialization and final network solution. In this section, we ask: are (non-)monotonic interpolations unique to the initialization and final solution pair?

\begin{figure}
\begin{minipage}{0.33\linewidth}
\centering
\includegraphics[width=\linewidth]{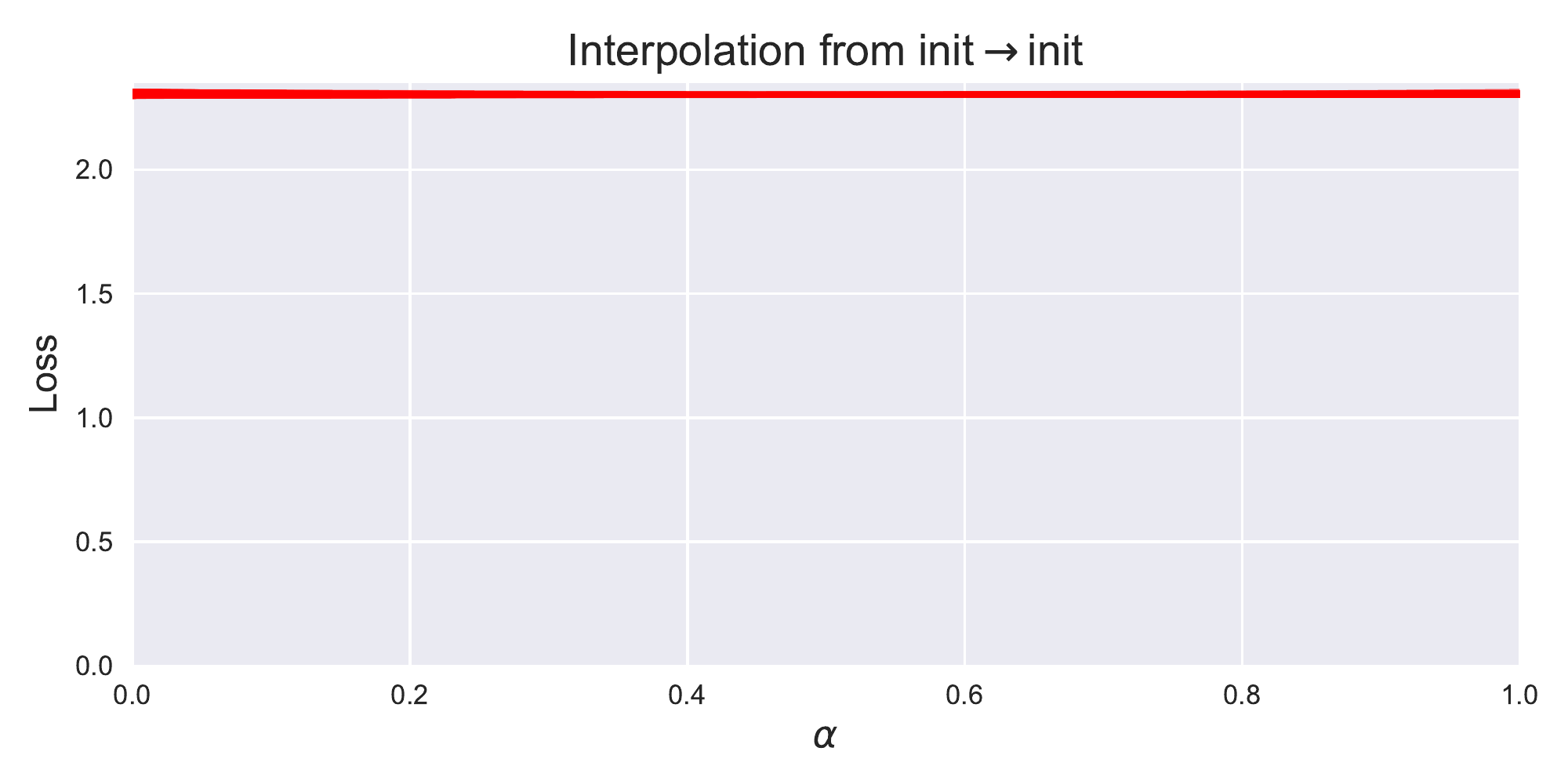}
\end{minipage}\hfill%
\begin{minipage}{0.33\linewidth}
\centering
\includegraphics[width=\linewidth]{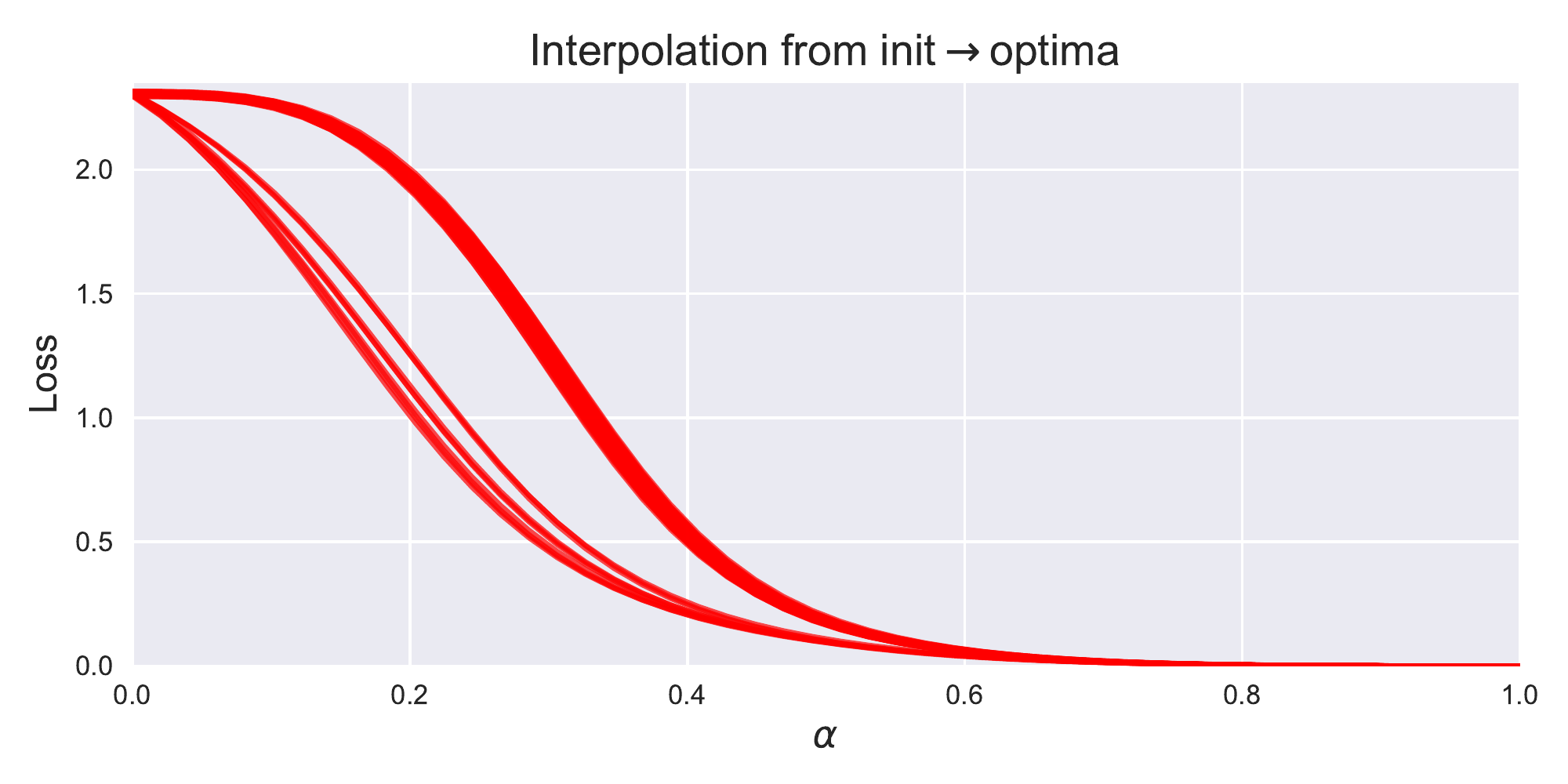}
\end{minipage}\hfill%
\begin{minipage}{0.33\linewidth}
\centering
\includegraphics[width=\linewidth]{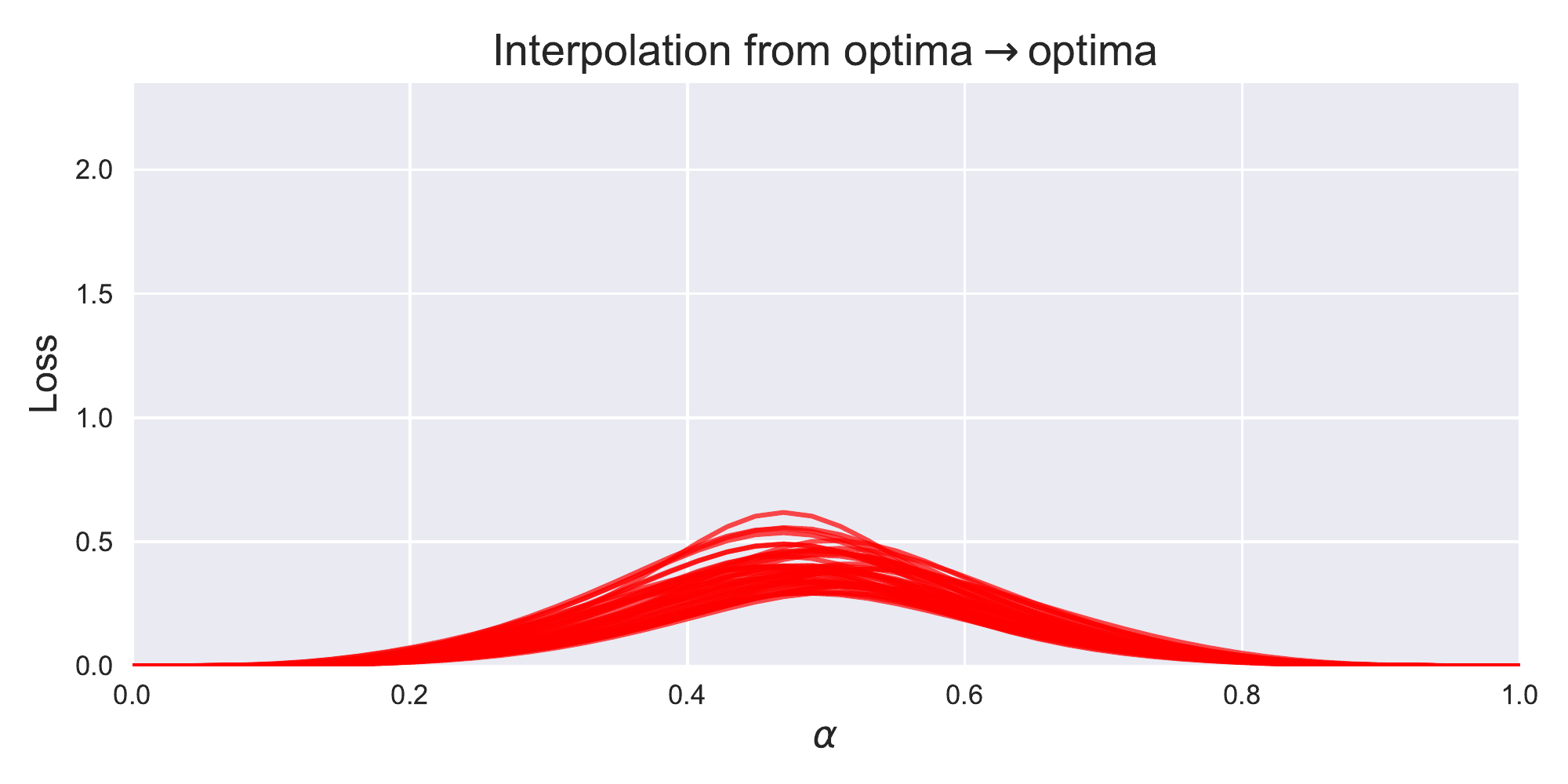}
\end{minipage}\vspace{-0.2cm}
\caption{\footnotesize Linear interpolation for 10 FashionMNIST classifiers with less than 0.01 final loss. Left: Interpolating between all pairs of initializations. Middle: Interpolating from all initializations to all optima. Right: Interpolating between all pairs of optima.}
\label{fig:li_pairings}\vspace{-0.65cm}
\end{figure}

To this end, we evaluated linear interpolations between learned network parameters and unrelated random initializations (Figure~\ref{fig:rand_init_to_optimum}). For a fully-connected MNIST classifier and a ResNet-20 trained on CIFAR-10, we found that random initializations display the same interpolation behaviour as the original initialization-solution pair. This suggests that the MLI property is not tied to a particular pair of parameters but rather is a global property of the loss landscape. We also explored linear interpolations between pairs of initializations, initialization to optima pairs, and pairs of optima in Figure~\ref{fig:li_pairings}. No barriers were observed between the pairs of initializations or the initialization$\rightarrow$optimum pairs, but barriers are present between the optima. This highlights the rich structure present in the loss landscape of these models and aligns well with the qualitative predictions of~\citet{NEURIPS2019_48042b1d}.

Finally, we provide visualizations of the loss landscape via 2D projections of the parameter space. While low-dimensional projections of high-dimensional spaces are often misleading, in the case of linear interpolations, the entire path lies in the projected plane. Therefore, these visualizations give us valuable insight into connectivity in the loss landscape for multiple initialization $\to$ final solution paths.

In Figure~\ref{fig:cut_roberta_esperanto}, we show 2D projections of the loss landscape for RoBERTa~\citep{liu2019roberta} trained as a language model on Esperanto~\citep{conneau2019unsupervised} using the HuggingFace library~\citep{wolf-etal-2020-transformers}. We trained two models and plotted the initial points and optima for both. Both initial points are monotonically connected to both minima.

\begin{figure}[H]
\begin{minipage}{0.5\linewidth}
\centering
\includegraphics[width=\linewidth]{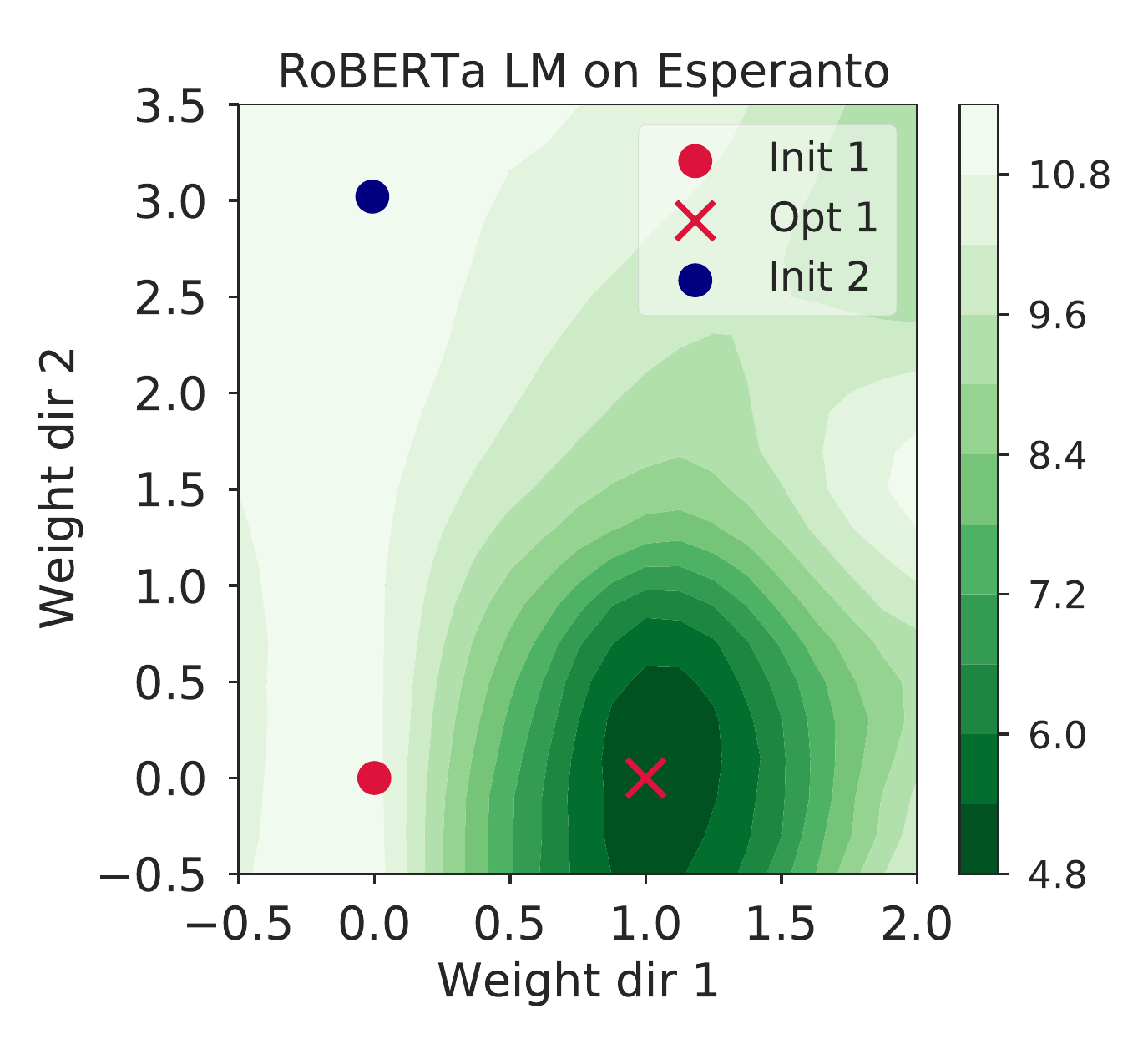}
\end{minipage}\hfill%
\begin{minipage}{0.5\linewidth}
\centering
\includegraphics[width=\linewidth]{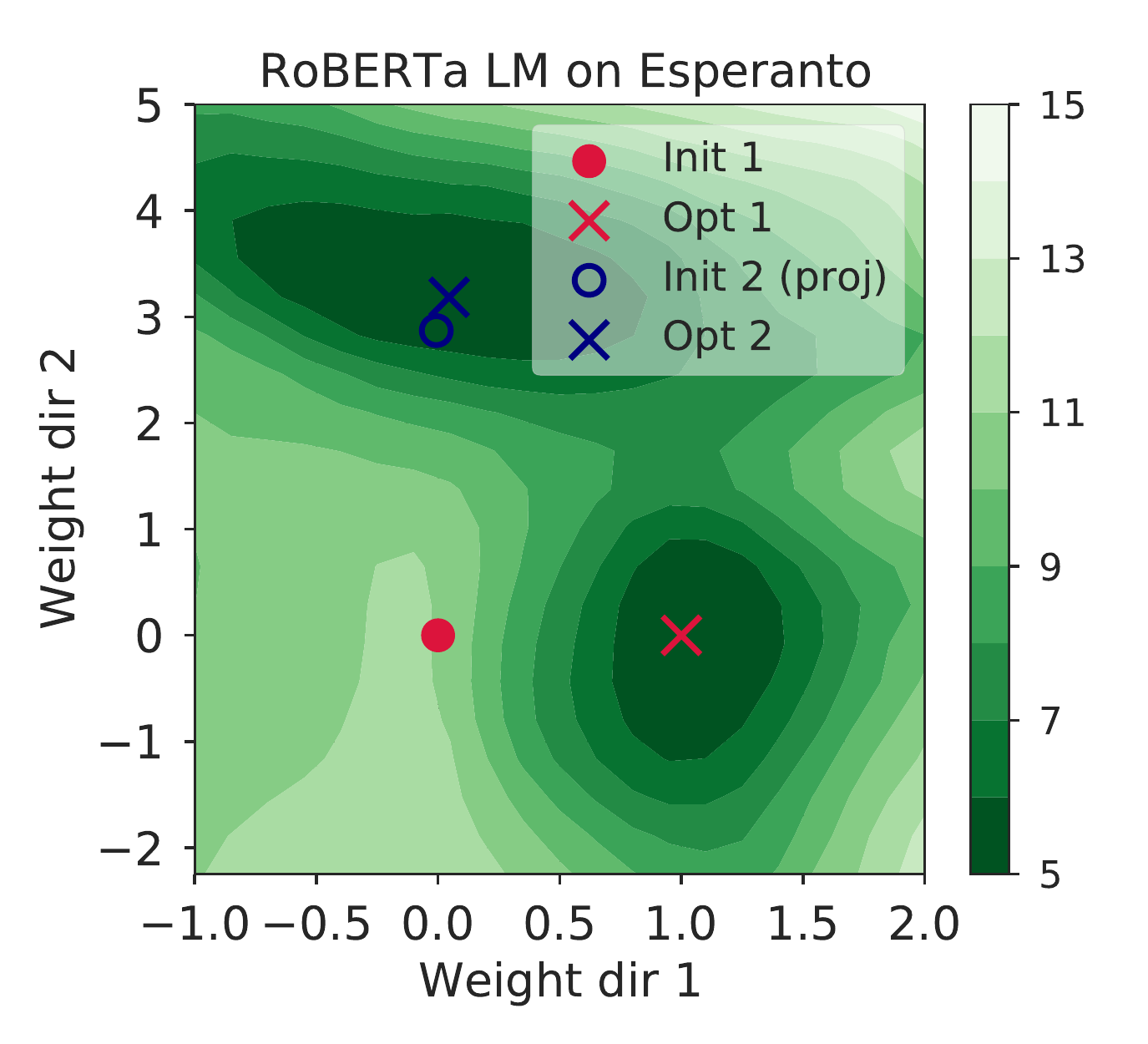}
\end{minipage}\vspace{-0.5cm}
    \caption{Two-dimensional sections of the weight space for RoBERTa trained as a language model on Esperanto. Left: plane defined by two initializations and the optima reached from one of them is shown. Right: plane defined by ``Init 1'' and two optima are shown (with ``Init 2'' projected onto the plane).}\label{fig:cut_roberta_esperanto}
\end{figure}

%% file: sections/conclusion.tex
\section{Conclusion}
\label{sec:conclusion}

\citet{goodfellow2014qualitatively} first showed that linear interpolation between initial and final network parameters monotonically decreases the training loss. In this work, we provided the first evidence that this so called, Monotonic Linear Interpolation (MLI), is not a stable property of neural network training. In doing so, we provided a deeper theoretical understanding of the MLI property and properties of the loss landscape in general. Our empirical investigation of the MLI property explored variations in datasets, architecture, optimization, and other training mechanisms. We identified several mechanisms that systematically produce trained networks that violate the MLI property, and connected these mechanisms to our theoretical explanations of the MLI property. Additional results indicate that the MLI property is not unique to the initialization$\to$solution pair produced by training, but rather is a global property of the loss landscape connecting arbitrary initialization$\to$solution pairs. The empirical and theoretical analysis we presented highlights the intriguing properties of neural network loss landscapes.

%% file: appendix/appendix.tex
\input{appendix/gauss_length_theory}

\input{appendix/exp_details}
\input{appendix/additional_experiments}
\input{appendix/additional_theory}

%% file: appendix/gauss_length_theory.tex
\section{Theoretical Gauss Length Analysis}
\label{app:gauss-length}

In this section, we provide an analysis of the MLI property via the Gauss Length. In particular, we prove Theorem~\ref{thm:small_gauss_mse_mono} that states that if the logit interpolation of a network (from initialization to optimum) has small Gauss length then it must satisfy the MLI property. Using Theorem~\ref{thm:small_gauss_mse_mono}, we provide sufficient conditions for the MLI property to hold for two-layer linear models. And prove, under a class of these models satisfying some standard assumptions, that the MLI property holds almost surely.

Let's first recall the definition of the Gauss length.
\gausslength*

Explicitly, we have,
\[\langle \partial_\alpha \hat{\bv}(\alpha), \partial_\alpha \hat{\bv}(\alpha)\rangle = \frac{(\bv \cdot \bv) (\ba \cdot \ba) - (\ba \cdot \bv)^2}{\bv \cdot \bv} = \kappa(\alpha)^2 (\bv \cdot \bv),\]
where $\ba = \frac{\partial \bv}{\partial \alpha}$ and $\kappa$ denotes the curvature of $\bz$. Theorem~\ref{thm:small_gauss_mse_mono} is reproduced below for convenience.

\smallgaussmono*

To prove this result, we will require the following Lemma.
\begin{lemma}\label{lemma:wide_tangents}
Let $\bx^* \in \bbR^d$. Consider a smooth curve $\bz(t) \in \bbR^d$ for $t \in [0,1)$ with $\Vert \bz(0) - \bx^* \Vert > 0$ and $\bz(1) = \bx^*$. If there exists $b \in [0,1)$ with,
\[\Vert \bz(b) - \bx^* \Vert_2 > \Vert \bz(0) - \bx^* \Vert_2,\]
then there exists $t_1 \in [0,b)$ and $t_2 \in (b,1)$ such that $\langle \dot{\bz}(t_1), \dot{\bz}(t_2) \rangle \leq 0$.
\end{lemma}

\begin{proof}
We prove the contrapositive statement: If for all $t_1 \in [0, b)$ and $t_2 \in (b, 1)$, we have $\langle \dot{\bz}(t_1), \dot{\bz}(t_2) \rangle > 0$, then, for all $b \in [0,1)$, we have $\Vert \bz(b) - \bx^* \Vert_2 \leq \Vert \bz(0) - \bx^* \Vert_2$.

By the fundamental theorem of calculus, we have,
\begin{align*}0 < \int_{b}^{1} \int_{0}^{b} \langle \dot{\bz}(t_1), \dot{\bz}(t_2) \rangle dt_1 dt_2 &= \langle \bz(b) - \bz(0), \bx^* - \bz(b) \rangle,\\
&= \langle \bx^* - \bz(0) + \bz(b) - \bx^*, \bx^* - \bz(b) \rangle\\
&= \langle \bx^* - \bz(0), \bx^* - \bz(b) \rangle - \Vert \bx^* - \bz(b) \Vert_2^2,\\
\end{align*}
Now, notice that as $\Vert \bx^* - \bz(b) \Vert_2^2 \geq 0$, we must have $\langle \bx^* - \bz(0), \bx^* - \bz(b) \rangle > 0$. Thus, by applying the Cauchy-Schwarz inequality, 
\begin{align*}
\langle \bx^* - \bz(0), \bx^* - \bz(b) \rangle - \Vert \bx^* - \bz(b) \Vert_2^2 &\leq \Vert \bx^* - \bz(0) \Vert_2 \Vert \bx^* - \bz(b) \Vert_2 - \Vert \bx^* - \bz(b) \Vert_2^2,\\
&= \Vert \bx^* - \bz(b) \Vert_2 \left(\Vert \bx^* - \bz(0) \Vert_2 - \Vert \bx^* - \bz(b) \Vert_2\right).
\end{align*}
It follows immediately that for any $b$ we must have,
\[\Vert \bz(b) - \bx^* \Vert_2 \leq \Vert \bz(0) - \bx^* \Vert_2,\]
as required.
\end{proof}

With this result, we proceed with the proof of Theorem~\ref{thm:small_gauss_mse_mono}.

\begin{proof}
We prove this theorem by considering the contrapositive statement: if there exists $a < b \in (0,1)$ such that $f(\bz(a)) < f(\bz(b))$ then the Gauss length is greater than $\pi / 2$.

Given such a pair $(a, b)$, we consider the restriction of $\bz$ to $[a, 1)$. By Lemma~\ref{lemma:wide_tangents}, there exists $t_1$ and $t_2$ such that $\langle \dot{\bz}(t_1), \dot{\bz}(t_2) \rangle < 0$.

Therefore, the normalized tangents also satisfy $\langle \hat{\bv}(t_1), \hat{\bv}(t_2) \rangle < 0$. Thus, the angle between the two normalized tangents (considered in the plane containing these two points and $\bx^*$) is at least $\pi / 2$. Therefore, the Gauss length of the curve on $(a, 1)$ must be at least $\pi / 2$ (with the minimum Gauss length path given by the shortest path on the projective plane connecting $\hat{\bv}(t_1)$ and $\hat{\bv}(t_2)$).
\end{proof}

Finally, we note here that the converse of Theorem~\ref{thm:small_gauss_mse_mono} does not hold. For example, one may define a curve that spirals towards the minima. This curve is monotonically decreasing but has arbitrarily large Gauss length.

\input{appendix/gauss_length_lae}

%% file: appendix/gauss_length_lae.tex
\subsection{Two-layer linear models and the MLI property}
\label{app:two_layer_linear}

In this section, we apply Theorem~\ref{thm:small_gauss_mse_mono} to two-layer linear models. In particular, we prove sufficient conditions on any two-layer linear model to satisfy the MLI property and then prove that under certain assumptions the MLI property holds almost surely.

Our focus is on two-layer linear models, of the form $f(\bx) = VW \bx$, for $W \in \bbR^{k \times d}$ and $V \in \bbR^{m \times k}$. We consider optimizing these models with respect to the mean squared error.
\[\calL(X, Y; V, W) = \frac{1}{2n}\Vert VW X - Y \Vert_2^2,\]
where $X \in \bbR^{d \times n}$ and $Y \in \bbR^{m \times n}$. Note that this model also captures the linear autoencoder, when we set $X = Y$ with $m=d$.

We consider learning in the student-teacher setting, where the labels $Y$ are provided by a two-layer linear model with $k$ hidden units. This allows the application of Theorem~\ref{thm:small_gauss_mse_mono}, as the interpolation trajectory can reach the minimum of the objective. However, outside of this realizable setting we can still apply Theorem~\ref{thm:small_gauss_mse_mono} to the surrogate objective with $Y$ replaced by the minimum achievable target $\hat{Y}$ --- this objective aligns with the original at the global minimum.

Now consider a linear interpolation over initial parameters $V_0$, $W_0$ and final parameters $V_T$, $W_T$, denoted,
\[\bz(\alpha) = \left(V_0 + \alpha(V_T - V_0)\right)\left( W_0 + \alpha(W_T - W_0)\right)X.\]
Going forwards, we write $D_1 = (V_T - V_0)$ and $D_2 = (W_T - W_0)$. We first observe that the tangent to this curve is a linear function of $\alpha$:
\begin{equation}
    \bz'(\alpha) = \left(D_1 W_0 + V_0 D_2 + 2\alpha D_1 D_2 \right)X.
\end{equation}
The Gauss length of the interpolated trajectory is given by the length of the projection of the tangent vectors onto the projective space, in this case the sphere with antipodal points identified. Immediately, we note that this line projects onto the sphere as an arc, with end points given by the projection of $\bz'(0)$ and $\bz'(1)$. Following this, the Gauss length is less than $\pi/2$ exactly when the (vectorized) inner product of the two endpoint is positive (implying the angle between them is at most $\pi / 2$). Furthermore, the Gauss length of the interpolation path is at most $\pi$ for any initial-final parameter pair.

The two endpoints are given by:
\begin{equation}
    \bz'(0) = \left(D_1 W_0 + V_0 D_2\right)X \:\:\:\:\textrm{ and }\:\:\:\: \bz'(1) = \left(D_1 W_T + V_T D_2\right)X
\end{equation}

Recall the Kronecker product identity $\vecop{AX} = (I \otimes A)\vecop{X}$, where $\vecop{\cdot}$ indicates column-major vectorization. Then we have,
\begin{align*}
    \langle\bz'(0) , \bz'(1)\rangle &= \vecop{\left(D_1 W_0 + V_0 D_2\right)X}^\top \vecop{\left(D_1 W_T + V_T D_2\right)X}\\
    &= \vecop{X}^\top \left(I \otimes (D_1 W_0 + V_0 D_2)^\top\right)\left(I \otimes (D_1 W_T + V_T D_2)\right) \vecop{X}\\
    &= \vecop{X}^\top \left(I \otimes \left((D_1 W_0 + V_0 D_2)^\top(D_1 W_T + V_T D_2)\right) \right)\vecop{X}
\end{align*}
Now, noting that $I \otimes A$  has the same eigenvalues as $A$ (with increased multiplicity), we have $\langle \bz'(0), \bz'(1) \rangle > 0$ for all $X$ if and only if all eigenvalues of $(D_1 W_0 + V_0 D_2)^\top(D_1 W_T + V_T D_2)$ are positive.

\paragraph{Proving that the MLI property holds with probability 1.} Under the \emph{tabula rasa} assumptions from \citet{saxe2019mathematical} we can prove that the MLI property holds almost surely. The assumptions that underly this setting are as follows.
\begin{enumerate}
    \item The inputs are whitened ($\frac{1}{n}XX^\top = I$).
    \item Initialization is balanced ($V_0 = W_0^\top$).
    \item The learning rate of gradient descent is sufficiently small (relative to the largest singular value of the input-output correlation matrix ($\frac{1}{n} YX^\top = USR^\top$).
\end{enumerate}

Under these assumptions, \citet{saxe2019mathematical} prove that,
\[W(t) = Q \sqrt{A(t)} U^\top \:\textrm{ and }\: V(t) = U \sqrt{A(t)} Q^{-1},\]
for some invertible matrix $Q \in \bbR^{k\times k}$.

Under these dynamics, we have,
\[D_1 = U (\sqrt{A(t)} - \sqrt{A(0)}) Q^{-1} \textrm{ and } D_2~=~Q~(\sqrt{A(t)}~-~\sqrt{A(0)})~U^\top.\] Thus,
\begin{align*}
    (D_1 W_0 + V_0 D_2)^\top(D_1 &W_T + V_T D_2) \\&= 4U\left(\sqrt{A(t)} - \sqrt{A(0)}\right)\sqrt{A(0)}U^\top U \left(\sqrt{A(t)} - \sqrt{A(0)}\right)\sqrt{A(t)} U^\top\\
    &= 4U \sqrt{A(0)A(t)}\left(\sqrt{A(t)} - \sqrt{A(0)}\right)^2 U^\top
\end{align*}

This matrix is positive definite, and thus has positive eigenvalues. Therefore, the found solution will satisfy the MLI property.

%% file: appendix/exp_details.tex
\section{Experiment Details}
\label{app:exp_details}

In this section, we provide full details of our experimental set-up. For all experiments, we discretize $\alpha$ in the interval $[0, 1]$ using 50 uniform steps to examine the MLI property. When training networks with SGD, we used a momentum coefficient of 0.9 and when training networks with the Adam optimizer, we used $\beta_1 = 0.9, \beta_2 = 0.999$ and $\epsilon=1e-08$. Unless specified otherwise, we used a batch size of 128.

\subsection{Image reconstruction experiments}
\label{app:exp_details_reconstruct}

In the image reconstruction experiments, we used deep autoencoders with the ReLU activation function. Our architecture consisted of $784 \rightarrow 512 \rightarrow H \rightarrow 512 \rightarrow 784$ units in each respective layer with $H \in \{1, 2, 5, 10, 25, 50, 100\}$. We trained the networks using either SGD with momentum or Adam. Each model was trained for 200 epochs using fixed learning rates in the set $\{0.3, 0.1, 0.03, 0.01, 0.003, 0.001, 0.0003, 0.0001\}$ and batch sizes of 512.

\subsection{Image classification experiments}
In the image classification experiments, we explored a large number of different architectures. We summarize all setting we explored below.

\paragraph{Multilayer Perceptron.} We train fully connected networks with varying widths and depths. For all experiments (except Figure~\ref{fig:mnist_delta_heatmaps}), widths were chosen from the set \{16, 128, 1024, 2048, 4096\} and depth was chosen from \{2, 4, 8\}. We trained each model using one of SGD, RMSProp, Adam, or KFAC, with fixed learning rates from the set \{3.0, 1.0, 0.3, 0.1, 0.03, 0.01, 0.003, 0.001, 0.0003, 0.0001\}. We experimented with 3 activation functions: tanh, sigmoid, and ReLU. We trained the networks for 200 epochs both with and without batch normalization.

\paragraph{Convolution Neural Network.} We trained Simple CNN, VGG16, VGG19, and ResNet-\{18, 20, 32, 44, 50, 56\} with and without batch normalization, on CIFAR-10 and CIFAR-100. The Simple CNN had two convolutional layers with a $5\times 5$ kernel followed by a single fully connected layer. We trained the networks with both SGD and the Adam optimizer. For all models, we used an initial learning rate in the set  $\{0.3, 0.1, 0.03, 0.01, 0.003, 0.001, 0.0003, 0.0001\}$. For most models we fixed the learning rate throughout training but for the ResNets we used a waterfall learning rate decay (at 60, 90, and 120 epochs).

For the ResNet experiments without batch normalization, when using the Fixup \citep{zhang2019fixup} or block identity initialization \citep{goyal2017accurate} we replaced the batch normalization layers with scale and bias parameters taking the role of the standard batch norm affine transformation. The block identity initialization essentially consists of setting the final scale/bias parameters in each residual block to zero, so that the block computes only the skip connection (with possible down-sampling).

\subsection{Language modeling experiments}
We trained a RoBERTa transfomer-based model~\citep{liu2019roberta} on the language modelling task on an Esperanto dataset with the Huggingface framework~\citep{wolf-etal-2020-transformers}, as described in their tutorial\footnote{\url{https://huggingface.co/blog/how-to-train}} and building on a notebook they published\footnote{\url{https://colab.research.google.com/github/huggingface/blog/blob/master/notebooks/01_how_to_train.ipynb}}. We trained the model from two distinct random initializations for 1 epoch (taking approximately 2 hours on a free Google Colab GPU).

\subsection{Experiment specifics}
\label{app:experiment-specific}

\paragraph{MNIST \& Fashion-MNIST batch norm comparison.} We describe the experimental set-up used to produce Figure~\ref{fig:784_logit_viz} and Figure~\ref{fig:784_all}. We trained fully-connected networks whose architecture consisted of $784 \rightarrow 1024 \rightarrow 1024 \rightarrow 10$ units in each layer. We explored ReLU, sigmoid, and tanh activation functions and trained the networks with and without batch norm layers, that when used were inserted after each linear layer (except the last layer). The networks were trained for 200 epochs using fixed learning rates in the set $\{3.0, 1.0, 0.3, 0.1, 0.03, 0.01, 0.003, 0.001\}$ and with either the Adam optimizer or SGD with momentum.

\paragraph{Problem difficulty experiments.} For the experiments evaluating problem difficulty (parameter complexity and label corruption), described in Appendix~\ref{app:problem-diff}, we trained fully-connected networks on the FashionMNIST dataset. In all cases, the networks used ReLU activations and were trained with batch sizes of at most 512 (depending on dataset size), and for 200 epochs. Learning rates were fixed throughout training. When varying the dataset size, we trained models on random subsets of FashionMNIST with sizes in the set $\{10, 30, 100, 300, 1000, 3000, 10000, 30000, 60000\}$. We evaluated networks trained with learning rates in the set $\{0.03, 0.1, 0.3, 1.0\}$. For the experiments with varying levels of label corruption, we trained fully-connected networks with 2 hidden-layers each of width 1024 and without batch normalization.

%% file: appendix/additional_experiments.tex
\section{Extended empirical evaluation}
\label{app:experiments}

In our exploration of the MLI property, we performed many additional experiments. Generally, we found that turning common knobs of neural network training did not have a significant impact on networks satisfying the MLI property. For example, varying activation functions, loss functions, batch size, regularization, and different forms of initialization had no significant effect on the MLI property. In this section, we present a few of the more interesting additional experiments that we performed.

\subsection{Relationship between the MLI property and generalization}
\label{app:experiments_gen}

We are interested in whether the success/failure of the MLI property impacts the generalization ability of a neural network. To better understand this, we examined the test accuracy of the models we trained and studied the correlation with the MLI property on MNIST and CIFAR-10 datasets. Figure~\ref{fig:mlp_gen} shows the relationship between the test accuracy and $\min \Delta$. Note that we considered fully connected networks for MNIST dataset and VGG architectures for CIFAR-10. For the MNIST experiments, configurations that violated the MLI property had an average test accuracy of $96.94 (\pm 0.015)$ and those that satisfied the MLI property had the averaged test accuracy of $97.14 (\pm 0.018)$. Similarly, for CIFAR-10 experiments,  configurations that violated the MLI property had an average test accuracy of $75.83 (\pm 0.084)$ and those that satisfied the MLI property had an average test accuracy of $76.99 (\pm 0.082)$. Overall, we did not identify a clear pattern between the MLI property and the generalization property of the neural network. At the very least, we ascertain that models violating the MLI property can achieve competitive test accuracy.

\begin{figure}[!h]
\begin{minipage}{0.5\linewidth}
\centering
\includegraphics[width=0.9\linewidth]{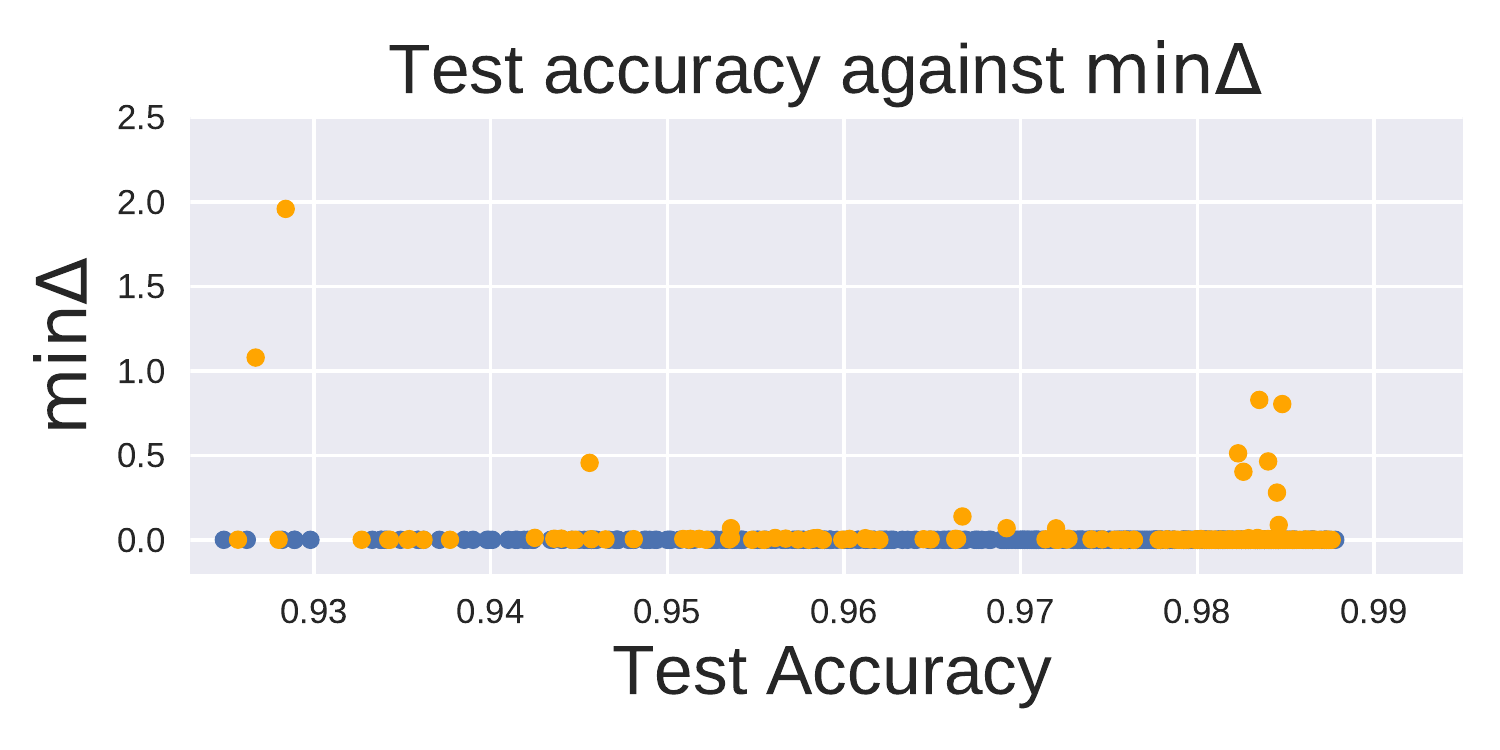}
\end{minipage}\hfill%
\begin{minipage}{0.5\linewidth}
\centering
\includegraphics[width=0.9\linewidth]{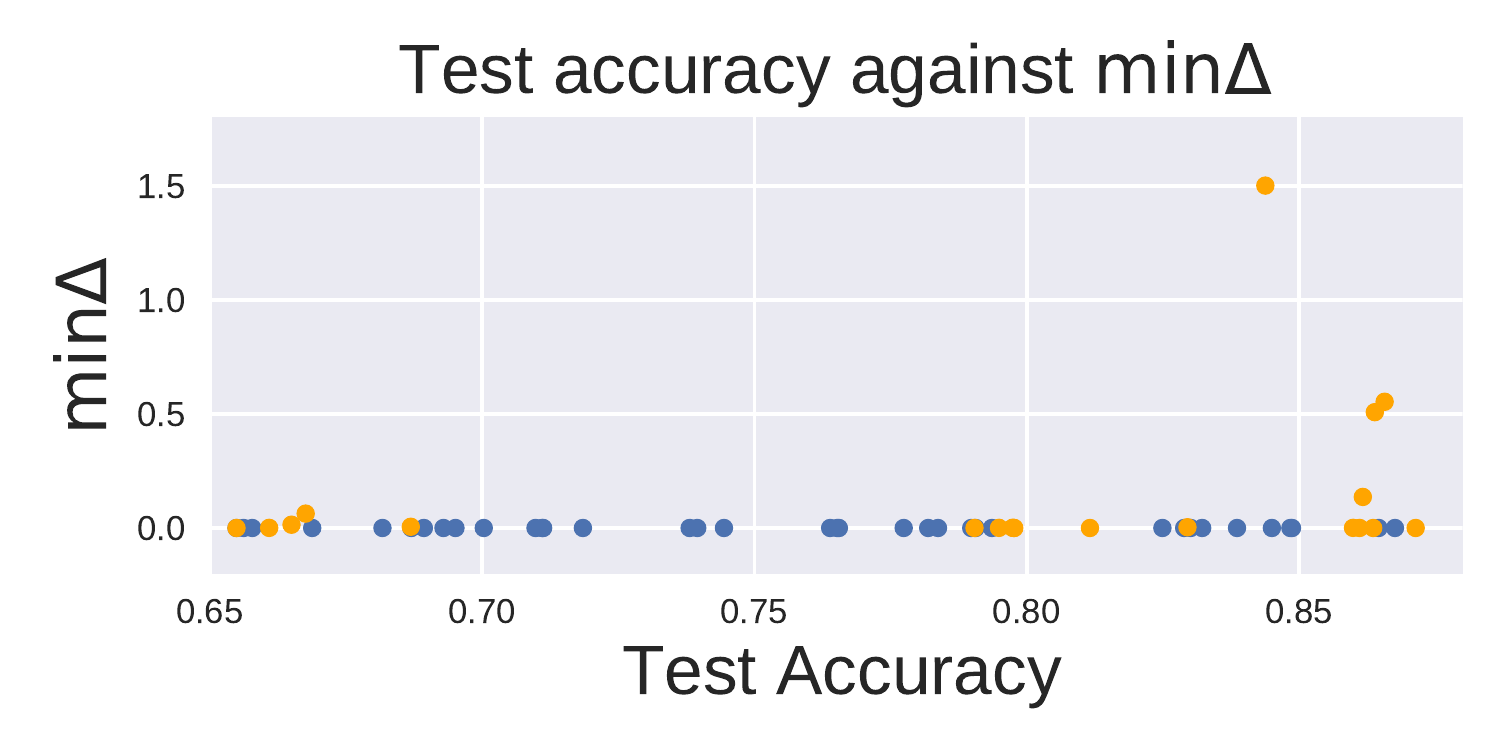}
\end{minipage}
\caption{Relationship between test accuracy and non-monotonicity in MNIST (left) and CIFAR10 (right) datasets. \textbf{\textcolor{blue}{Blue}} points represent networks where the MLI property holds and \textbf{\textcolor{orange}{orange}} points are networks where the MLI property fails.}
\label{fig:mlp_gen}
\end{figure}

\subsection{Impact of large learning rate}
\label{app:experiments_lr}

In Table~\ref{tab:cifar_lr}, we show the proportion of ResNets trained on CIFAR-10 and CIFAR-100 that violated the MLI property. The experimental set up matches that used to produce Tables~\ref{tab:cifar10_resnets} and Table~\ref{tab:cifar100_resnets}. There is a general trend towards higher learning rates encouraging non-monotonicity though the correlation is weaker than for the MNIST/Fashion-MNIST classifiers.
\begin{table*}[]
    \centering
    \small
    \begin{tabular}{|l|r|l l l l l l l |}\hline
        & LR: & 0.0003 & 0.001 & 0.003 & 0.01 & 0.03 & 0.1 & 0.3\\ \hline
\multirow{2}{*}{\rotatebox[origin=c]{90}{SGD}} & BN & 0.00 (3) & 0.25 (4) & 0.38 (13) & 0.12 (17) & 0.06 (17) & 0.24 (17) & 0.35 (17)\rule{0pt}{2.6ex}\rule[-1.2ex]{0pt}{0pt}\\
 & No BN & - & 0.00 (7) & 0.00 (14) & 0.24 (17) & 0.00 (17) & 0.25 (16) & 0.00 (15)\rule[-1.2ex]{0pt}{0pt}\\ \hline
\multirow{2}{*}{\rotatebox[origin=c]{90}{Adam}} & BN & 0.38 (16) & 0.18 (17) & 0.35 (17) & 0.29 (17) & 0.67 (9) & 1.00 (6) & 1.00 (1)\rule{0pt}{2.6ex}\rule[-1.2ex]{0pt}{0pt}\\
 & No BN & 0.00 (13) & 0.00 (17) & 0.00 (17) & 0.62 (8) & 0.00 (5) & 0.00 (11) & 0.00 (6)\rule[-1.2ex]{0pt}{0pt}\\ \hline
    \end{tabular}
    \caption{Proportion of trained CIFAR-10 and CIFAR-100 classifiers (achieving better than 1.0/2.0 training loss respectively) that had non-monotonic interpolations from initialization to final solution. The total number of runs in each bin is displayed in parentheses next to the proportion. A dashed line indicates that no networks achieved the threshold loss.}
    \label{tab:cifar_lr}
\end{table*}

We provide additional evaluations of large learning rates in Figure~\ref{fig:mnist_delta_heatmaps}, where we evaluate the effect of changing network depth and width over varying learning rates. Full details are given in Appendix~\ref{app:mnist_additional}.

\subsection{Impact of optimization algorithm}
\label{app:experiments_opt}

\begin{figure*}[!t]
    \begin{minipage}{0.24\linewidth}
    \centering
    \includegraphics[width=\linewidth]{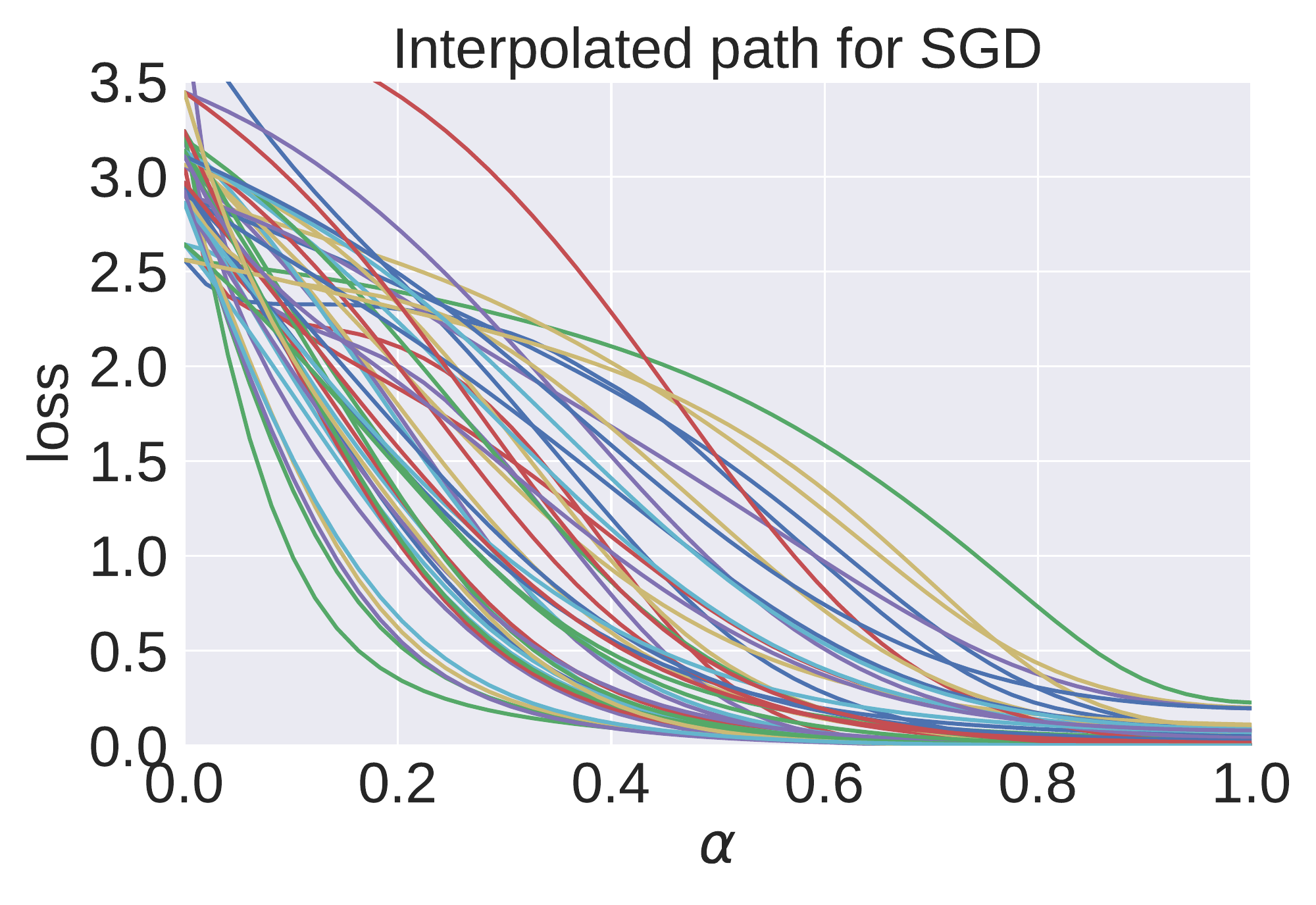}
    \end{minipage}\hfill%
    \begin{minipage}{0.24\linewidth}
    \centering
    \includegraphics[width=\linewidth]{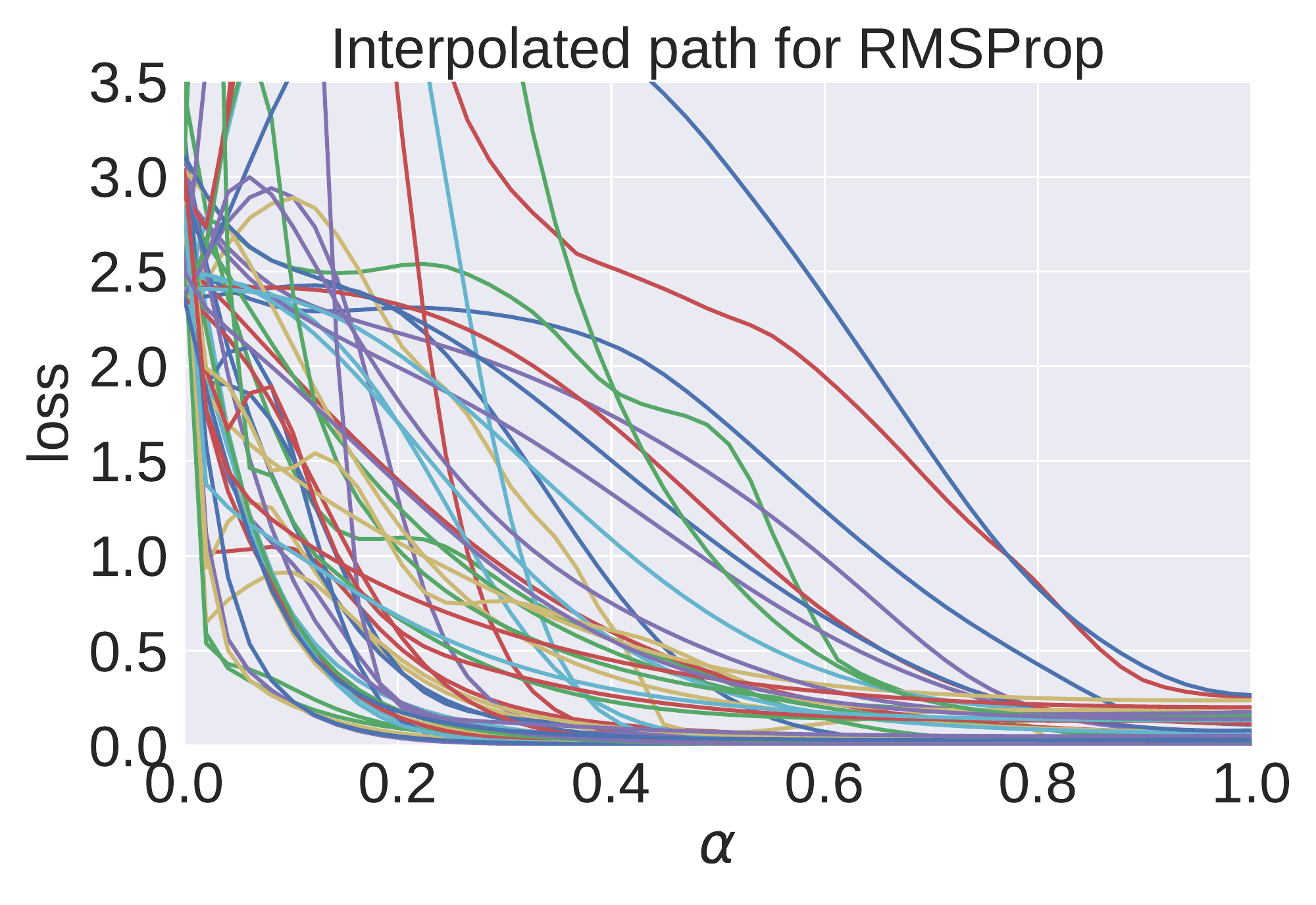}
    \end{minipage}\hfill%
    \begin{minipage}{0.24\linewidth}
    \centering
    \includegraphics[width=\linewidth]{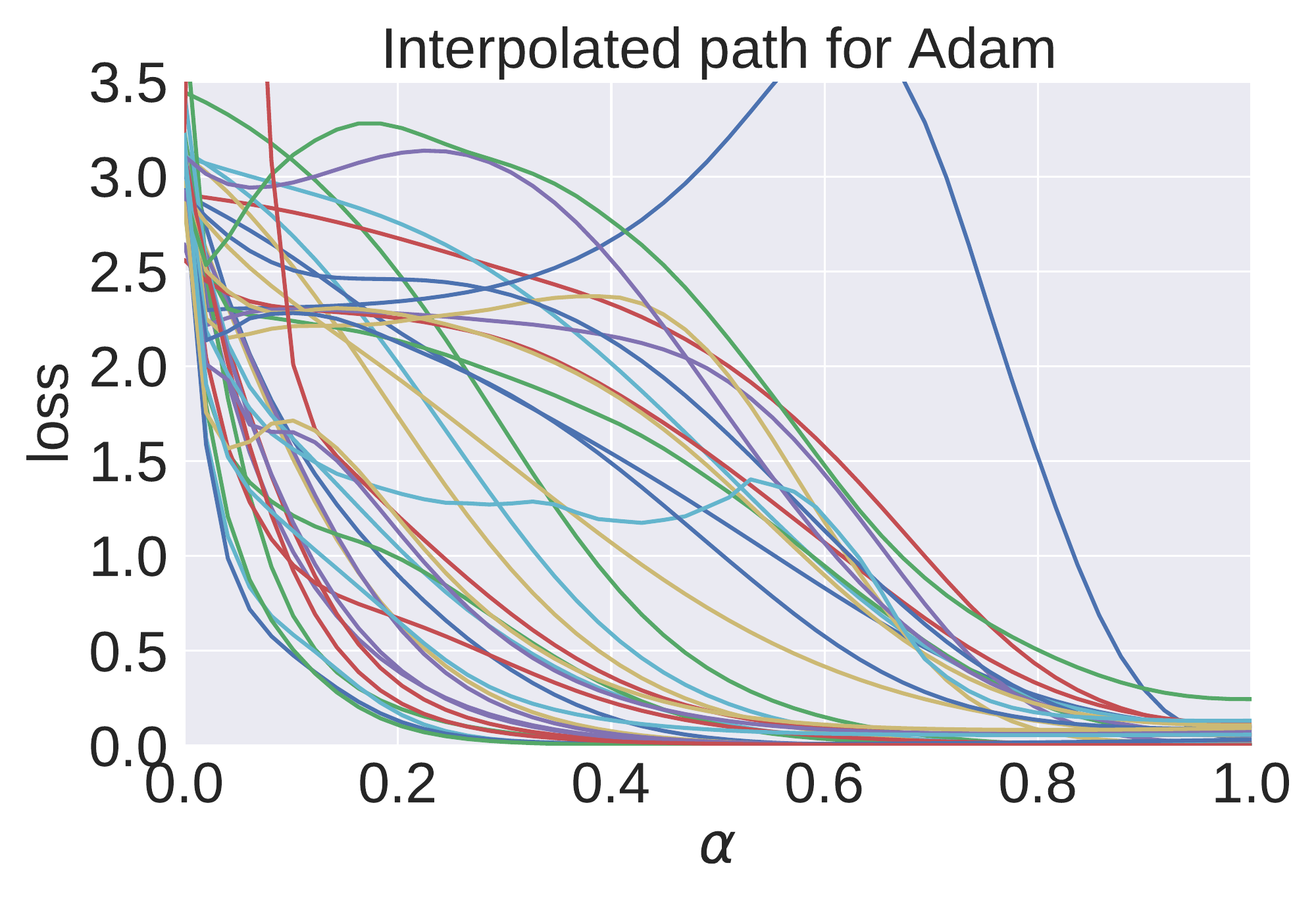}
    \end{minipage}
    \begin{minipage}{0.24\linewidth}
    \centering
    \includegraphics[width=\linewidth]{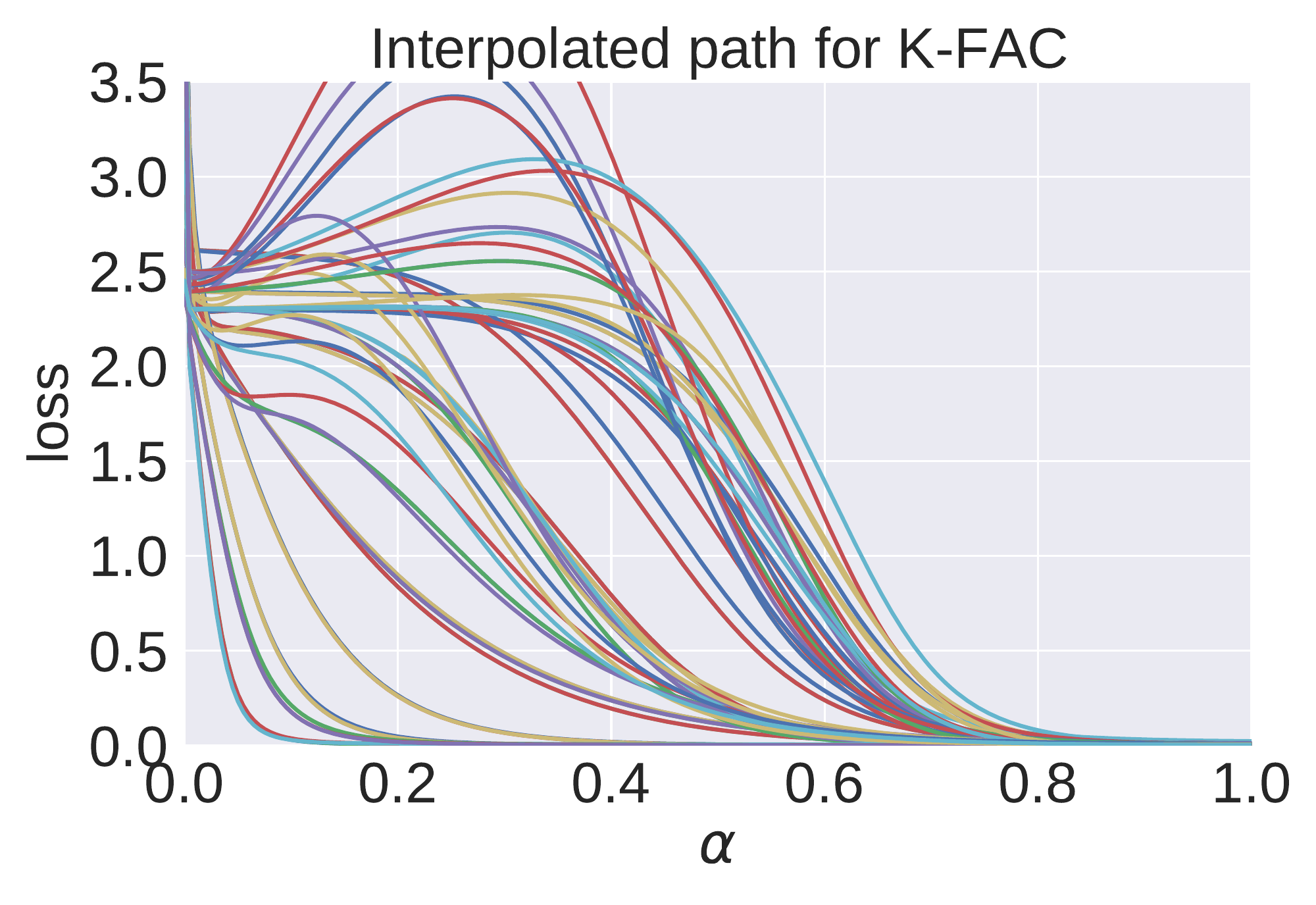}
    \end{minipage}
    \vspace{-0.1cm}
    \caption{Training loss over the linear interpolation connecting initial and final parameters. Each curve represents a network trained on MNIST \& Fashion-MNIST with different optimization algorithms. The MLI property generally holds for networks trained with SGD, but often fails for networks trained with RMSProp, Adam, and K-FAC.}
    \label{fig:mlp_vary_opt}
\end{figure*}

In Figure~\ref{fig:mlp_vary_opt}, we show the training loss over the line connecting the initial and final parameters. We found that adaptive optimizers such as RMSProp and Adam consistently find final solutions that violate the MLI property. To better understand this feature, we compare the distance travelled for all optimization methods. In Figure~\ref{fig:mlp_opt} (left), we show the distance travelled in weight space when trained with SGD and Adam for MNIST \& Fashion-MNIST classification tasks. When trained with Adam, the optimizer moved further away from the initialization --- confirming the results of~\citet{amari2020does}. On the other hand, models trained with SGD often traveled less. Moreover, non-monotonic configurations occurred more frequently for networks that travelled far from initialization, suggesting that the non-monotonicity of adaptive optimizers may be due to them encouraging parameters to travel far from initialization. 

We also investigated the relationship between non-monotonicity and Gauss length over varying optimizers. In Figure~\ref{fig:mlp_opt} (right), we show the Gauss length for networks trained using SGD and the Adam optimizer. When trained with Adam, on average the interpolation paths have a larger Gauss length and lead to more failures of the MLI property.

Table~\ref{tab:cifar100_resnets} contains our evaluation of the MLI property for ResNets trained with different architectures, optimizers, and initialization schemes (as in Table~\ref{tab:cifar10_resnets} for CIFAR-10 in the main paper). The general trends observed align with those observed on CIFAR-10 in the main paper.

\begin{table}[]
    \centering
    \begin{adjustbox}{max width=\linewidth}
    \begin{tabular}{|c|c|c|c|c|c|}\hline
    & & BN & BN-I & NBN-I & NBN-F\\\hline
    \multirow{2}{*}{\rotatebox[origin=c]{90}{SGD}} & \% Non-monotonic & 0.45 (20) & 0.00 (16) & 0.00 (18) & 0.28 (18)\rule{0pt}{2.6ex}\rule[-1.2ex]{0pt}{0pt}\\
    & $\min \Delta$ & 0.055 & 0 & 0 & 0.082\rule[-1.2ex]{0pt}{0pt}\\ \hline
    \multirow{2}{*}{\rotatebox[origin=c]{90}{Adam}} & \% Non-monotonic & 0.62 (16) & 0.00 (15) & 0.00 (15) & 0.00 (19)\rule{0pt}{2.6ex}\rule[-1.2ex]{0pt}{0pt}\\
    & $\min \Delta$ & 0.487 & 0 & 0 & 0\rule[-1.2ex]{0pt}{0pt}\\ \hline
    \end{tabular}
    \end{adjustbox}
    \caption{Evaluation of effect of batch normalization, initialization, and choice of optimizer for residual networks trained on CIFAR-100 (achieving at least 2.0 training loss). Full explanation of table is given in main text, Section~\ref{sec:exp:how_persistent:bn}.}
    \label{tab:cifar100_resnets}
\end{table}

\begin{figure}[!h]
\begin{minipage}{0.5\linewidth}
\centering
\includegraphics[width=0.9\linewidth]{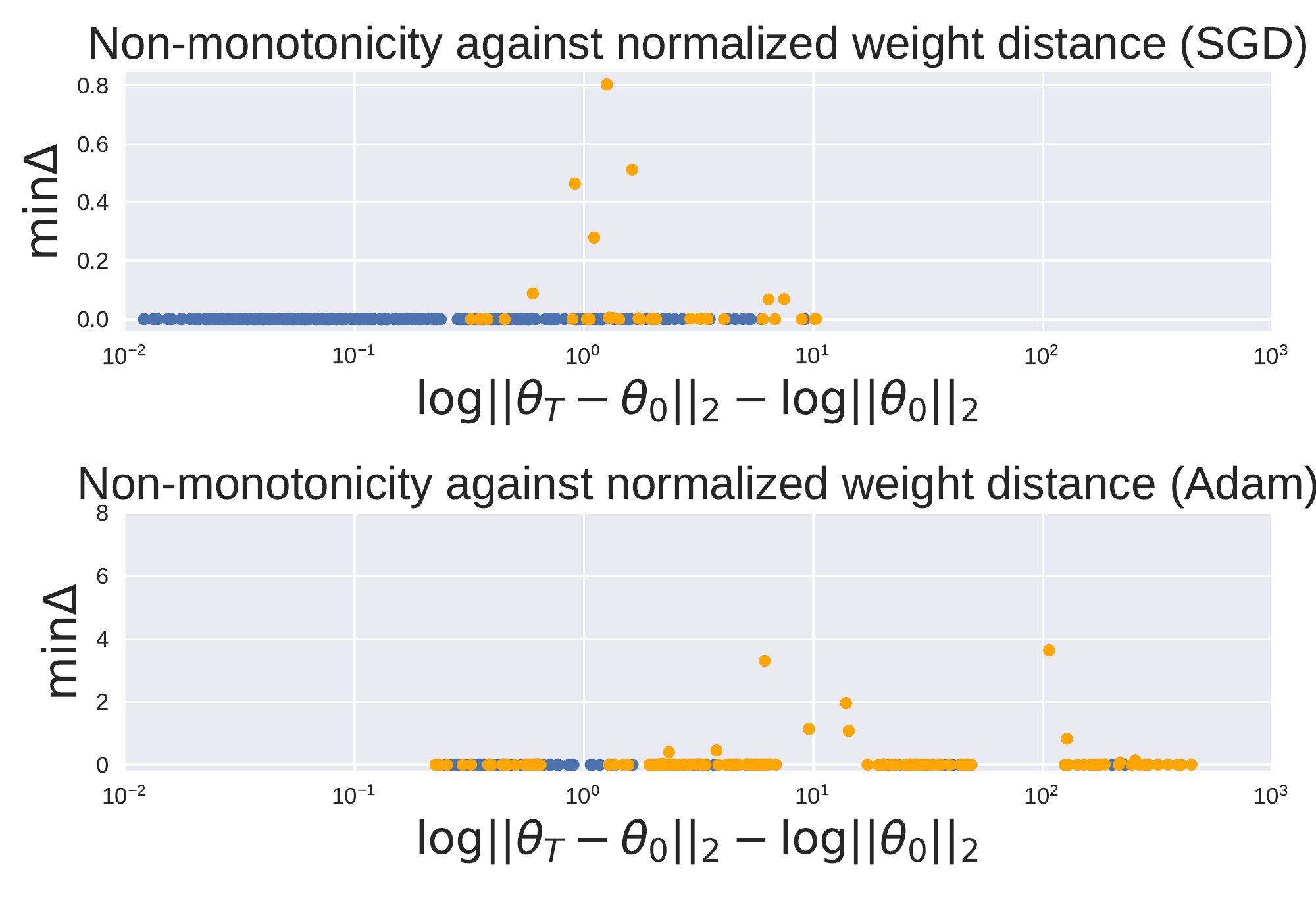}
\end{minipage}\hfill%
\begin{minipage}{0.5\linewidth}
\centering
\includegraphics[width=0.9\linewidth]{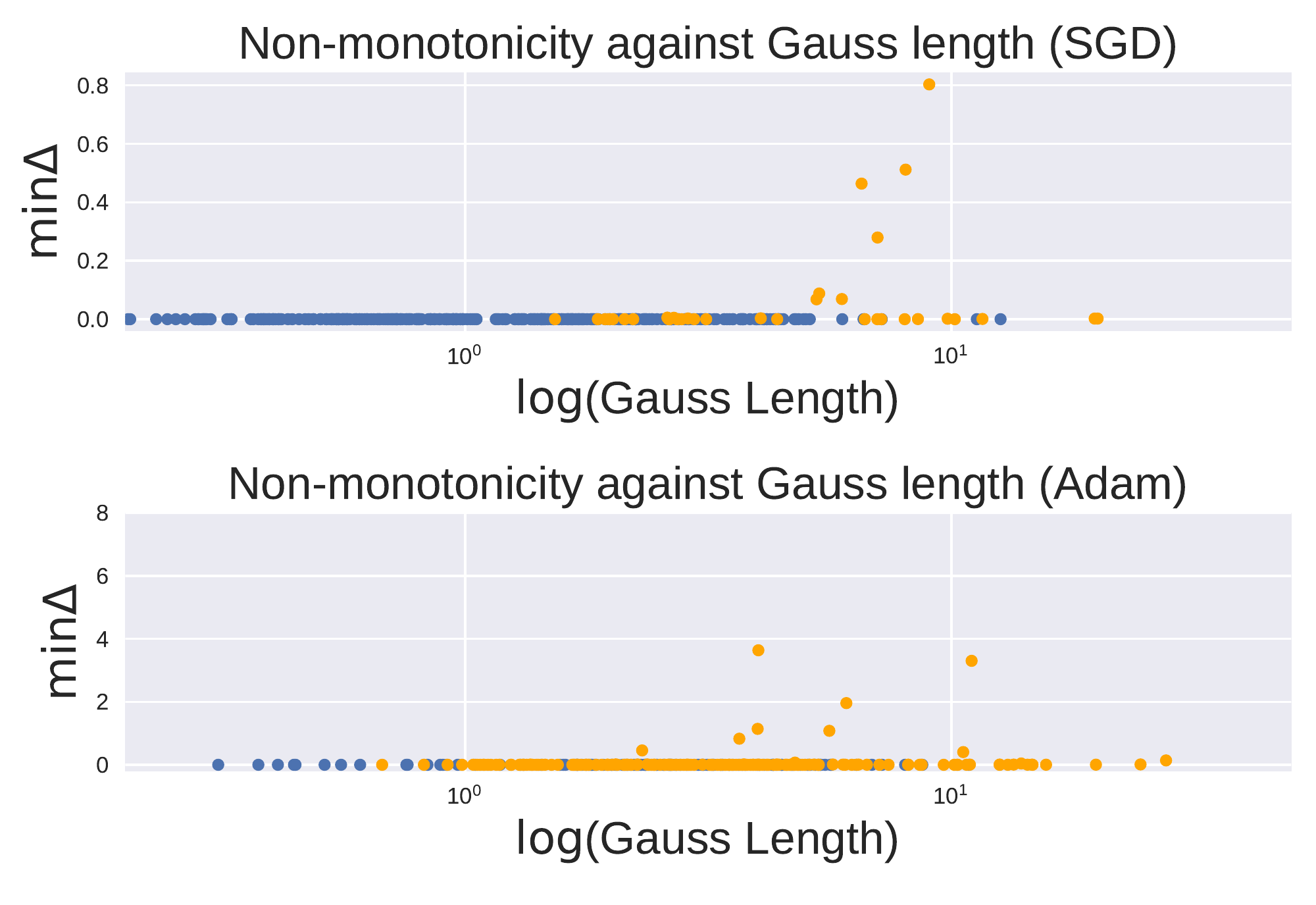}
\end{minipage}
\caption{For each MNIST \& Fashion-MNIST classifier, we compute the minimum $\Delta$ such that the interpolated loss is $\Delta$-monotonic. We plot models trained with SGD and Adam in the top and bottom rows respectively. On the left, we compare the distance moved in the weight space. On the right, we compare the Gauss length of the interpolated network outputs. \textbf{\textcolor{blue}{Blue}} points represent networks where the MLI property holds and \textbf{\textcolor{orange}{orange}} points are networks where the MLI property fails.}
\label{fig:mlp_opt}
\end{figure}

\subsection{Optimizer Ablations}
\label{app:experiments_opt_abl}

To get a better understanding of the influence of different optimization algorithms and the effects of moving in weight and function space, we conduct detailed experiments where we switch the optimizer during training. We use an architecture with 2 hidden layers of 1024  units on MNIST. The architecture has tanh activations and no batch normalization. We report the mean and standard error across five random seeds.

In the (SGD $\rightarrow$ Adam) experiments, we train the first $t=\{2, 10, 50\}$ epochs with SGD, using learning rates in the set $\{0.001, 0.003, 0.01, 0.03, 0.1\}$ and finish the training with Adam (with LR $0.001$) for $200 - t$ epochs. While using just SGD leads to a monotonic interpolation, switching to Adam made all runs not monotonic. These results are reported in Table \ref{tab:switch_sgd_adam}. Similarly, in Table \ref{tab:switch_sgd_adam}, we report results for (Adam $\rightarrow$ SGD). We switch to SGD with a learning rate of 0.03.

As in \citet{amari2020does}, we again find that Adam leads to much greater distance moved in weight space. Perhaps more surprisingly, the distance moved in weight space and the average Gauss length are largely consistent across different choices of the learning rate for SGD and the epoch at which we switch to Adam. In Table~\ref{tab:switch_adam_sgd}, we see that switching from Adam to SGD reduces both the average Gauss length and the distance travelled.

\begin{table}
\centering
\begin{tabular}{ccccc}
\hline
   SGD LR \textbackslash{} switch\_epoch & 2               & 10              & 50              & None            \\
\hline
                   0.001 & 7.0822 ± 0.0898 & 7.3366 ± 0.0034 & 6.9005 ± 0.0861 & 0.5341 ± 0.0069 \\
                   0.003 & 7.2208 ± 0.1517 & 7.0211 ± 0.0816 & 6.8331 ± 0.016  & 0.5773 ± 0.0069 \\
                   0.01  & 7.1144 ± 0.0773 & 7.2666 ± 0.0943 & 7.2307 ± 0.1776 & 0.6939 ± 0.0096 \\
                   0.03  & 7.3067 ± 0.026  & 7.3285 ± 0.0703 & 6.9953 ± 0.1248 & 1.1351 ± 0.0522 \\
                   0.1   & 7.3783 ± 0.1223 & 7.4592 ± 0.0714 & 6.8579 ± 0.1251 & 2.9144 ± 0.0824 \\
\hline
   Distance & 2               & 10              & 50              & None            \\
\hline
                   0.001 & 345.209 ± 2.917 & 340.998 ± 1.385 & 303.907 ± 0.809 & 8.096 ± 0.012   \\
                   0.003 & 346.826 ± 0.962 & 339.54 ± 0.39   & 303.721 ± 0.314 & 9.369 ± 0.015   \\
                   0.01  & 349.413 ± 0.856 & 339.056 ± 0.502 & 302.552 ± 0.368 & 10.893 ± 0.021  \\
                   0.03  & 345.592 ± 0.828 & 339.428 ± 0.608 & 304.109 ± 1.031 & 14.177 ± 0.041  \\
                   0.1   & 349.016 ± 1.009 & 350.103 ± 0.475 & 326.38 ± 0.628  & 108.508 ± 2.497 \\
\hline
\end{tabular}
\caption{Average Gauss length (top) and distance traveled (bottom) for given SGD learning rate and switching to Adam with learning rate 1e-3 during training.}
\label{tab:switch_sgd_adam}
\end{table}

\begin{table}
\centering
\begin{tabular}{ccccc}
\hline
   Adam LR \textbackslash{} switch\_epoch & 2             & 10            & 50            & None           \\
\hline
               0.001 & 1.447 ± 0.036 & 2.349 ± 0.024 & 4.472 ± 0.061 & 7.135 ± 0.048  \\
               0.003 & 3.766 ± 0.052 & 7.325 ± 0.071 & 9.365 ± 0.115 & 9.933 ± 0.097  \\
               0.01  & 6.156 ± 0.266 & 7.391 ± 0.423 & 9.873 ± 0.427 & 10.313 ± 0.231 \\
\hline

   Distance & 2              & 10               & 50              & None             \\
\hline
               0.001 & 22.257 ± 0.113 & 59.986 ± 0.089   & 174.958 ± 0.061 & 350.174 ± 0.704  \\
               0.003 & 58.773 ± 0.565 & 196.794 ± 0.773  & 420.652 ± 1.321 & 659.236 ± 2.785  \\
               0.01  & 185.892 ± 2.26 & 384.983 ± 13.863 & 726.063 ± 8.888 & 1220.432 ± 6.893 \\
\hline
\end{tabular}
\caption{Average Gauss length (top) and distance traveled (bottom) for given Adam learning rate and switching to SGD with learning rate 0.03 during training.}
\label{tab:switch_adam_sgd}
\end{table}

To investigate why Adam is responsible for breaking monotonicity further, we follow the ``grafting" experiment described in \cite{agarwal2020disentangling}, where two optimizers are combined by using the step magnitude from the first and step direction from the second. Results where we use the SGD step magnitude (which varies with LR) and Adam direction are shown in \ref{tab:grafting}. All the runs are monotonic, so the direction chosen by Adam is not the primary influence on the optimization trajectory. In contrast, when we use the SGD step direction and the Adam magnitude we observe all runs to be non-monotonic and find the average distance traveled to be 381.65, suggesting that the magnitude of the updates is responsible for breaking monotonicity.

\begin{table}
\centering
\begin{tabular}{rll}
\hline
   lr \textbackslash{} Optimizer & SGD           & Adam         \\
\hline
              0.001 & 9.198 ± 0.016  & 8.096 ± 0.012   \\
               0.003 & 10.932 ± 0.015 & 9.369 ± 0.015   \\
               0.01  & 12.978 ± 0.013 & 10.893 ± 0.021  \\
               0.03  & 16.619 ± 0.037 & 14.177 ± 0.041  \\
\hline
\end{tabular}
\caption{Average distance traveled where we use the SGD step magnitude and step direction given by SGD or Adam, respectively}
\label{tab:grafting}
\end{table}

\subsection{Additional weight distance experiments}
\label{app:experiments_weight_dis}

In this section, we investigate the relationship between normalized weight distance and non-monotonicity on image reconstruction (MNIST) and image classification (CIFAR-10 \& CIFAR-100) tasks.

\begin{figure*}[!ht]
\begin{minipage}{0.49\linewidth}
        \centering%
        \includegraphics[width=\linewidth]{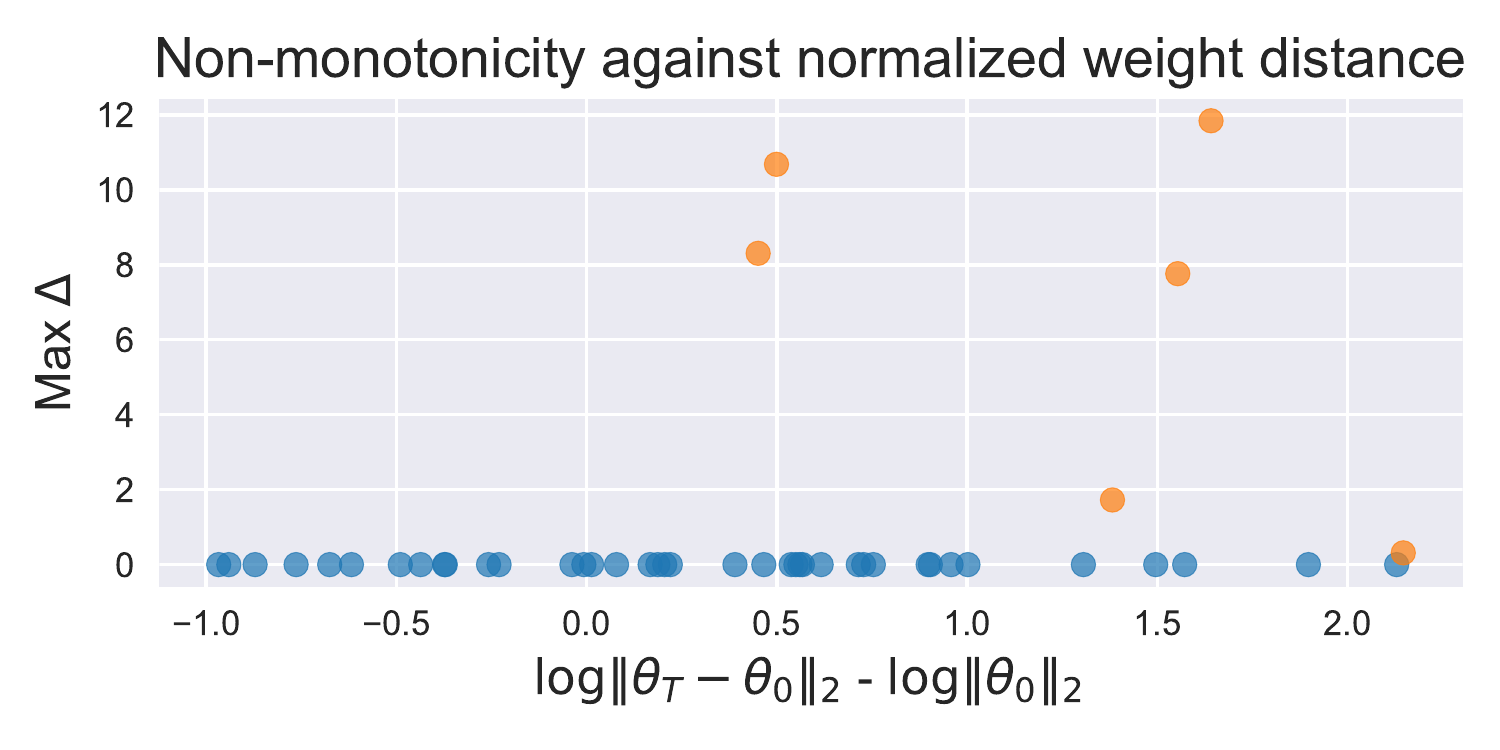}%
    \end{minipage}\hfill%
    \begin{minipage}{0.49\linewidth}
        \centering%
        \includegraphics[width=\linewidth]{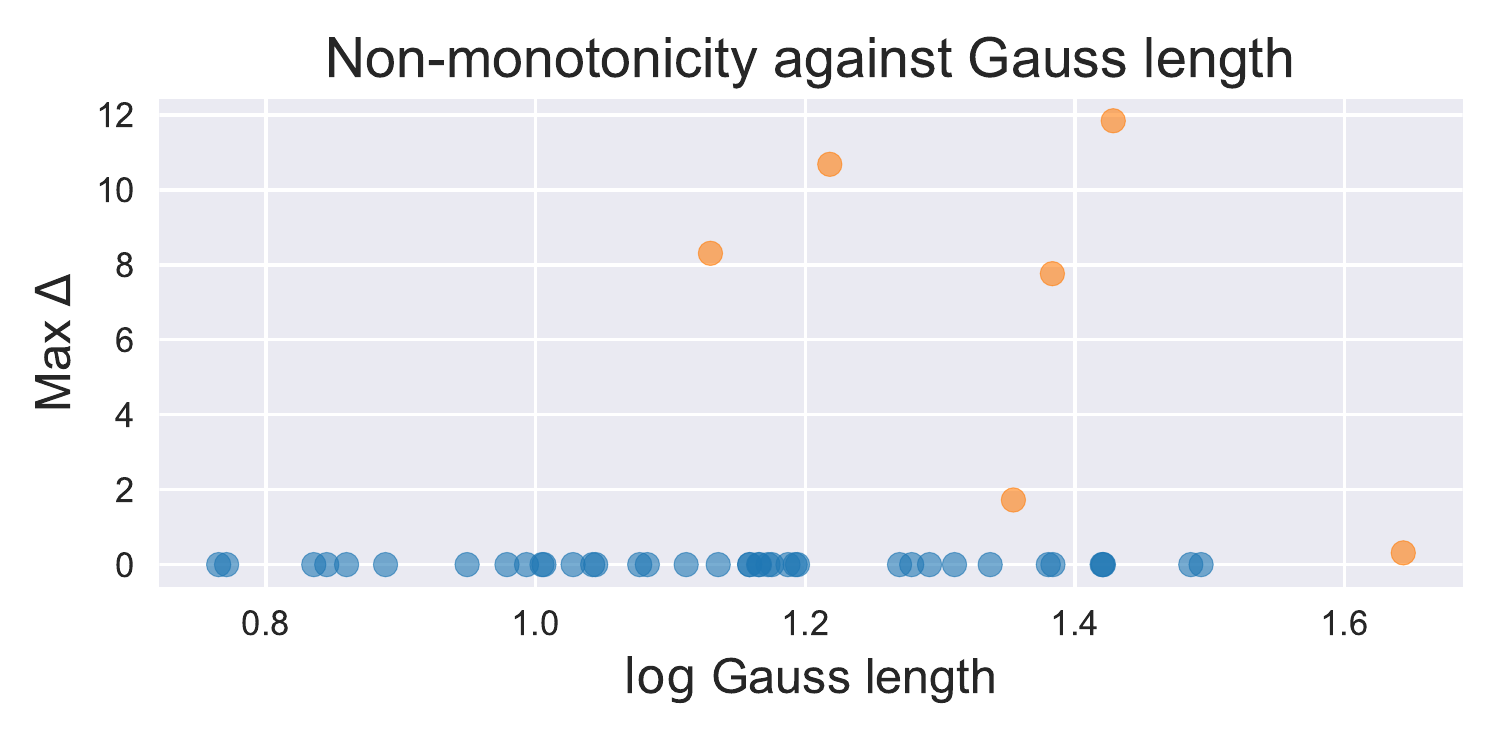}%
    \end{minipage}
    \caption{Weight distance (left) and Gauss length (right) against maximum non-monotonic bump height for image reconstruction task. For clarity, \textbf{\textcolor{blue}{Blue}} points represent networks where the MLI property holds and \textbf{\textcolor{orange}{orange}} points are networks where the MLI property fails.}
    \label{fig:ae_metrics}
\end{figure*}

In Figure~\ref{fig:ae_metrics}, we show the correlation between the distance travelled in parameter space and the smallest $\Delta$ such that the loss interpolation is $\Delta$-monotonic (Definition~\ref{def:delta_mono}) for image reconstruction task. As expected,  a small distance moved (or Gauss length, respectively) leads to monotonic interpolation. And beyond the strict limits of our analysis, we observed that larger weight distances are correlated with non-monotonicity.

In Figure~\ref{fig:cifar_wdist_delta}, we display the distance moved in weight space against the minimum $\Delta$ such that ResNets trained on CIFAR-10 and CIFAR-100 have $\Delta$-monotonic interpolations from initial to final parameters. In general, larger bumps occur at larger distances moved, as in our other experiments. 
\begin{figure}
    \begin{minipage}{0.5\linewidth}
    \centering
    \includegraphics[width=\linewidth]{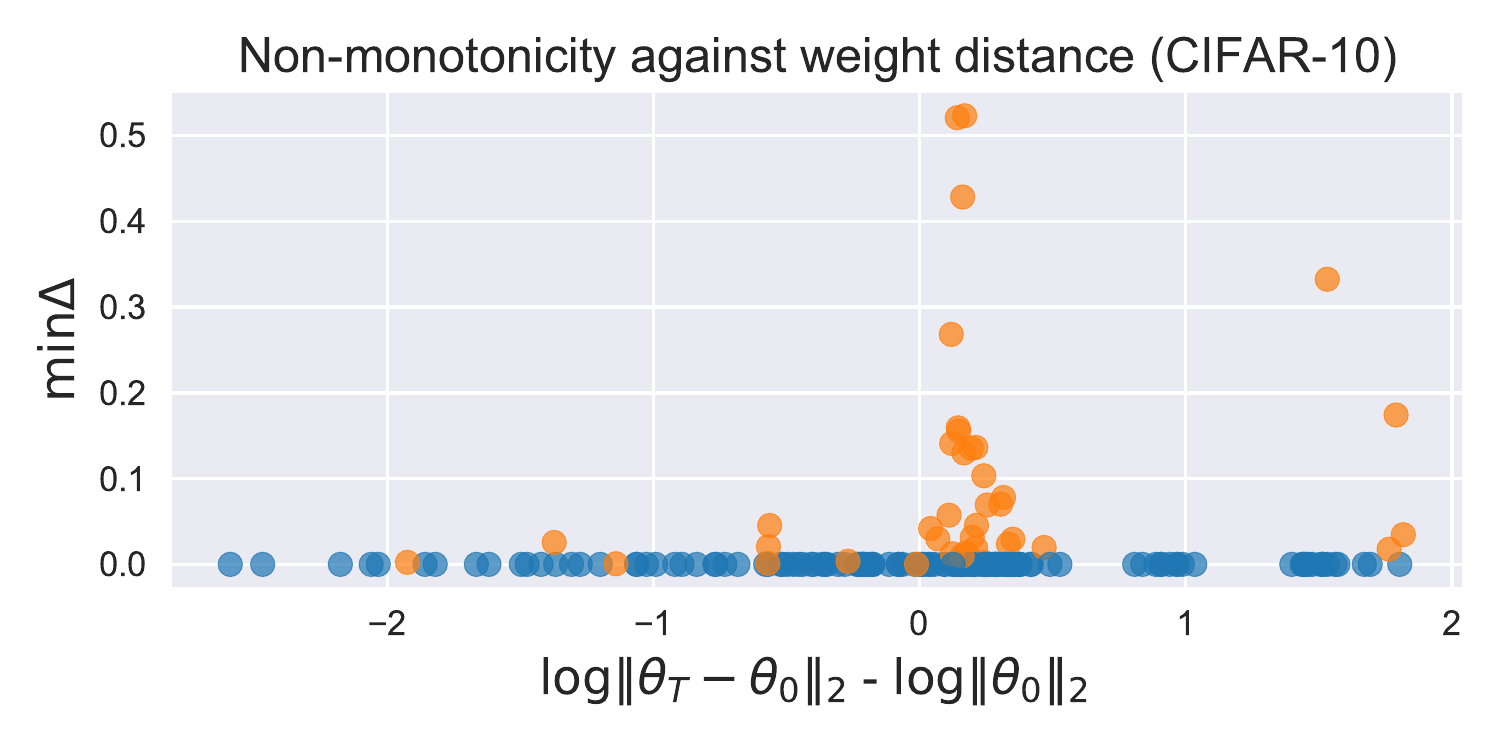}
    \end{minipage}\hfill%
    \begin{minipage}{0.5\linewidth}
    \centering
    \includegraphics[width=\linewidth]{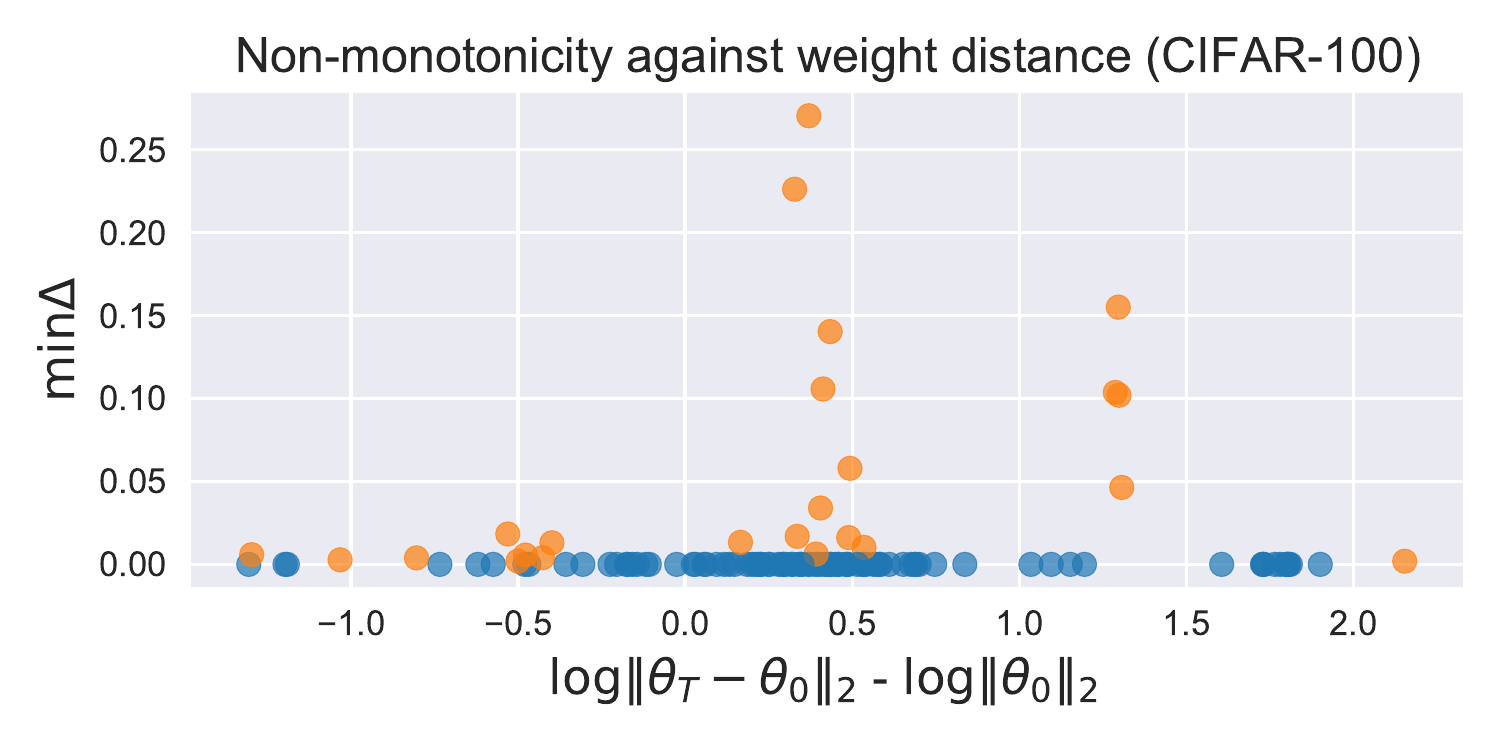}
    \end{minipage}
    \caption{Distance moved in weight space against the minimum $\Delta$ such that ResNets trained on CIFAR-10 (left) and CIFAR-100 (right) have $\Delta$-monotonic interpolations from initialization to final parameters. We observe a general trend that larger distance in weight space corresponds to more significant non-monotonicities.}
    \label{fig:cifar_wdist_delta}
\end{figure}

\subsection{Additional Gauss length experiments}
\label{app:experiments_gl}

\begin{wrapfigure}[12]{l}{0.5\textwidth}
    \centering
    \vspace{-0.5cm}
    \includegraphics[width=\linewidth]{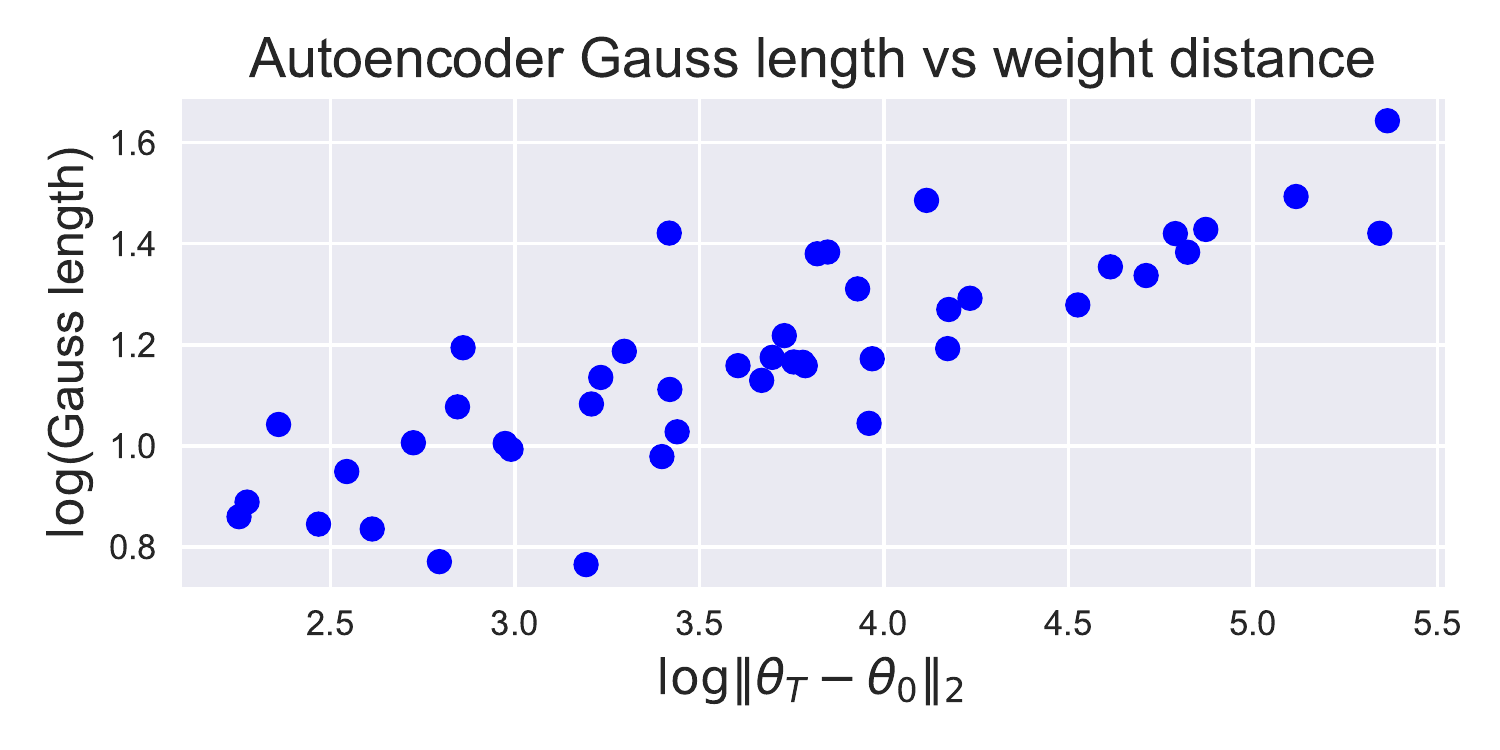}
    \vspace{-0.8cm}
    \caption{Power law relationship between Gauss length and weight distance for autoencoders trained on MNIST ($R^2 = 0.705$).}
    \label{fig:ae_power_law}
\end{wrapfigure}

In this section, we investigate the relationship between Gauss length and non-monotonicity on image reconstruction and image classification (CIFAR-10 \& CIFAR-100) tasks.

In Figure~\ref{fig:ae_metrics}, we show the correlation between the Gauss length and the smallest $\Delta$ such that the loss interpolation is $\Delta$-monotonic for the image reconstruction task. Similar to the weight distances, a small Gauss length leads to monotonic interpolation. We also observe that larger Gauss length are correlated with the failure of MLI property.

In Figure~\ref{fig:cifar_wdist_delta}, we display the distance moved in weight space against the minimum $\Delta$ such that ResNets trained on CIFAR-10 and CIFAR-100 have $\Delta$-monotonic interpolations from initial to final parameters. In general, larger bumps occur at larger distances moved, as in our other experiments. 
\begin{figure}
    \begin{minipage}{0.5\linewidth}
    \centering
    \includegraphics[width=\linewidth]{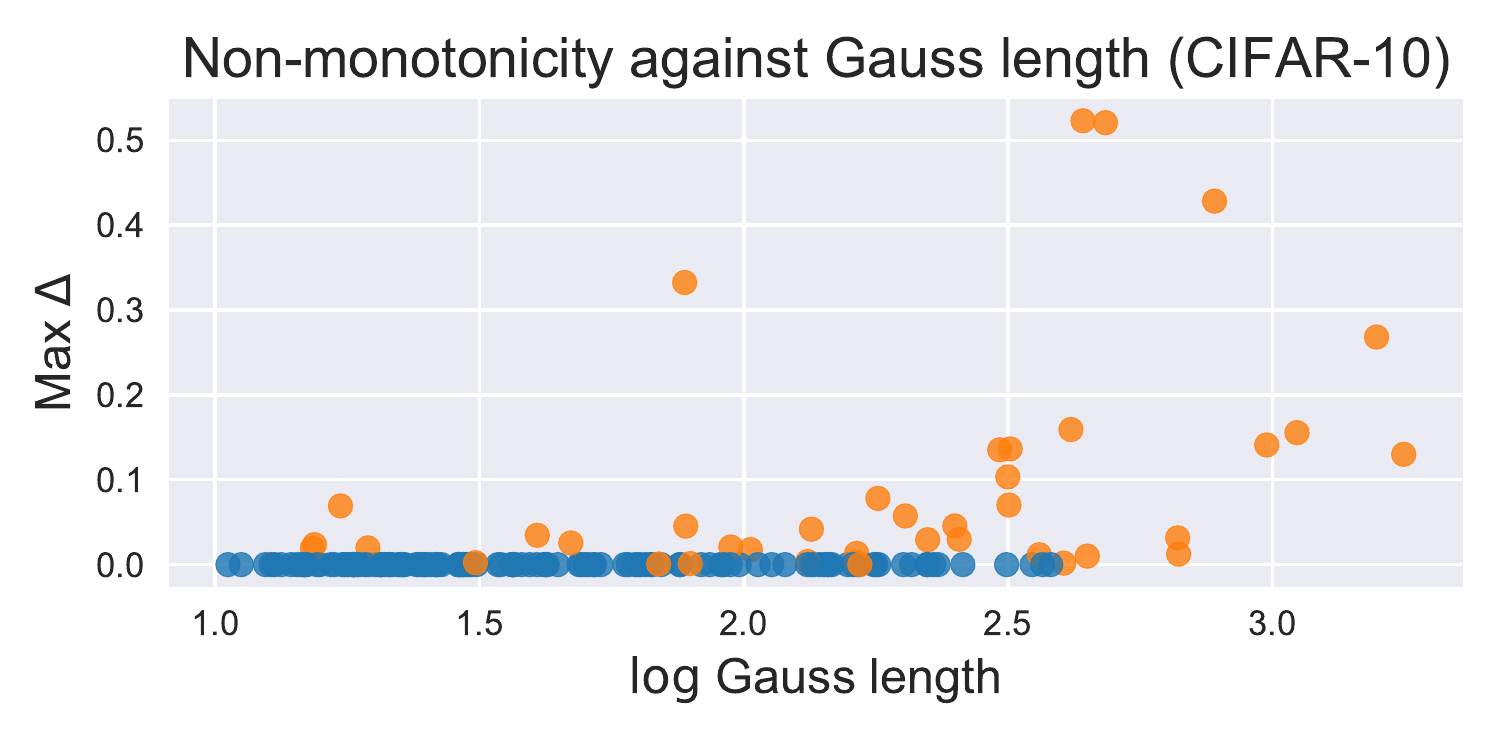}
    \end{minipage}\hfill%
    \begin{minipage}{0.5\linewidth}
    \centering
    \includegraphics[width=\linewidth]{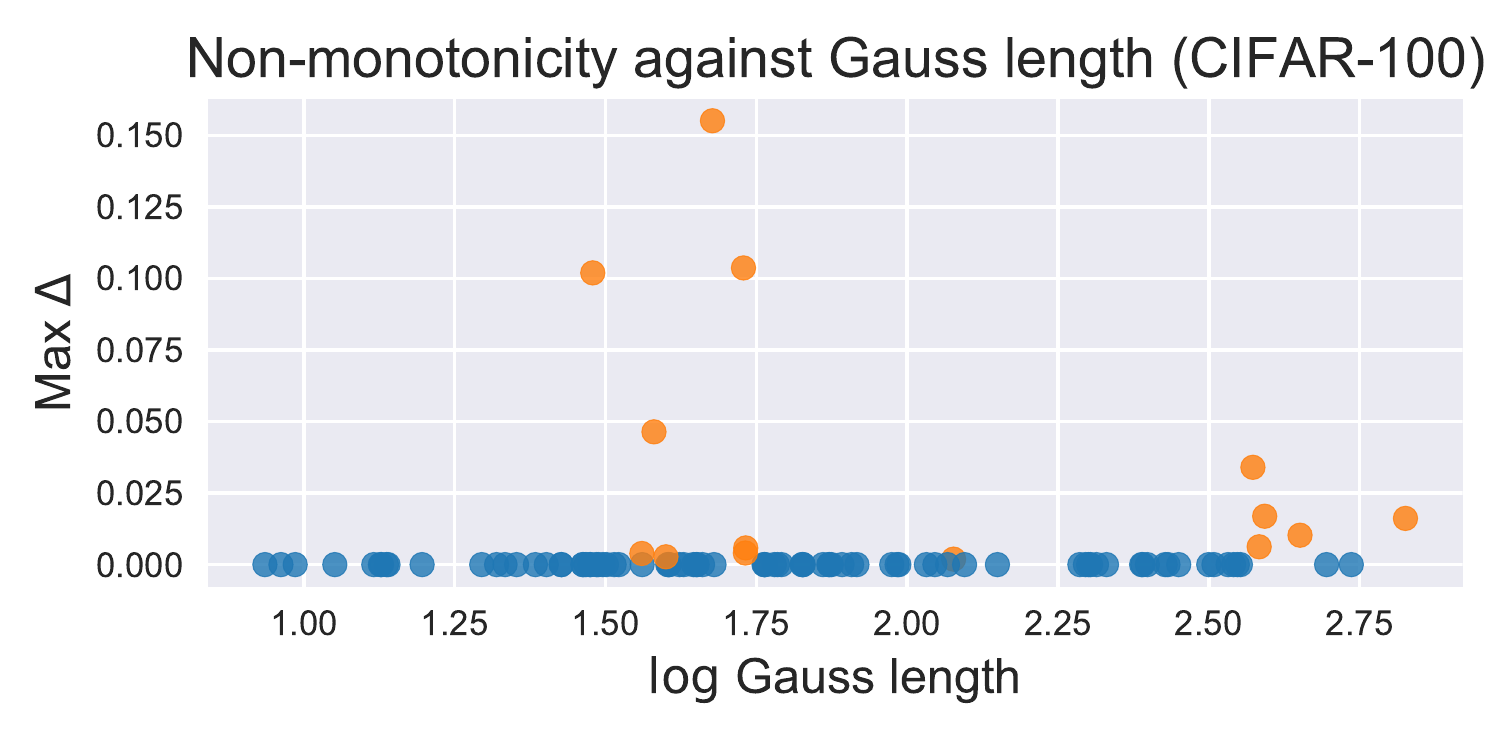}
    \end{minipage}
    \caption{Gauss length of function-space interpolation path against the minimum $\Delta$ such that ResNets trained on CIFAR-10 (left) and CIFAR-100 (right) have $\Delta$-monotonic interpolations. At small Gauss lengths, the networks generally satisfy the MLI property while larger bumps in the interpolation path are achieved for interpolations with larger Gauss lengths.}
    \label{fig:cifar_gl_delta}
\end{figure}

\subsection{Additional Gauss length vs weight distance}
\label{app:experiments_gl_wd}

In this section, we investigate the relationship between Guass length and weight distance travelled for image reconstruction and image classification (CIFAR-10 \& CIFAR-100) tasks.

In Figure~\ref{fig:ae_power_law}, we plot the Gauss length of the interpolation path against the distance moved in weight space for autoencoders trained on MNIST. In this case, as in Figure~\ref{fig:mlp_power_law}, we observe a clear power-law relationship between the two.

In Figure~\ref{fig:cifar_power_law}, we plot the Gauss length of the interpolation path against the distance moved in weight space for ResNets trained on CIFAR-10 and CIFAR-100, over varying initialization schemes, optimizers, and the use of batchnorm (as in Tables~\ref{tab:cifar10_resnets}~and~\ref{tab:cifar100_resnets}). In this case, there is not a clear power law relationship but nonetheless a clear positive correlation remains between the Gauss length and the distance moved in weight space.
\begin{figure}
    \begin{minipage}{0.5\linewidth}
    \centering
    \includegraphics[width=\linewidth]{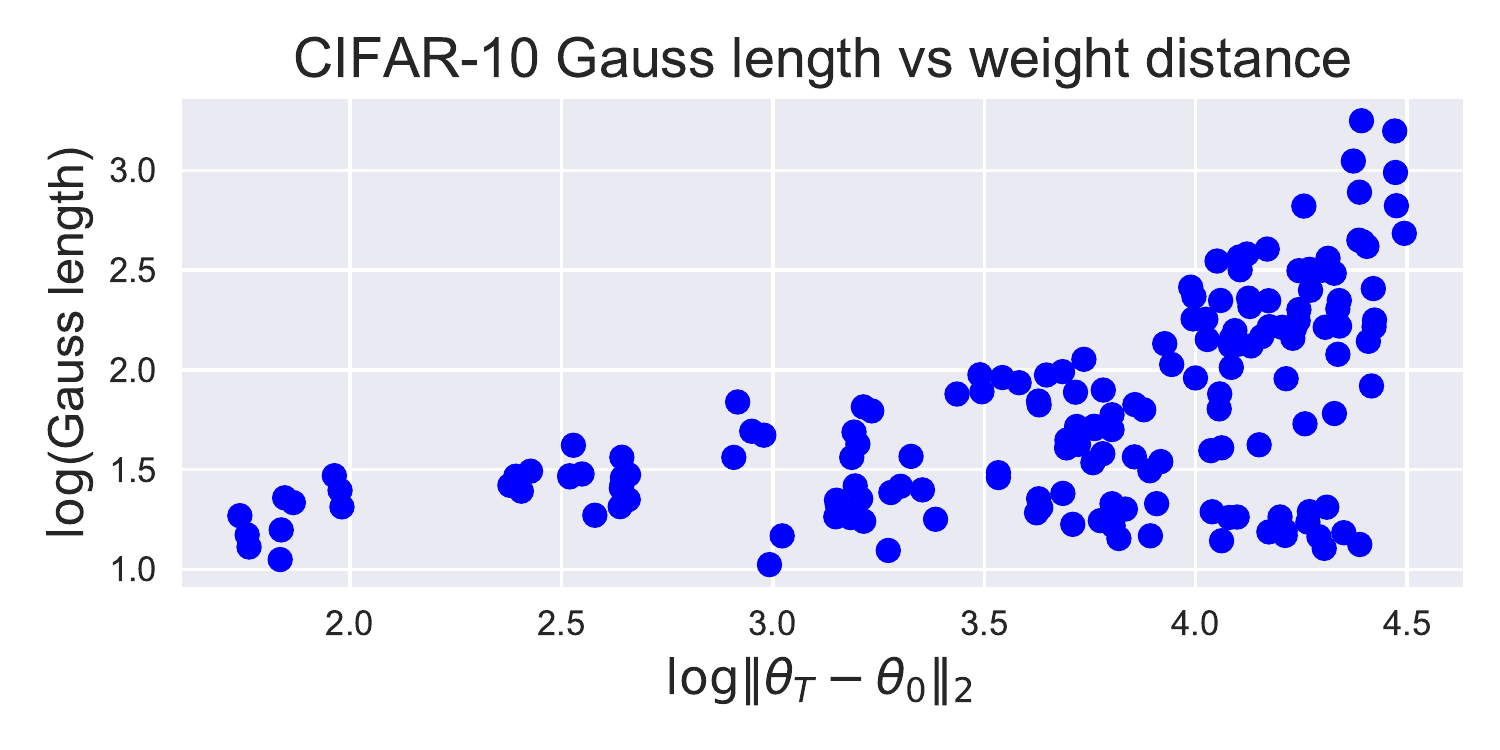}
    \end{minipage}\hfill%
    \begin{minipage}{0.5\linewidth}
    \centering
    \includegraphics[width=\linewidth]{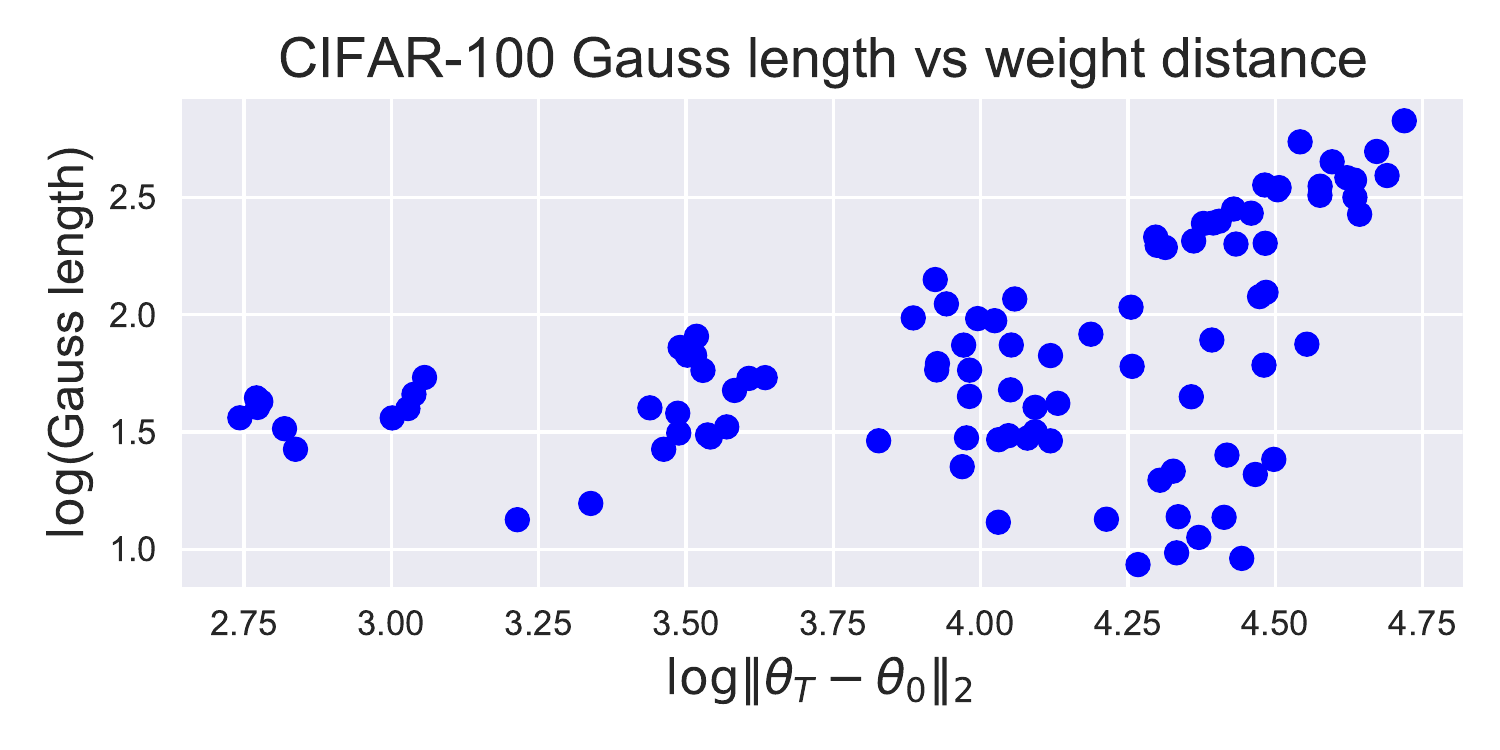}
    \end{minipage}
    \caption{Gauss length of interpolation path against distance moved in weight space for ResNets trained on CIFAR-10 and CIFAR-100. While there is a positive correlation, the goodness of fit is lower for these networks than the MNIST classifiers and autoencoders ($R^2 = 0.399$ for CIFAR-10 and $R^2=0.402$ for CIFAR-100).}
    \label{fig:cifar_power_law}
\end{figure}

\subsection{Impact of batch normalization}

In Figure~\ref{fig:mlp-bn}, we compare the distance travelled between models trained with batch normalization and without batch normalization for classifiers trained on the MNIST \& Fashion-MNIST datasets. We plot the minimum $\Delta$ such that the interpolated loss is $\Delta$-monotonic against the distance moved in weight space. We observe that models trained with batch normalization had a higher variance of distance travelled in weight space compared to models trained without batch normalization. Hence, when batch normalization is used, there are more configurations that travelled further in weight space. Consistent with our prior analysis, configurations that travelled far in parameter space tend to more break the MLI property. This hints that such behaviour of batch normalization can cause more frequent violation of the MLI property.

\begin{figure}
    \vspace{-0.8cm}
    \centering
    \includegraphics[width=0.8\linewidth]{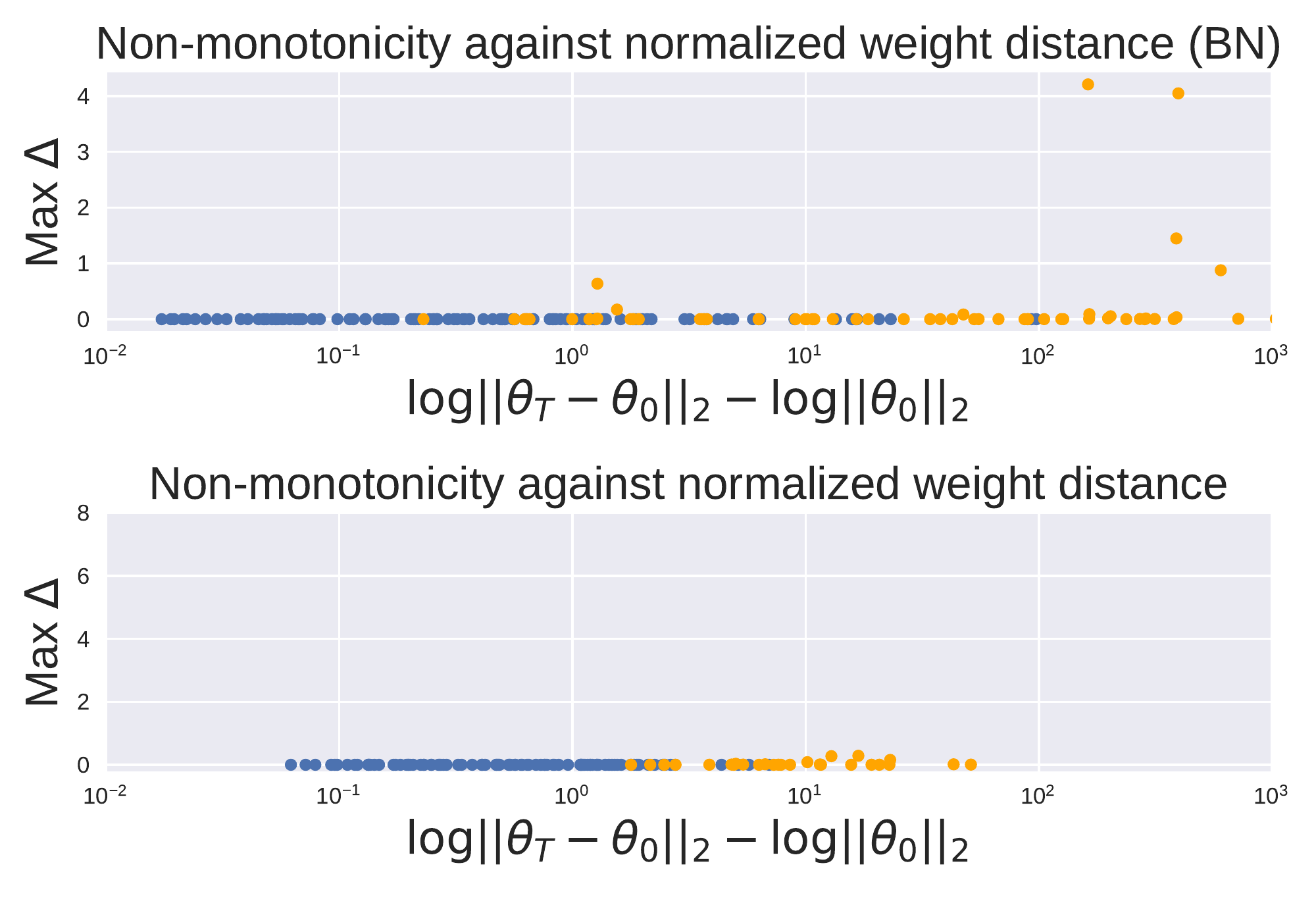}
    \vspace{-0.5cm}
    \caption{Monotonicity against distance moved in weight space for MNIST \& Fashion-MNIST classifiers. \textbf{\textcolor{blue}{Blue}} points represent networks where the MLI property holds and \textbf{\textcolor{orange}{orange}} points are networks where the MLI property fails.}
    \label{fig:mlp-bn}
\end{figure}

\subsection{Additional loss landscape experiments}
\label{app:experiments_ll}

\paragraph{Loss landscape for network failing MLI.} In Section~\ref{sec:exp:what_landscape}, we showed loss landscape visualizations for networks that satisfied the MLI property. In Figure~\ref{fig:nonmono_loss_landscape}, we show the 2D projection of the loss landscape for a fully-connected FashionMNIST classifier that does not satisfy the MLI property. In this 2D slice, we observe a wide barrier in the loss landscape followed by a region of extremely flat curvature.
\begin{figure}
    \centering
    \includegraphics[width=0.9\linewidth]{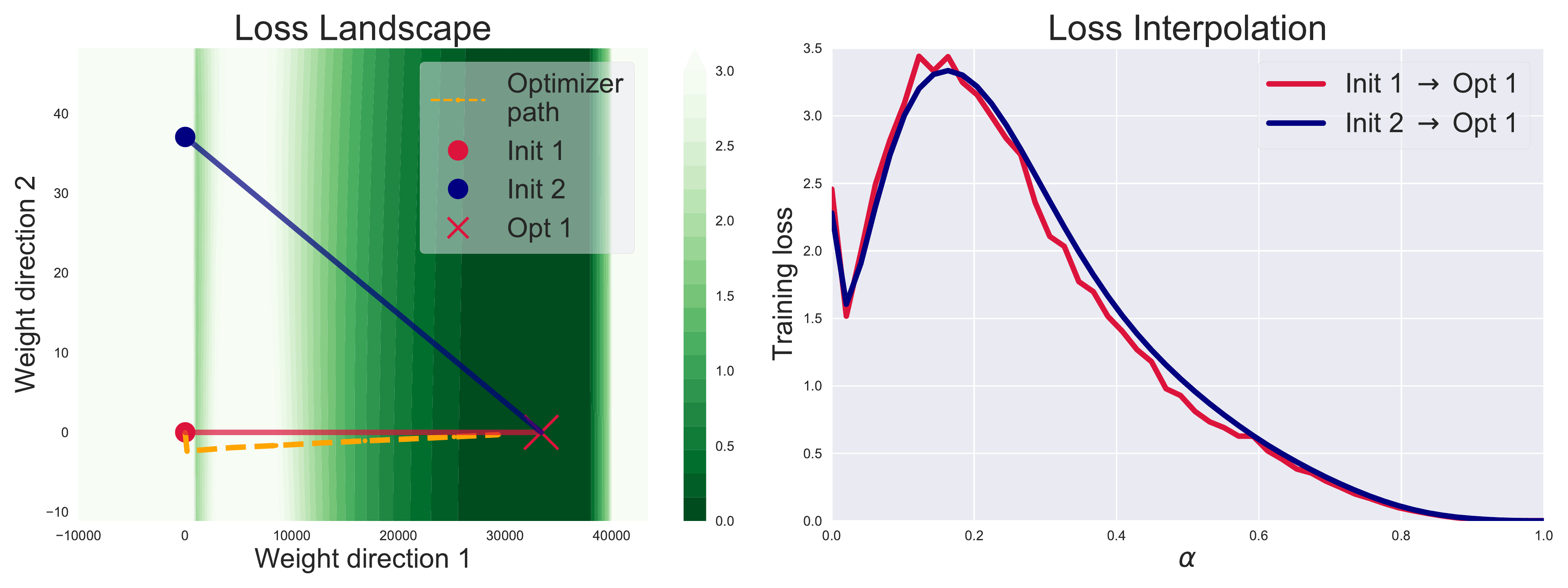}
    \vspace{-0.5cm}
    \caption{Loss landscape projection of a FashionMNIST classifier that does not satisfy the MLI property. We observe a barrier in the loss followed by a region of extremely flat curvature.}
    \label{fig:nonmono_loss_landscape}
\end{figure}

\begin{figure}
    \vspace{-0.5cm}
    \centering
    \includegraphics[width=0.5\linewidth]{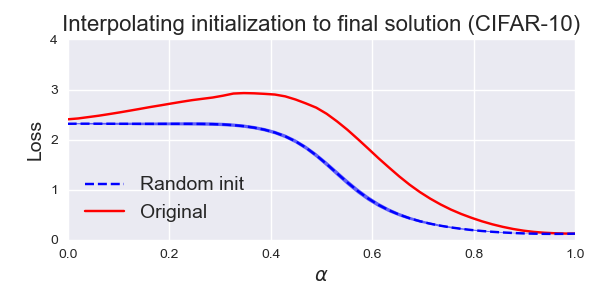}
    \vspace{-0.5cm}
    \caption{Interpolating from random intializations to SGD solution, with original initialization-solution pair being non-monotonic. The initialization scheme used differs from the one used to train the network, and surprisingly leads to monotonic interpolations.}
    \label{fig:cifar_rand_init_interp}
\end{figure}

\paragraph{MLI over permutation symmetries.} In addition to random initializations, we explored interpolations over the permutation symmetry group of initialization-solution pairs for fully-connected networks on MNIST. In Figure~\ref{fig:permuting_mli}, we utilized the fact that adjacent linear layers can be permuted without modifying the output function to randomly permute the initialization and final solution. This leads to different paths through weight space but with the end-points of the interpolation fixed at the original values. We observed that these permutations preserve the MLI property.
\begin{figure}
\begin{minipage}{0.48\linewidth}
    \centering
    \includegraphics[width=\linewidth]{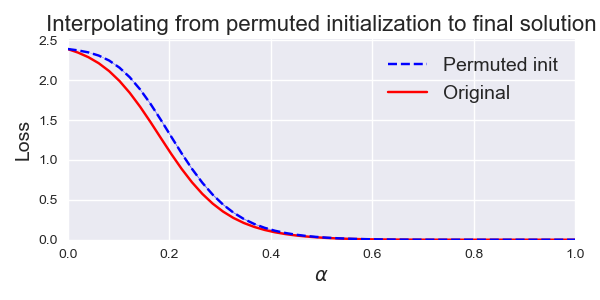}\\
\end{minipage}\hfill%
\begin{minipage}{0.48\linewidth}
    \centering
    \includegraphics[width=\linewidth]{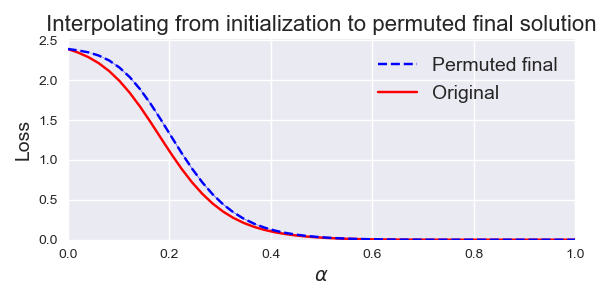}
\end{minipage}
    \centering
    \caption{Interpolation loss between initial points and final solution on training set. (Left) random permutations of the initialization are shown. (Right) Random permutations of the solution are shown. Mean loss shown with ($\pm 1$) standard deviation as filled region.}
    \label{fig:permuting_mli}
\end{figure}

\paragraph{Interpolations with different initialization distributions.} In Figure~\ref{fig:rand_init_to_optimum}, we showed that the monotonic (or non-monotonic) interpolations persist across different random initializations for a given final solution. However, it is possible that the monotonicity of the interpolations can change if we modify the initialization scheme. We took the network from the bottom right plot of Figure~\ref{fig:rand_init_to_optimum} and chose our initializations according to the scheme described in \citet{goyal2017accurate} --- where the final batch norm layer in each residual block is initialized to be zero so that the network function is close to the identity function. The result is shown in Figure~\ref{fig:cifar_rand_init_interp}, in this case, the random initializations are linearly connected to the solution while random samples from the original initialization distribution are not.

\paragraph{Additional landscape visualizations.} In Figure~\ref{fig:cut_resnet_cifar10_2inits_2optima}, we show additional 2D projections of the loss landscapes for ResNet20v1 networks on CIFAR-10. This confirms results seen elsewhere: linear interpolation between unrelated initial points and optima yield monotonic decreases in training loss and monotonic increases in accuracy.

\begin{figure}[!htb]
    \begin{minipage}{0.5\linewidth}
        \centering%
        \includegraphics[width=\linewidth]{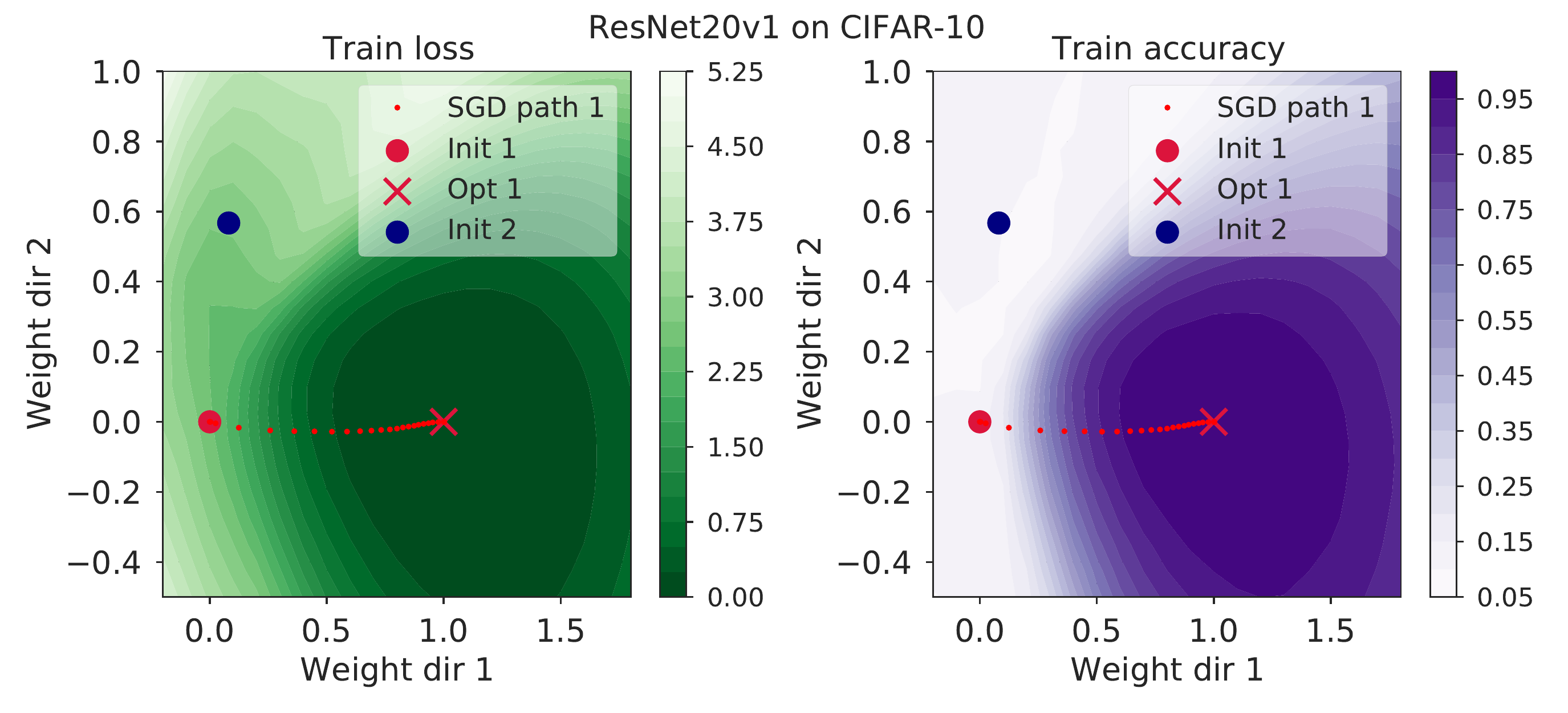}
    \end{minipage}\hfill%
    \begin{minipage}{0.5\linewidth}
        \centering%
        \includegraphics[width=\linewidth]{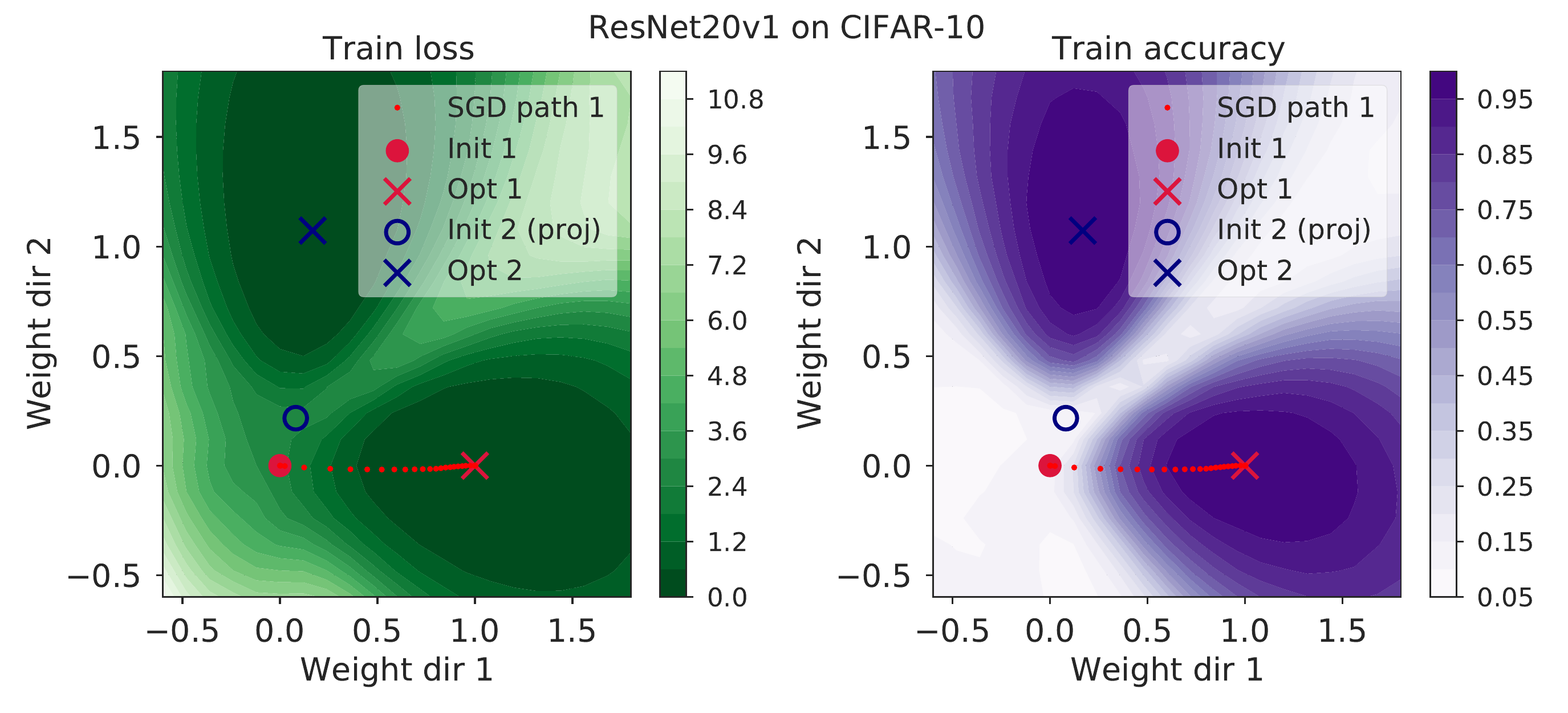}
    \end{minipage}
    \caption{\label{fig:cut_resnet_cifar10_2inits_2optima} Two-dimensional sections of the weight space for ResNet20v1 trained on CIFAR-10. (\textbf{Left}) The plane is defined by 2 initializations (circles) and an optimum (cross) reached from one of them. (\textbf{Right}) The plane is defined by an initialization (circle) and two optima (crosses). For both training loss ({\color{green} \textbf{green}}) and training accuracy ({\color{purple} \textbf{purple}}), interpolations between both minima and optima yield monotonic decreases/increases, respectively.}
\end{figure}

In Figure~\ref{fig:roberta_esperanto_landscape_interp}, we show the same loss landscape projections as displayed in Figure~\ref{fig:cut_roberta_esperanto}. Additionally, here we include the loss over the interpolated paths.
\begin{figure}
\begin{minipage}{0.5\linewidth}
    \centering
    \includegraphics[width=\linewidth]{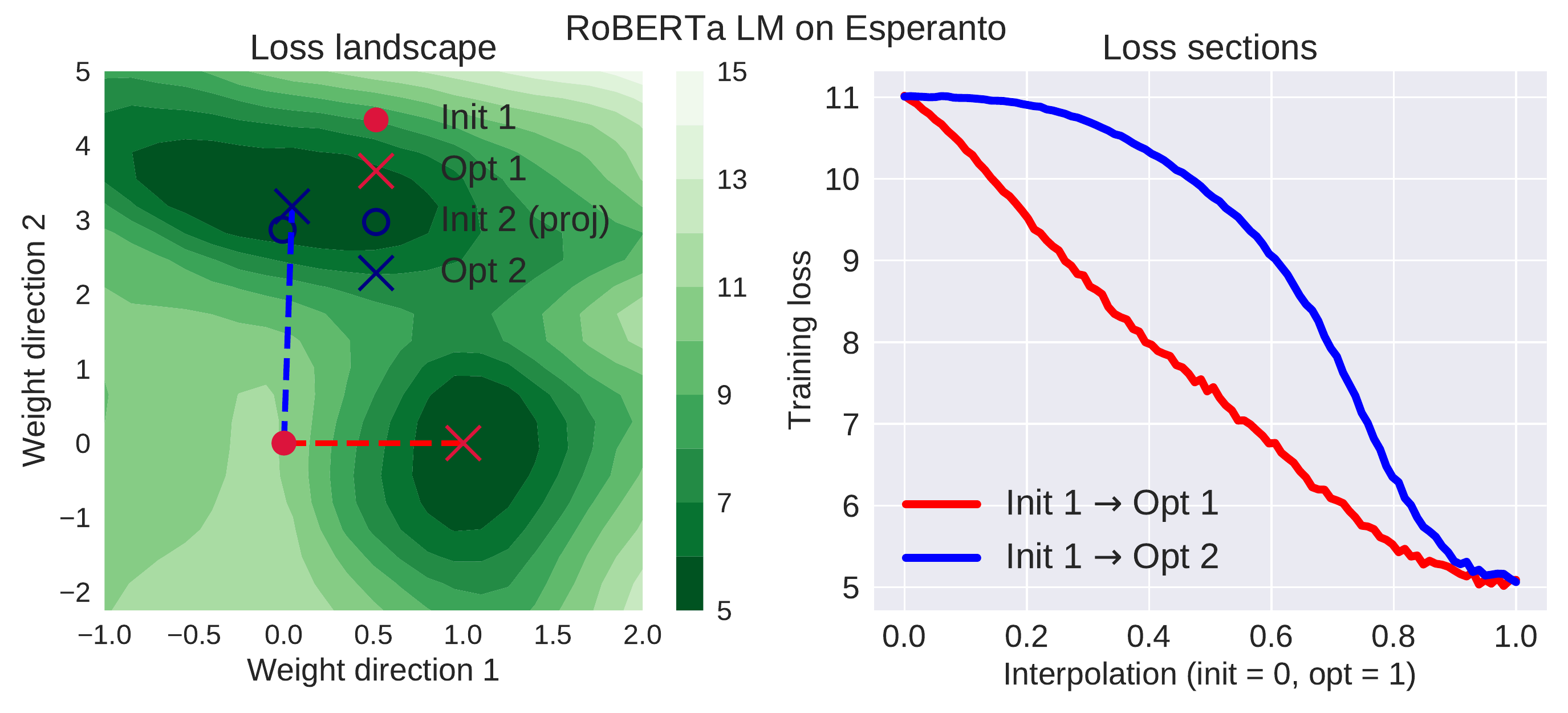}
\end{minipage}\hfill%
\begin{minipage}{0.5\linewidth}
    \centering
    \includegraphics[width=\linewidth]{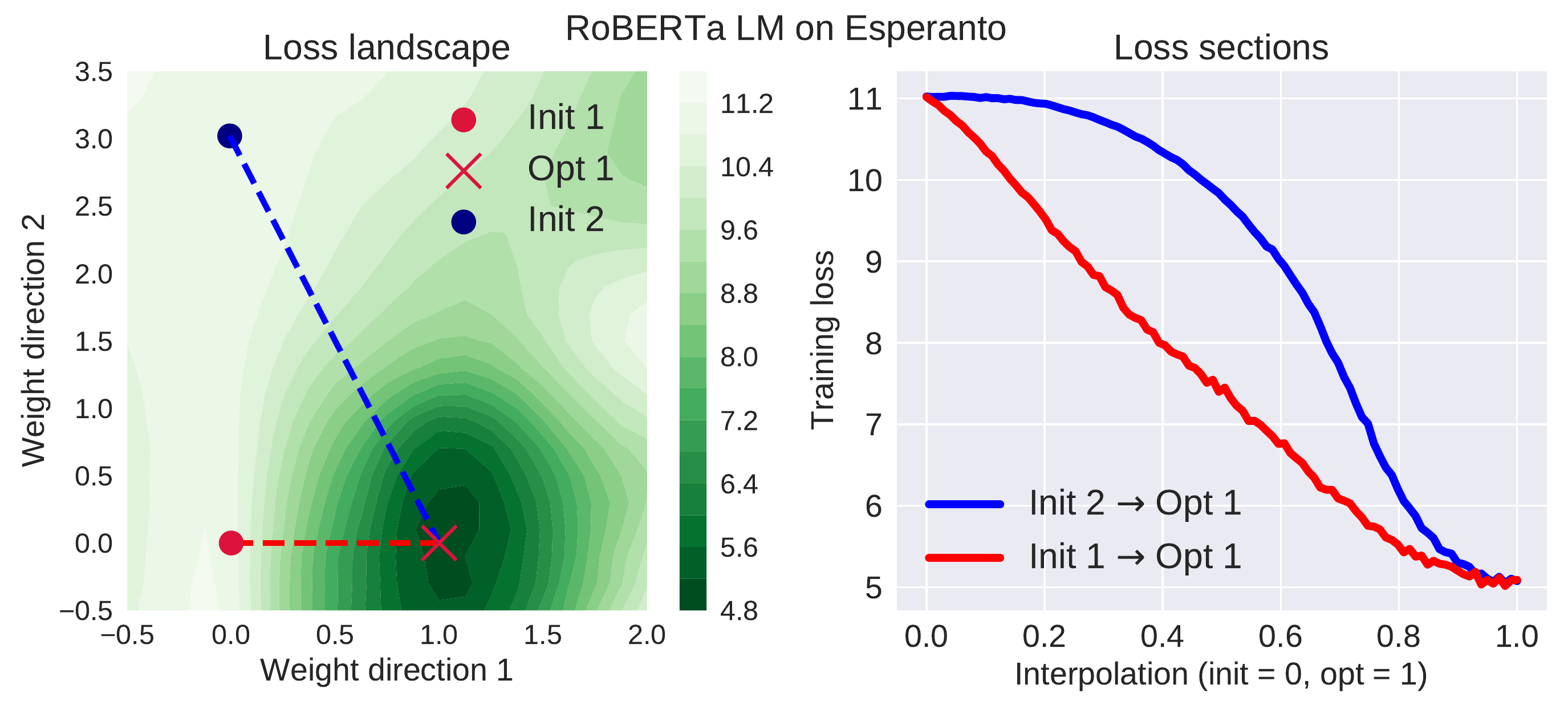}
\end{minipage}\hfill%
    \caption{Loss landscapes with loss over linear interpolations between initial and final parameters for RoBERTa trained as language model on Esperanto. Linear interpolation between all pairs leads to monotonic reduction in the loss. (\textbf{Left}) The loss over the plane defined by the initial parameters, optimum found by SGD, and an unrelated optimum. (\textbf{Right}) The loss over the plane defined by the initial parameters, optimum found by SGD, and an unrelated initialization.}
    \label{fig:roberta_esperanto_landscape_interp}
\end{figure}

\subsection{Additional MNIST results}\label{app:mnist_additional}

In Figure~\ref{fig:784_all}, we show the full set of network interpolations used to produce Table~\ref{tab:mnist_lr}. Overall, we observed a significant effect from introducing batch normalization across all other settings considered, but particularly when the learning rate is large.
\begin{figure}
    \begin{minipage}{0.5\linewidth}
        \centering
        \includegraphics[width=\linewidth]{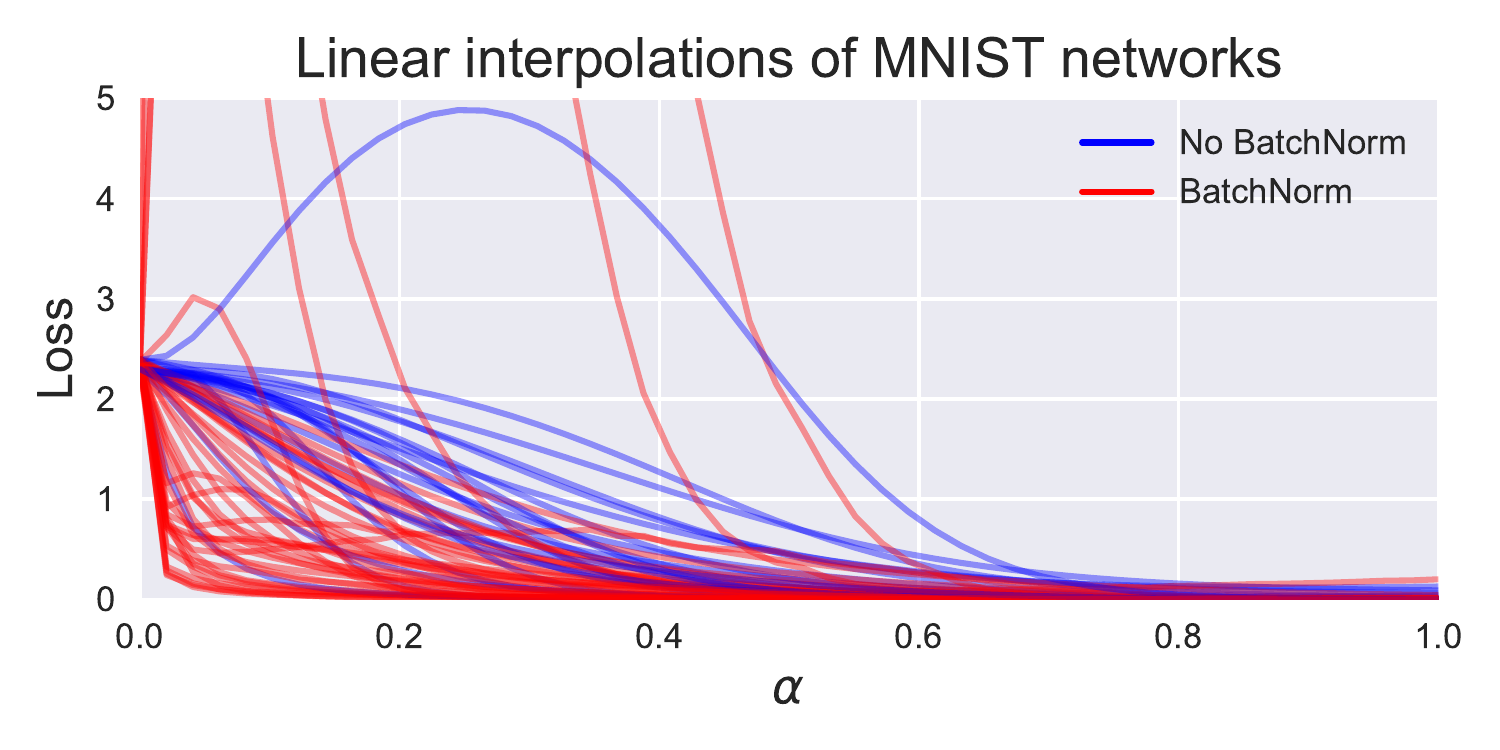}    
    \end{minipage}\hfill%
    \begin{minipage}{0.5\linewidth}
        \centering
        \includegraphics[width=\linewidth]{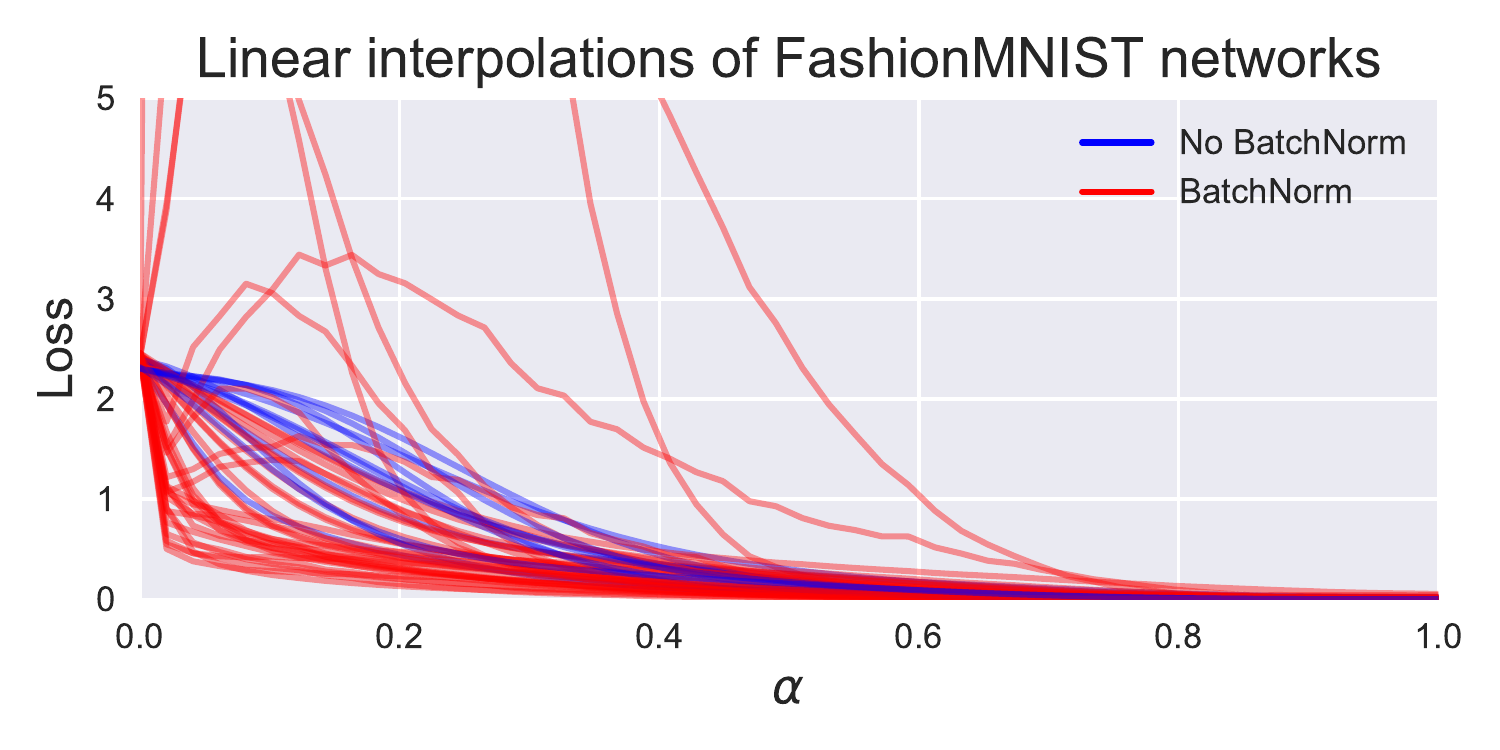}
    \end{minipage}
    \caption{Linear interpolations for MNIST (top) and Fashion-MNIST (bottom). Different curves represent trained networks with varying activation function, learning rate, choice of optimizer, and batch normalization. All networks achieve at most 0.1 final training loss.}
    \label{fig:784_all}
\end{figure}

\paragraph{Varying depth and hidden size.} We explored the effect of varying depth and hidden size on the MLI property. Overall, we did not observe any substantial correlation between these factors and the MLI property (especially, when taking into account implicit effects on the critical learning rate).

In Figure~\ref{fig:mnist_delta_heatmaps}, we display heatmaps of $\min \Delta$ as a function of the learning rate, hidden size and depth of fully-connected neural networks trained on MNIST and Fashion-MNIST. Overall, we do not observe any significant effect from changing either the hidden size or the network depth --- the learning rate accounts for the dominant changes in the monotonicity of the interpolation. We trained each network with ReLU activations for 200 epochs with batch sizes of 512, using both Adam and SGD and with/without batch normalization. Only those models that achieved a training loss of 0.1 are displayed (cyan patches indicate that no model met these criteria for the corresponding configuration).
\begin{figure}
    \centering
    \resizebox{\linewidth}{!}{%
    \includegraphics[height=5cm]{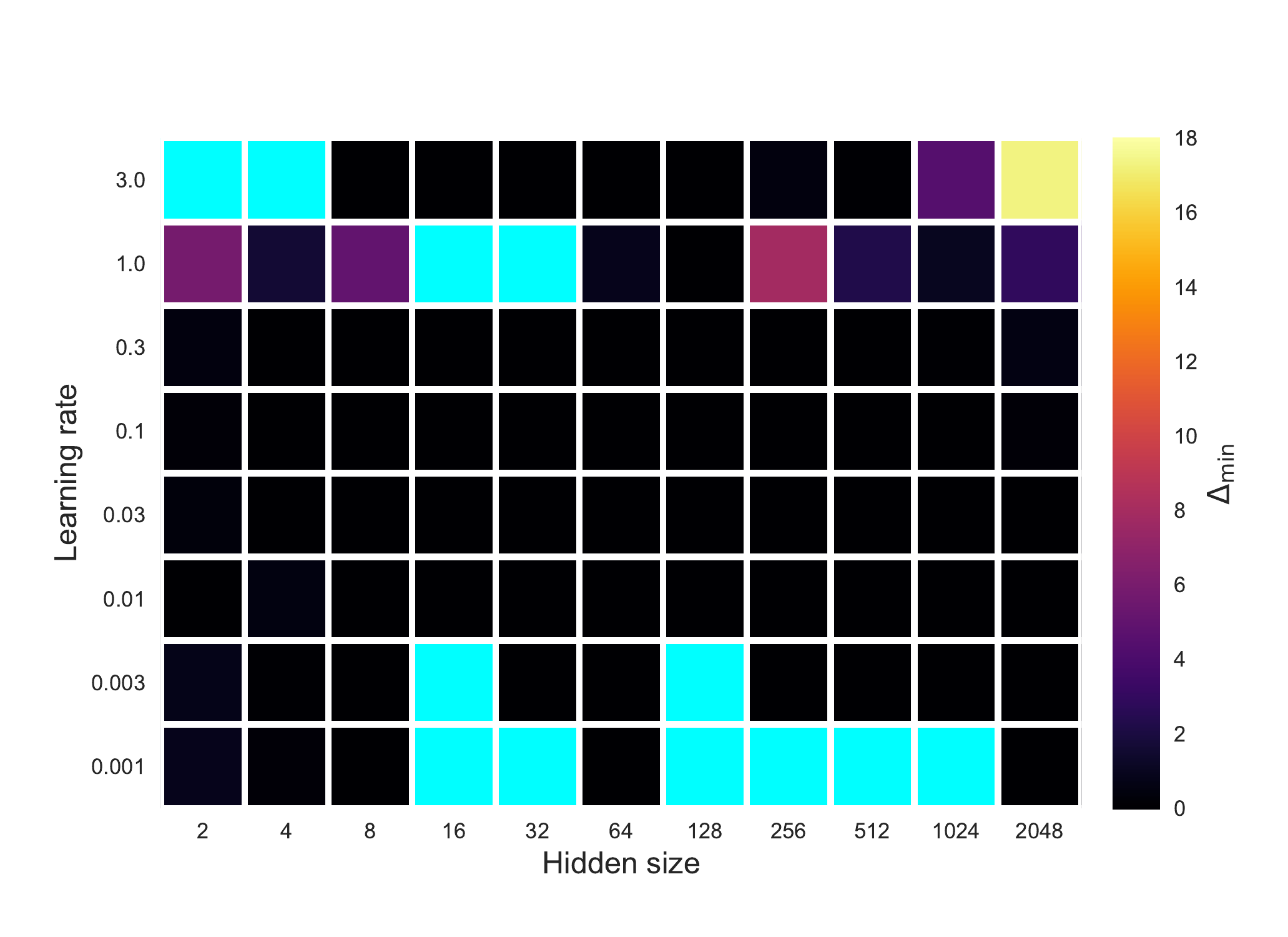}
    \includegraphics[height=5cm]{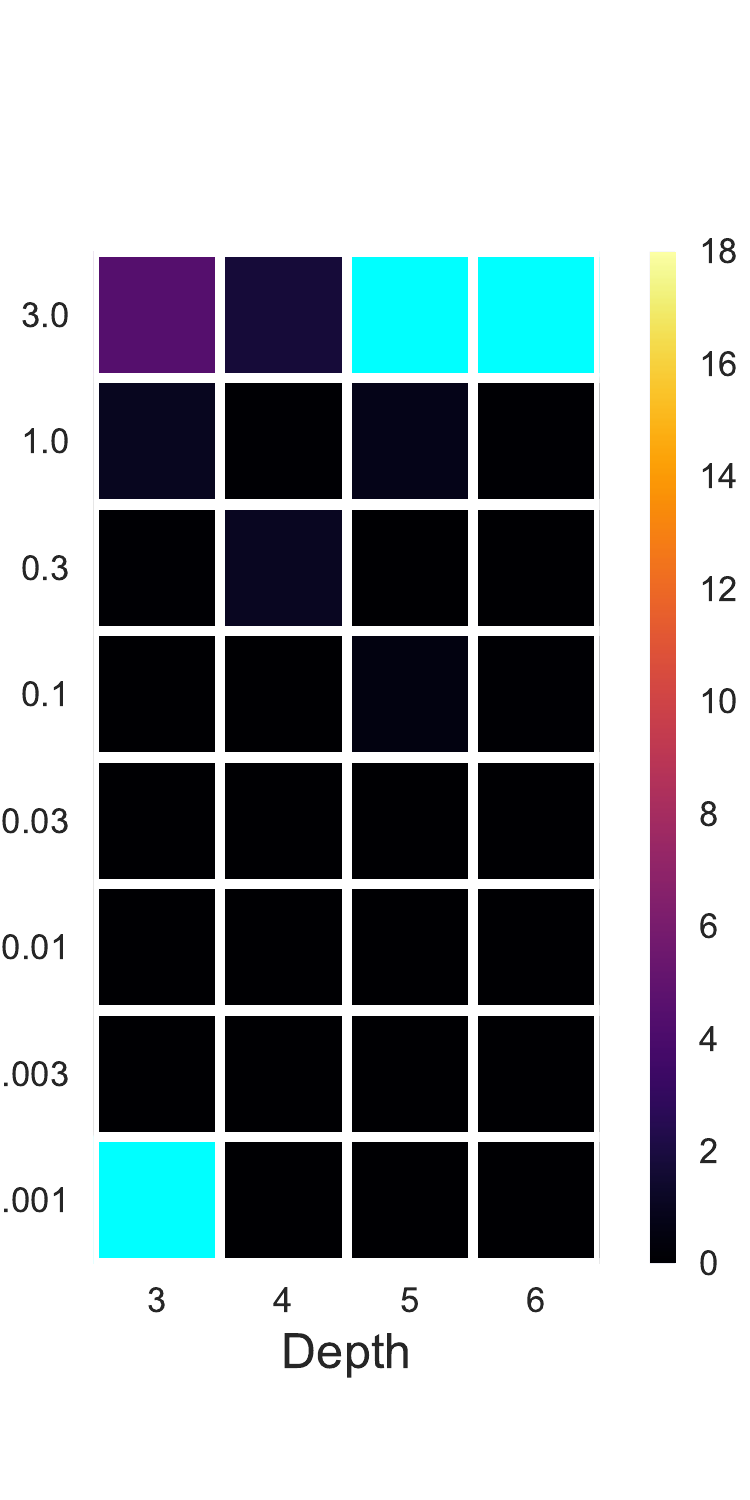}}
    \caption{Heatmaps of the average $\min \Delta$ as a function of the learning rate, hidden size and network depth for fully-connected networks trained on MNIST and FashionMNIST. On the left, depth 3 networks with varying hidden sizes are compared. On the right, networks with hidden size 1024 are compared over varying depth. {\color{cyan} \textbf{Cyan}} patches indicate that no model with the given configuration achieved a minimum training loss of 0.1.}
    \label{fig:mnist_delta_heatmaps}
\end{figure}

\subsection{Problem Difficulty}
We revisited the conclusion of \citet{goodfellow2014qualitatively} that the MLI property holds due to the relative ease of optimization. We explored this question on three fronts. First, we used a fixed network size and varied the number of data points in the dataset. Second, we used a fixed dataset size and varied the number of hidden units in a network of fixed depth (Figure~\ref{fig:mnist_delta_heatmaps}). And third, we varied random corruption of labels in the training dataset.

\label{app:problem-diff}
\paragraph{Dataset size.} We trained fully-connected networks on the Fashion-MNIST dataset using SGD with a learning rate of 0.1. The networks had a single hidden layer with 1000 hidden units, and we varied the dataset size from 10 up to the full size 60000. Figure~\ref{fig:vary_dsize} shows the linear interpolation trained on varying dataset sizes. We observed that even when the training dynamics are unstable and highly non-linear, the interpolation is still monotonically decreasing.

\begin{figure}
    \begin{minipage}{0.32\linewidth}
    \centering
    \includegraphics[width=\linewidth]{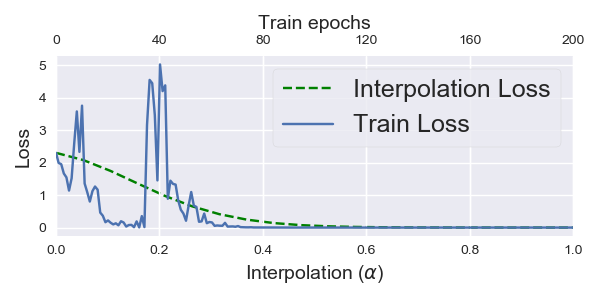}
    \end{minipage}\hfill%
    \begin{minipage}{0.32\linewidth}
    \centering
    \includegraphics[width=\linewidth]{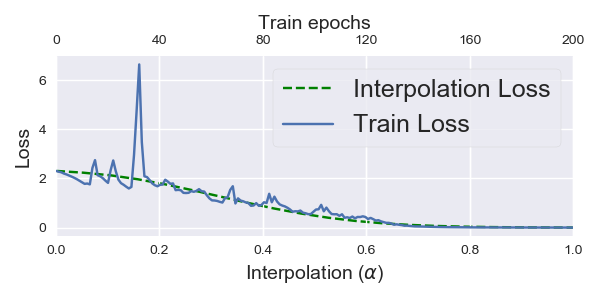}
    \end{minipage}\hfill%
    \begin{minipage}{0.32\linewidth}
    \centering
    \includegraphics[width=\linewidth]{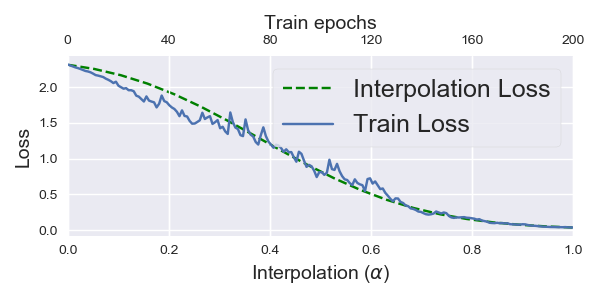}
    \end{minipage}
    \caption{Linear interpolations (green) for neural networks trained on varying dataset sizes (30, 300, 3000 from left-to-right), with loss during training overlaid (blue). Even when the training dynamics are unstable and highly non-linear, the interpolation produces a smooth monotonic curve.}
    \vspace{-0.4cm}
    \label{fig:vary_dsize}
\end{figure}

\paragraph{MLI vs.~label corruption.} When the dataset is sufficiently simple, the learning problem is easy and SGD consistently finds solutions with the MLI property. To explore this hypothesis, we trained neural networks with label corruption. We trained a neural network with two hidden layers each with 1024 units (more details can be found at Appendix~\ref{app:experiment-specific}). The labels were corrupted by uniformly sampling labels for some proportion of the data points. We varied the label corruption from 0\% to 100\% in 2.5\% intervals. We varied the proportion of label corruption from 0\% up to 100\%. At all levels of label corruption, the MLI property persisted. One possible explanation for this result follows from the fact that logit gradients cluster together by logit index --- even for inputs belonging to different true classes \citep{fort2019emergent}. This provides an explanation for gradient descent exploring a low dimensional subspace relative to the parameter space. Therefore, corrupting the label will not disrupt this clustering at initialization and, as empirically verified, is unlikely to prevent the MLI property from holding.

\subsection{Learning Dynamics}
\label{appendix-dynamics}

\citet{lewkowycz2020large} observed a region of critical large learning rates wherein gradient descent breaks out of high-curvature regions at initialization and explores regions of high-loss before settling in a low-loss region with lower curvature. We might expect that such trajectories lead to initialization-solution pairs that do not satisfy the MLI property. On one hand, in Figure~\ref{fig:vary_dsize}, we observed several runs where SGD is seen to overcome large barriers but the MLI property holds. However, in Figure~\ref{fig:nonmono_loss_landscape} we observe a projection of the loss landscape which aligns with the qualitative description of the catapult phase: a barrier in the loss, with SGD settling in a region of much lower curvature. Overall, we consider our findings inconclusive on this front.

\subsection{MLI on held-out data}

In this work, we are primarily concerned with better understanding of the interaction between the MLI property and the training loss. Therefore, all of the results that we have reported are based on statistics computed over the training set. However, the same observations also hold generally when evaluating using held-out data (up to overfitting effects). This was confirmed by \citet{goodfellow2014qualitatively}, and in this section, we provide a short qualitative study verifying this for the settings that we have studied.

\paragraph{Image reconstruction.} In Figure~\ref{fig:mnist_holdout_compare} (left two plots), we compare the loss interpolations on the training set and test set for two trained autoencoders. In the first plot, the network satisfies the MLI property but in the second it does not. In both cases, the test loss interpolation closely follows the training loss.

\begin{figure}
\begin{minipage}{0.5\linewidth}
\centering%
\includegraphics[width=\linewidth]{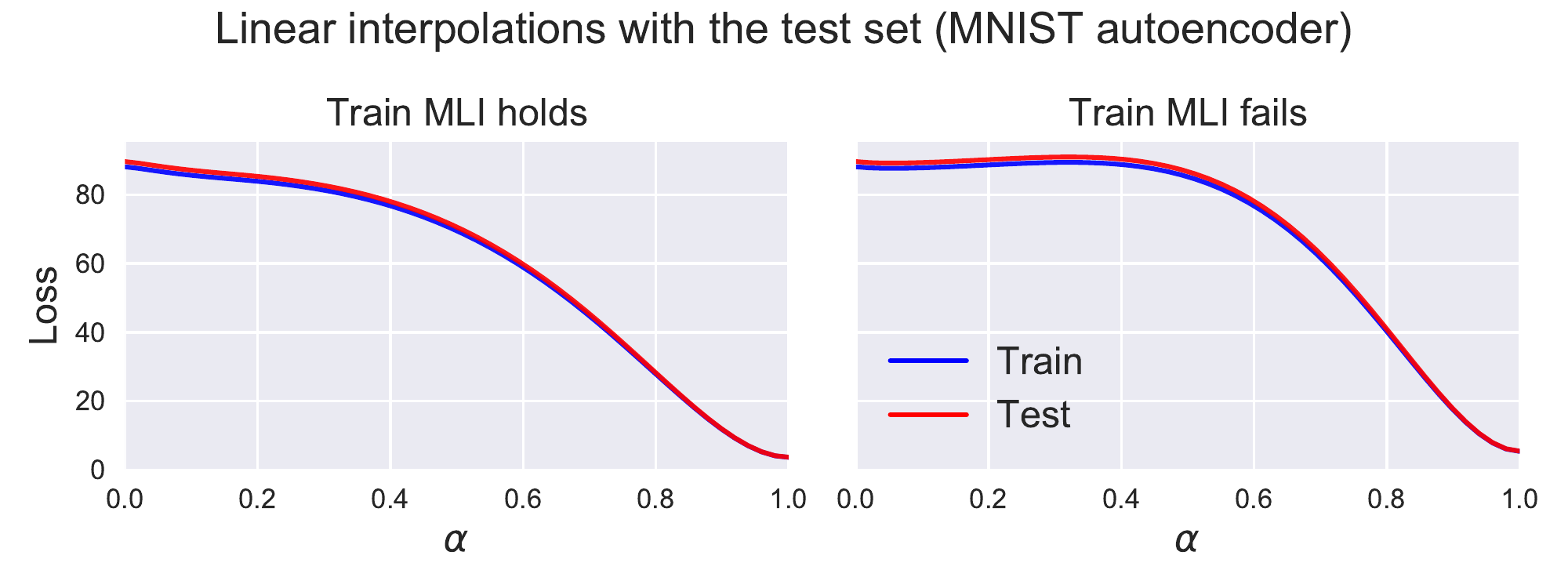}
\end{minipage}\hfill%
\begin{minipage}{0.5\linewidth}
\centering%
\includegraphics[width=\linewidth]{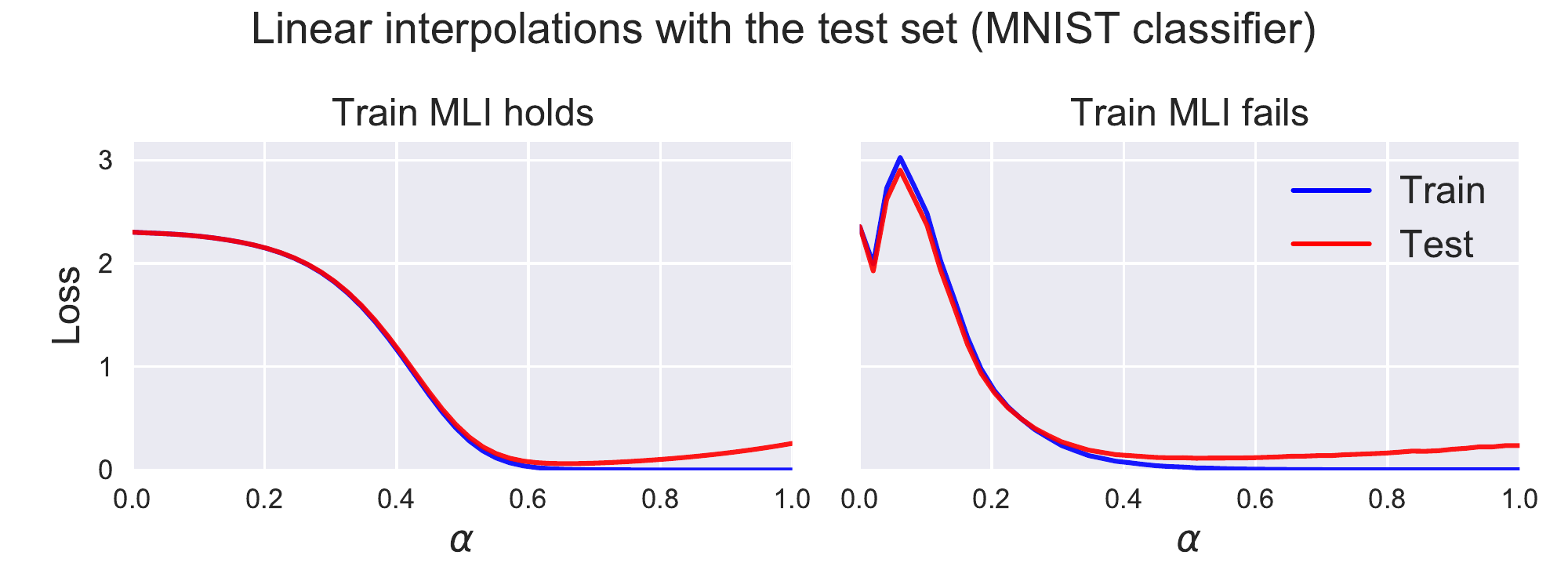}
\end{minipage}
\caption{Comparing loss interpolations on the train and test set. In the first and second plots, fully-connected autoencoders trained on MNIST are evaluated that do/don't satisfy the MLI property (respectively). The third and fourth plots display fully-connected MNIST classifiers.}
\label{fig:mnist_holdout_compare} 
\end{figure}

\paragraph{MNIST Classifiers.} The third and fourth plots in Figure~\ref{fig:mnist_holdout_compare} show the train and test loss interpolations for fully-connected MNIST classifiers. In this case, the test loss increases towards the end of the interpolation path while the training loss stays small. This happens because the network becomes over-confident in its predictions and pays a larger cost for misclassification on the test-set (even though the accuracy remains the same). This observed behaviour is one reason why we favour exploration of the training loss throughout our work. Despite this, we do still observe the test loss following the general shape of the training loss for most of the interpolation path.

\paragraph{CIFAR-10 \& CIFAR-100 Classifiers.} In Figure~\ref{fig:cifar_holdout_compare}, we compare the loss interpolations on the training set and test set for ResNets trained on CIFAR-10 and CIFAR-100. The first two plots show CIFAR-10 classifiers with the third and fourth plot showing CIFAR-100 classifiers. The first and third plots show networks that satisfy the MLI property on the training loss, while the second and fourth show networks that fail to satisfy the MLI property on the training loss. As with the MNIST classifiers, we observe that the test loss has a tendency to increase towards the end of the interpolation path (while following the overall trend of the training loss).
\begin{figure}
\begin{minipage}{0.5\linewidth}
\centering%
\includegraphics[width=\linewidth]{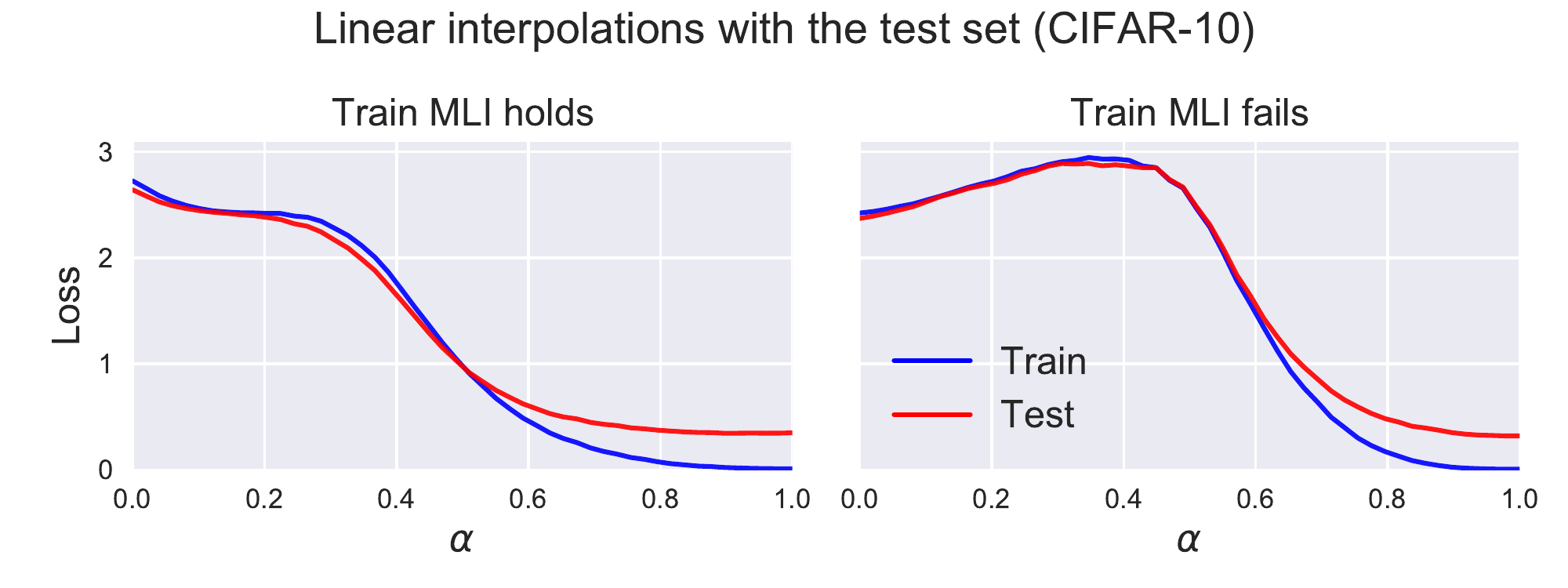}
\end{minipage}\hfill%
\begin{minipage}{0.5\linewidth}
\centering%
\includegraphics[width=\linewidth]{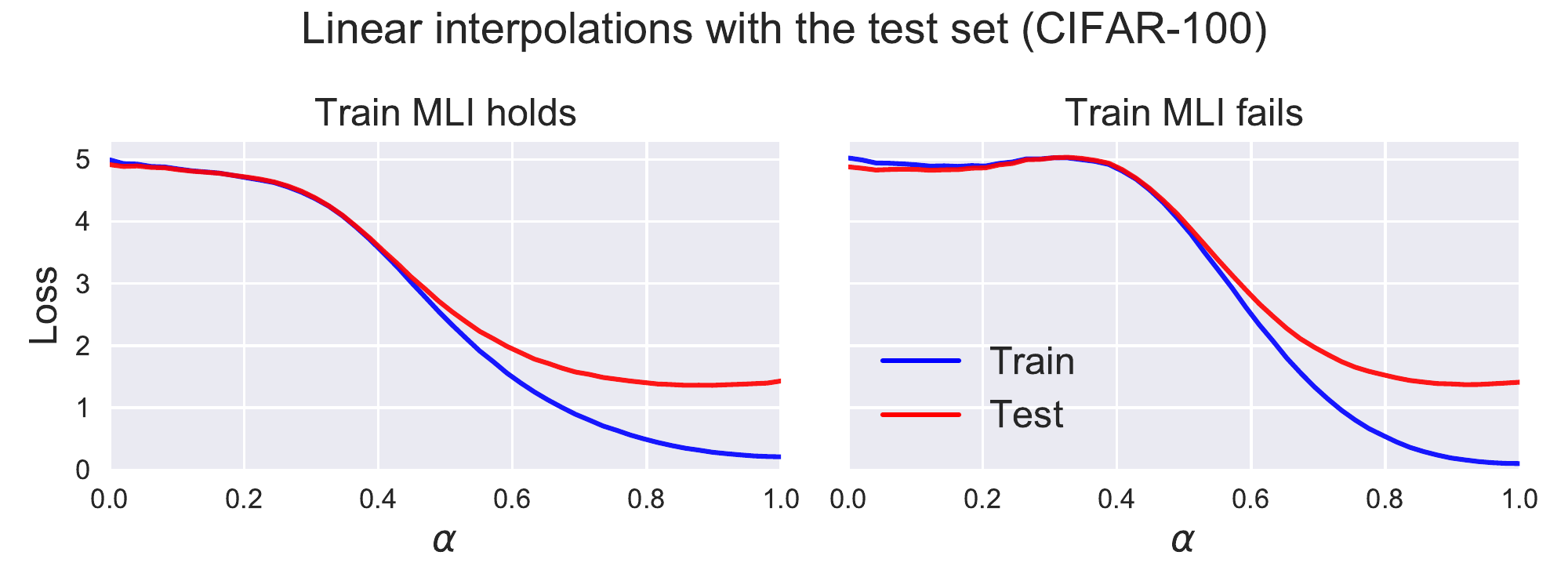}
\end{minipage}
\caption{Comparing loss interpolations on the train and test set. In the first and second plots, ResNets trained on CIFAR-10 are evaluated that do/don't satisfy the MLI property (respectively). The third and fourth plots display interpolation plots for ResNets trained on CIFAR-100.}
\label{fig:cifar_holdout_compare} 
\end{figure}

%% file: appendix/additional_theory.tex
\section{Additional Theoretical Analysis}\label{app:additional_theory}
In this section, we present additional theoretical analysis of the MLI property.

\subsection{Wide neural networks}
\label{app:wide_nets}

In this section, we prove that sufficiently wide fully-connected networks satisfy the MLI property. To do so, we lean on prior analysis from \citet{lee2019wide}. We assume that the fully-connected network has the following layer sizes $d \rightarrow m \rightarrow \ldots m \rightarrow k$, with $m \rightarrow \infty$. We also assume our loss function is mean-squared error,
\[\calL(\btheta) = \frac{1}{2}\sum_{i=1}^n \Vert f_{\btheta}(\bx_i) - \by_i \Vert^2.\]
\paragraph{Assumptions.} We borrow the setting established by \citet{lee2019wide} that consists of four assumptions.
\begin{enumerate}
    \item The widths of the hidden layers are identical (as stated above).
    \item The neural tangent kernel, $\frac{1}{n}J(\btheta)^\top J(\btheta)$, is full-rank with finite singular values. I.e.,
    \[0 < \lambda_{\min}\left(\frac{1}{n}J(\btheta)^\top J(\btheta)\right) \leq \lambda_{\max}\left(\frac{1}{n}J(\btheta)^\top J(\btheta)\right) < \infty.\]
    Further, we define $\eta_{\textrm{critical}} := 2/(\lambda_{\min} + \lambda_{\max})$
    \item The training set $\{(\bx_i, \by_i)\}_{i=1}^n$ is contained in a compact set and contains no duplicate inputs.
    \item The activation function, $\phi$, in the network satisfies the following,
    \[ \vert \phi(0)\vert < \infty, \:\: \Vert \phi' \Vert_\infty < \infty, \:\: \sup_{\bx \neq \tilde{\bx}} \frac{\vert \phi'(\bx) - \phi'(\tilde{\bx})\vert}{\bx - \tilde{\bx}} < \infty\]
\end{enumerate}

\paragraph{Background.} We utilize two results from \citet{lee2019wide}. The first of which bounds the Jacobian matrix in Frobenius norm about initialization.
\begin{lemma}\label{lemma:locally_lipschitz_jacobian}(Locally Lipschitz Jacobian [Lemma~1 \citep{lee2019wide}])
Assume conditions 1-4 above. There is a $K > 0$ such that for every $C > 0$, with high probability over random initialization,
\begin{align*}
   \frac{1}{\sqrt{m}}\Vert J(\btheta) \Vert_F &\leq K, \\
   \frac{1}{\sqrt{m}}\Vert J(\btheta) - J(\btheta')\Vert_F &\leq K\Vert \btheta - \btheta'\Vert_2,
\end{align*}
for all $\btheta$ and $\btheta'$ such that $\Vert \btheta - \btheta_0\Vert \leq Cm^{-1/2}$ and $\Vert \btheta' - \btheta_0\Vert \leq Cm^{-1/2}$.
\end{lemma}
In words, Lemma~\ref{lemma:locally_lipschitz_jacobian} guarantees that the Frobenius norm of the Jacobian is close to initialization as width grows and that it does not vary too quickly. The second of these two constraints also guarantees that the norm of the network Hessian is bounded (by considering $\btheta$ and $\btheta'$ arbitrarily close).

The second result that we borrow provides a high-probability guarantee that infinitely wide neural networks find solutions near to their initialization (the lazy training regime \cite{chizat2018lazy}).
\begin{lemma}\label{lemma:small_weight_change}(Lazy training [Theorem~G.1 \citep{lee2019wide}])
Assume conditions 1-4 above. For all $\delta > 0$ and $\eta_0 < \eta_{\textrm{critical}}$, there exists $M \in \bbN$, $R_0 > 0$, and $K > 1$ such that for every $m > M$, with probability at least $1-\delta$ over random initialization, gradient descent with learning rate $\eta = \eta_0 / m$ applied for $T$ steps satisfies,
\[\Vert \btheta_T - \btheta_0 \Vert_2 \leq \frac{3KR_0}{\lambda_{\min}}m^{-1/2}.\]
\end{lemma}

\paragraph{MLI for infinite width networks.}
From the above, we can prove that in the limit of infinite width, gradient descent with a suitably small learning rate finds a solution that is linearly connected to the initialization.

Intuitively, this result holds as in a region near a minimum the objective is locally convex. As the width of the network grows, the minimum found by gradient descent becomes arbitrarily close to initialization and thus the linear interpolation is acting over a convex function.

For completeness, we first provide a simple proof that linear interpolations satisfy the MLI property in convex loss landscapes. The result itself follows from standard techniques presented in, for example, \citet{boyd2004convex}.
\begin{restatable}[Linearity and convexity gives MLI]{lemma}{linearconvexmli}\label{lemma:linear_convex_mli}
Let $\calL : \bbR^d \rightarrow \bbR$ be a convex, differentiable loss function. Further, let $\btheta^* \in \argmin \calL$. Then, for all $\btheta_0 \in \bbR^d$, $g(\alpha) := \calL(\btheta_0 + \alpha(\btheta^* - \btheta_0))$ is monotonically decreasing for $\alpha \in [0,1)$.
\end{restatable}
\begin{proof}
We have that $g(\alpha)$ is also a convex, differentiable function. Therefore, using the first-order convexity condition on $g$,
\[g'(\alpha) \leq \frac{g(1) - g(\alpha)}{1 - \alpha} \leq 0\]
\end{proof}

We now proceed with the main result of this section.

\begin{restatable}[Wide networks satisfy the MLI Property]{theorem}{infwidthmli}\label{thm:inf_width_mli}
Assume conditions 1-4 above. For all $\delta > 0$ and $\eta_0 < \eta_{\textrm{critical}}$, there exists $M \in \bbN$ such that for every $m > M$, with probability at least $1-\delta$ over random initialization, gradient descent with learning rate $\eta = \eta_0 / m$ satisfies,
\[\calL(\btheta_{\alpha_2}) - \calL(\btheta_{\alpha_1}) \leq 0,\]
for all $\alpha_2 > \alpha_1 \in [0,1)$.
\end{restatable}
\begin{proof}
For brevity, we write $\Delta\btheta = \btheta_T - \btheta_0$, with $\btheta_\alpha = \btheta_0 + \alpha \Delta\btheta$. Our approach is to linearize the loss in function-space and show that all remaining terms are quadratic in $\Delta\btheta$ and so are dominated by the linear terms for a sufficiently wide network.

We begin by considering the Taylor series of $\calL(\btheta_\alpha)$ about $\btheta_0$, using the Lagrange form of the remainder,
\begin{align}
    \calL(\btheta_\alpha) &= \calL(\btheta_0) + \alpha \nabla_{\btheta}\calL(\btheta_0)^\top \Delta\btheta + \frac{1}{2}\alpha^2 \Delta\btheta^\top \nabla^2_{\btheta} \calL(\bx_i; \bar{\btheta}_\alpha) \Delta\btheta,\\
    &= \calL(\btheta_0) + \frac{\alpha}{2n}\sum_{i=1}^n(f(\bx_i; \btheta_0) - \by_i)^\top J(\bx_i; \btheta_0) \Delta\btheta + \frac{1}{2}\alpha^2 \Delta\btheta^\top \nabla^2_{\btheta} \calL(\bx_i; \bar{\btheta}_\alpha) \Delta\btheta,
\end{align}
for some $\bar{\btheta}_\alpha$ on the line $[\btheta_0, \btheta_\alpha]$. Now, noting that the Hessian of $f$ with respect to $\btheta$ is a third-order tensor, we can utilize the integral form of the Taylor expansion to write,
\begin{equation}
    \left(J(\bx_i; \btheta_0)\Delta\btheta\right)_j = f(\bx_i; \btheta_T)_j - f(\bx_i; \btheta_0)_j - \frac{1}{2}\Delta\btheta^\top\left(\int_0^1 \frac{\partial^2 f_j}{\partial \btheta^2}(\bx_i; \btheta_{\alpha'}) d\alpha'\right) \Delta\btheta,
\end{equation}
where the $j$ subscript notation indicates vector indexing. Collecting terms, we have
\begin{align*}
    \calL(\btheta_\alpha) - \calL(\btheta_0) =& \frac{1}{2n}\sum_{i=1}^n \Bigl[\alpha(f(\bx_i; \btheta_0) - \by_i)^\top (f(\btheta_T) - f(\btheta_0)) + \frac{1}{2}\alpha^2 \Delta\btheta^\top \nabla^2_{\btheta} \calL(\bx_i; \bar{\btheta}_\alpha) \Delta\btheta,\\
    & -\frac{1}{2}\alpha\sum_{j=1}^k (f(\bx_i; \btheta_0) - \by_i)_k \Delta\btheta^\top\left(\int_0^1 \frac{\partial^2 f_k}{\partial \btheta^2}(\bx_i; \btheta_{\alpha'}) d\alpha' \right) \Delta\btheta \Bigr].
\end{align*}
Now, noting that $\calL(\btheta_{\alpha_2}) - \calL(\btheta_{\alpha_1}) = \left(\calL(\btheta_{\alpha_2}) - \calL(\btheta_0)\right) - \left(\calL(\btheta_{\alpha_1}) - \calL(\btheta_0)\right)$, we have
\begin{align*}
    \calL(\btheta_{\alpha_2}) - \calL(\btheta_{\alpha_1}) =& \frac{1}{2n}\sum_{i=1}^n \Bigl[(\alpha_2 - \alpha_1)(f(\bx_i; \btheta_0) - \by_i)^\top (f(\btheta_T) - f(\btheta_0))\\
    & + \frac{1}{2}\Delta\btheta^\top\left(\alpha_2^2 \nabla^2_{\btheta} \calL(\bx_i; \bar{\btheta}_{\alpha_2})  - \alpha^2_1 \nabla^2_{\btheta} \calL(\bx_i; \bar{\btheta}_{\alpha_2}) \right)\Delta\btheta,\\
    & -\frac{1}{2}(\alpha_2 - \alpha_1)\sum_{j=1}^k (f(\bx_i; \btheta_0) - \by_i)_k \Delta\btheta^\top\left(\int_0^1 \frac{\partial^2 f_k}{\partial \btheta^2}(\bx_i; \btheta_{\alpha'}) d\alpha' \right) \Delta\btheta \Bigr].
\end{align*}
The first term in the sum is negative as $\calL$ is convex in $f$ (and $\alpha_2 > \alpha_1$). It remains to show that the other terms behave asymptotically like $\Vert\Delta\btheta\Vert^2$. First, notice that we can decompose the Hessian of the loss as follows,
\begin{equation}
    \nabla_{\btheta}^2\calL(\bx_i; \btheta) = J(\bx_i; \btheta)^\top J(\bx_i; \btheta) + \sum_{j=1}^k (f(\bx_i; \btheta) - \by_i)_j \frac{\partial^2 f_j}{\partial \btheta^2}(\bx_i; \btheta)
\end{equation}
Furthermore, by Lemma~\ref{lemma:small_weight_change}, there exists an $M' \in \bbN$ such that for all $m > M'$ we have $\Vert \Delta\btheta\Vert \leq O(m^{-1/2})$ with probability at least $1 - \delta$. Under this event, we can apply Lemma~\ref{lemma:locally_lipschitz_jacobian} to guarantee that the average Jacobian and Hessian norms are bounded about initialization:
\[\frac{1}{n}\sum_{i=1}^n \left\Vert J(\bx_i; \btheta) \right\Vert^2_F < \infty \:\:\:\textrm{ and }\:\:\: \frac{1}{n}\sum_{i=1}^n\sum_{j=1}^k\left\Vert \frac{\partial^2 f_j}{\partial \btheta^2}(\bx_i; \btheta) \right\Vert^2_F < \infty.\]
Therefore, there exists an $M \geq M'$, such that for all $m > M$ the negative first-order term dominates the second order terms. Under the $1-\delta$ probability event, this guarantees that the loss is monotonically decreasing along the linear interpolation.
\end{proof}

\subsection{A Noisy Quadratic Model}
\label{app:nqm}
The noisy quadratic model (NQM) \citep{schaul2013no, wu2018understanding, zhang2019algorithmic} serves as a useful guide for understanding the effects of stochasticity in asymptotic neural network training. Indeed, \citet{zhang2019algorithmic} demonstrate that the NQM makes predictions that are aligned with experimental results on deep neural networks. Using this model, we can provide an explanation for one possible cause of non-monotonicity: an inflection point of the interpolation curve with positive second derivative close to $\alpha=1$. Intuitively, we can imagine a bowl-shaped loss surface where the final parameters lies on the opposite side of the optima relative to the initialization. This non-monotonicity is likely to occur when training with smaller batch sizes and/or using larger (fixed) learning rates.

Let our loss function be as follows:
\def\rvtheta{{\boldsymbol{\theta}}}
\def\c{{\textbf{c}}}
\begin{align}
    \mathcal{L}(\rvtheta) = \frac{1}{2} \rvtheta^{\top} \mathbf{K} \rvtheta,
\end{align}
where $\rvtheta \in \mathbb{R}^d$ and $\mathbf{K} \in \mathbb{R}^{d \times d}$. The optimization algorithm receives stochastic gradients $\mathbf{K} \rvtheta + \mathbf{c}$, where $\mathbf{c} \sim \calN (\mathbf{0}, \mathbf{K})$. Consider the iterates $\{\rvtheta_{i}\}_{i=0}^{\top}$ produced by gradient descent. With a sufficiently small learning rate, the expected value of the iterate converges i.e. $\lim_{t \to \infty} \bbE[\mathcal{L}(\rvtheta_t)] = 0$. 

Also consider interpolating between arbitrary $\btheta_1$ and $\btheta_2$. The loss along the interpolation direction is $\mathcal{L}(\btheta_1 + \alpha (\btheta_2 - \btheta_1))$. We compute the derivative with respect to $\alpha$:
\begin{align}
    \frac{\partial \mathcal{L}}{\partial \alpha} (\btheta_1 + \alpha (\btheta_2 - \btheta_1))
    &= \frac{\partial}{\partial \alpha}  \left[\frac{1}{2}(\btheta_1 + \alpha (\btheta_2 - \btheta_1))^{\top} \mathbf{K} (\btheta_1 + \alpha (\btheta_2 - \btheta_1)) \right] \\ 
    &=  (\btheta_2 - \btheta_1)^{\top} \mathbf{K} (\btheta_1 + \alpha (\btheta_2 - \btheta_1)) 
\end{align}
Hence, the loss is monotonically decreasing if, for all $\alpha \in [0,1]$, 
\begin{align}
    (\btheta_2 - \btheta_1)^{\top} \mathbf{K} (\btheta_1 + \alpha (\btheta_2 - \btheta_1)) < 0
\end{align}
In the one dimension case, this equation is saying that interpolation is non-monotonic when $\btheta_1$ and $\btheta_2$ are on the opposite side of the minima. More generally, note that because $\frac{\partial \mathcal{L}}{\partial \alpha}$ is linear in $\alpha$, the interpolation is monotonically decreasing if and only if both of these conditions at the endpoints are satisfied:
\begin{align}
  (\btheta_2 - \btheta_1)^{\top} \mathbf{K} \btheta_1 &< 0 \\
  (\btheta_2 - \btheta_1)^{\top} \mathbf{K} \btheta_2 &< 0 
\end{align}
These two conditions correspond to a negative derivative with respect to $\alpha$ at $\btheta_1$ and $\btheta_2$. Since we choose a learning rate so that the loss decreases in expectation (and hence the derivative is anti-aligned with $\c_2 - \c_1$ at initialization), it suffices to check just the second condition.

We simulate learning in this model to measure the effect of stochasticity under varying learning rates on the MLI property. As in \citet{zhang2019algorithmic}, we use $\btheta_1 := \btheta \sim \calN(\mathbf{0}, \mathbf{I})$ and $\mathbf{K} = diag \{1, \frac{1}{2}, \frac{1}{3}, \dots, \frac{1}{d}\}$. As $t \to \infty$, the point $\btheta_2 := \btheta_T \sim \calN(\mathbf{0}, \eta \mathbf{K})$, where $\eta$ is the final learning rate and the random variable comes from the noise in the gradient. 
Through empirical simulations, we verify that this is approximately a symmetric distribution about $0$, so the probability we have monotonic interpolation is roughly $\frac{1}{2}$. This is empirically verified in Figure~\ref{fig:nqm}. A smaller learning rate means that the distribution of  $(\btheta_2 - \btheta_1)^{\top} \mathbf{K} \btheta_2$ has less variance. Because we discretize $\alpha$ when we check for MLI, we have $P((\btheta_2 - \btheta_1)^{\top} \mathbf{K} \c_2) < \epsilon)$ increases as the learning rate decreases for some small $\epsilon$. 

\begin{figure}[t]
    \begin{minipage}{0.49\linewidth}
    \centering
    \includegraphics[width=\linewidth]{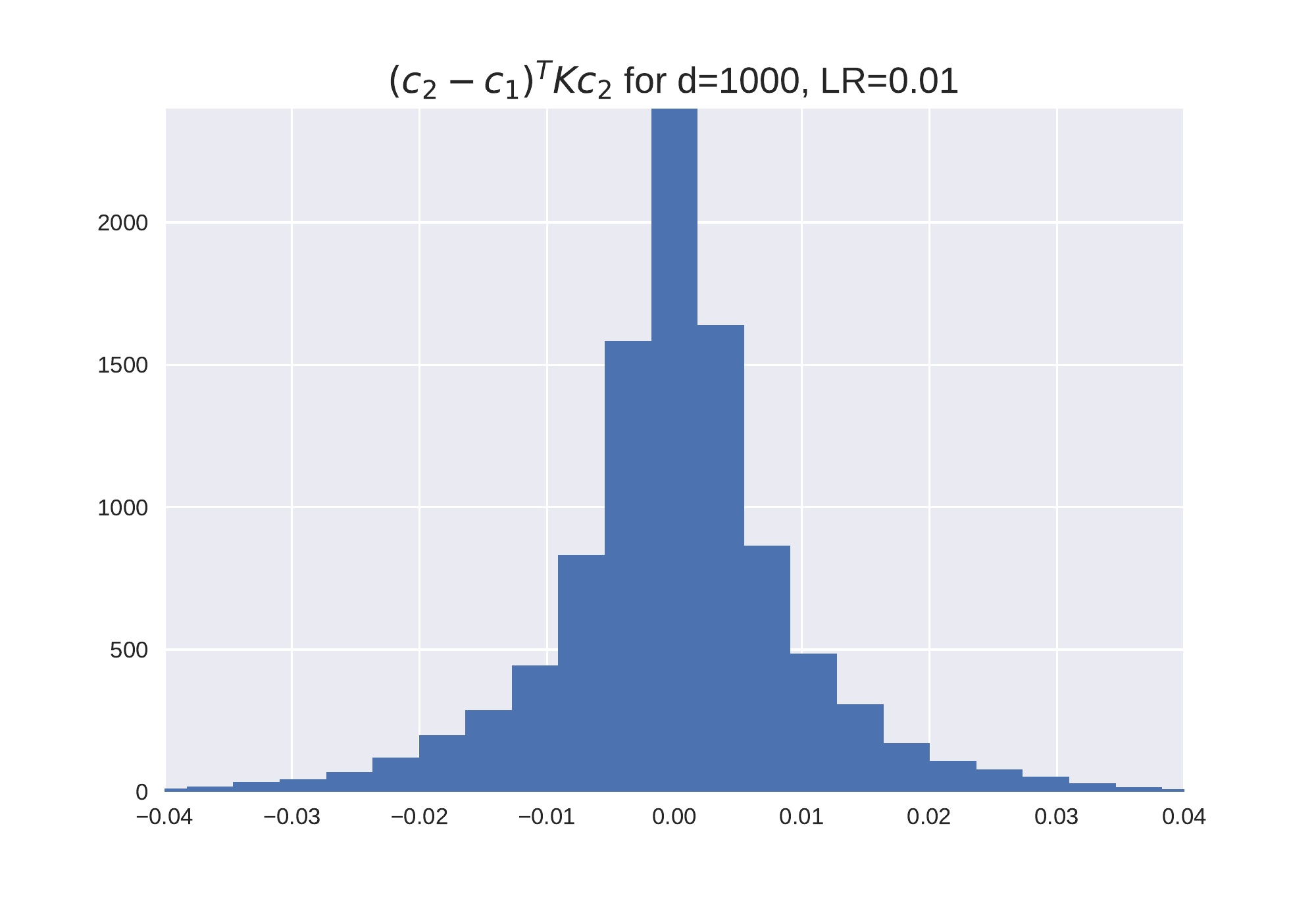}
    \end{minipage}\hfill%
    \begin{minipage}{0.49\linewidth}
    \centering
    \includegraphics[width=\linewidth]{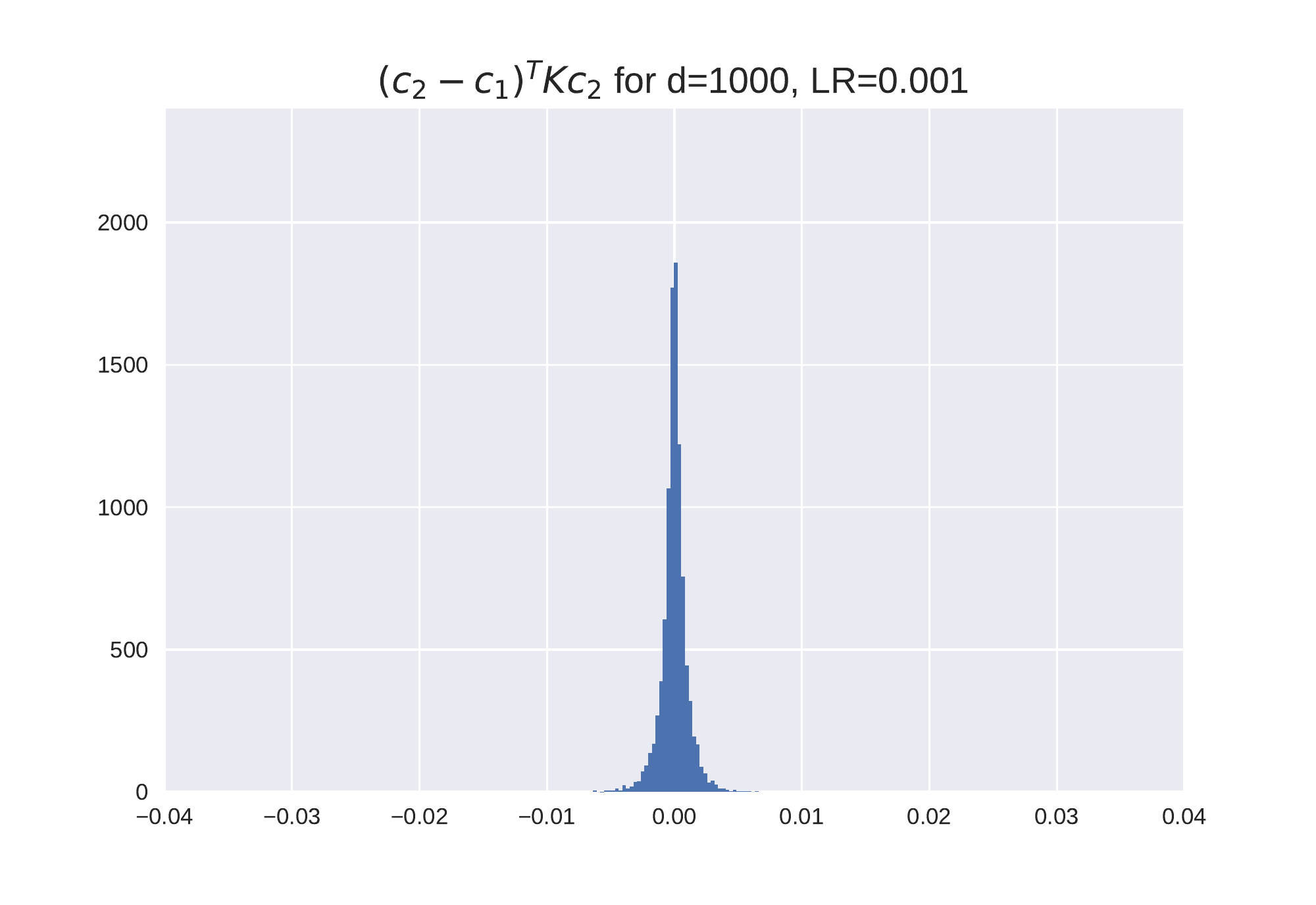}
    \end{minipage}
    \caption{For smaller learning rates, the standard deviation of the distribution goes down. Hence the probability that $P((\btheta_2 - \btheta_1)^{\top} \mathbf{K} \c_2) < \epsilon$ for some small $\epsilon$ goes up (indicating non-monotonicity from a inflection point near the optima that is hard to detect). We use an equal number of bins in both plots.}
    \vspace{-0.4cm}
    \label{fig:nqm}
\end{figure}

%% file: arxiv.bbl
\begin{thebibliography}{56}
\providecommand{\natexlab}[1]{#1}
\providecommand{\url}[1]{\texttt{#1}}
\expandafter\ifx\csname urlstyle\endcsname\relax
  \providecommand{\doi}[1]{doi: #1}\else
  \providecommand{\doi}{doi: \begingroup \urlstyle{rm}\Url}\fi

\bibitem[Agarwal et~al.(2020)Agarwal, Anil, Hazan, Koren, and
  Zhang]{agarwal2020disentangling}
Naman Agarwal, Rohan Anil, Elad Hazan, Tomer Koren, and Cyril Zhang.
\newblock Disentangling adaptive gradient methods from learning rates.
\newblock \emph{arXiv preprint arXiv:2002.11803}, 2020.

\bibitem[Amari(1998)]{amari1998natural}
Shun-Ichi Amari.
\newblock Natural gradient works efficiently in learning.
\newblock \emph{Neural computation}, 10\penalty0 (2):\penalty0 251--276, 1998.

\bibitem[Amari et~al.(2020)Amari, Ba, Grosse, Li, Nitanda, Suzuki, Wu, and
  Xu]{amari2020does}
Shun-ichi Amari, Jimmy Ba, Roger Grosse, Xuechen Li, Atsushi Nitanda, Taiji
  Suzuki, Denny Wu, and Ji~Xu.
\newblock When does preconditioning help or hurt generalization?
\newblock \emph{arXiv preprint arXiv:2006.10732}, 2020.

\bibitem[Boyd et~al.(2004)Boyd, Boyd, and Vandenberghe]{boyd2004convex}
Stephen Boyd, Stephen~P Boyd, and Lieven Vandenberghe.
\newblock \emph{Convex optimization}.
\newblock Cambridge university press, 2004.

\bibitem[Chizat et~al.(2018)Chizat, Oyallon, and Bach]{chizat2018lazy}
Lenaic Chizat, Edouard Oyallon, and Francis Bach.
\newblock On lazy training in differentiable programming.
\newblock \emph{arXiv preprint arXiv:1812.07956}, 2018.

\bibitem[Conneau et~al.(2019)Conneau, Khandelwal, Goyal, Chaudhary, Wenzek,
  Guzm{\'a}n, Grave, Ott, Zettlemoyer, and Stoyanov]{conneau2019unsupervised}
Alexis Conneau, Kartikay Khandelwal, Naman Goyal, Vishrav Chaudhary, Guillaume
  Wenzek, Francisco Guzm{\'a}n, Edouard Grave, Myle Ott, Luke Zettlemoyer, and
  Veselin Stoyanov.
\newblock Unsupervised cross-lingual representation learning at scale.
\newblock \emph{arXiv preprint arXiv:1911.02116}, 2019.

\bibitem[Dinh et~al.(2017)Dinh, Pascanu, Bengio, and Bengio]{dinh2017sharp}
Laurent Dinh, Razvan Pascanu, Samy Bengio, and Yoshua Bengio.
\newblock Sharp minima can generalize for deep nets.
\newblock \emph{arXiv preprint arXiv:1703.04933}, 2017.

\bibitem[Draxler et~al.(2018)Draxler, Veschgini, Salmhofer, and
  Hamprecht]{draxler2018essentially}
Felix Draxler, Kambis Veschgini, Manfred Salmhofer, and Fred~A Hamprecht.
\newblock Essentially no barriers in neural network energy landscape.
\newblock \emph{arXiv preprint arXiv:1803.00885}, 2018.

\bibitem[Fort and Ganguli(2019)]{fort2019emergent}
Stanislav Fort and Surya Ganguli.
\newblock Emergent properties of the local geometry of neural loss landscapes.
\newblock \emph{arXiv preprint arXiv:1910.05929}, 2019.

\bibitem[Fort and Jastrzebski(2019)]{NEURIPS2019_48042b1d}
Stanislav Fort and Stanislaw Jastrzebski.
\newblock Large scale structure of neural network loss landscapes.
\newblock In H.~Wallach, H.~Larochelle, A.~Beygelzimer, F.~d\textquotesingle
  Alch\'{e}-Buc, E.~Fox, and R.~Garnett, editors, \emph{Advances in Neural
  Information Processing Systems}, volume~32. Curran Associates, Inc., 2019.
\newblock URL
  \url{https://proceedings.neurips.cc/paper/2019/file/48042b1dae4950fef2bd2aafa0b971a1-Paper.pdf}.

\bibitem[Fort and Scherlis(2019)]{fort2019goldilocks}
Stanislav Fort and Adam Scherlis.
\newblock The {G}oldilocks zone: Towards better understanding of neural network
  loss landscapes.
\newblock In \emph{Proceedings of the AAAI Conference on Artificial
  Intelligence}, volume~33, pages 3574--3581, 2019.

\bibitem[Fort et~al.(2020)Fort, Dziugaite, Paul, Kharaghani, Roy, and
  Ganguli]{fort2020deep}
Stanislav Fort, Gintare~Karolina Dziugaite, Mansheej Paul, Sepideh Kharaghani,
  Daniel~M Roy, and Surya Ganguli.
\newblock Deep learning versus kernel learning: an empirical study of loss
  landscape geometry and the time evolution of the neural tangent kernel.
\newblock \emph{arXiv preprint arXiv:2010.15110}, 2020.

\bibitem[Frankle(2020)]{frankle2020revisiting}
Jonathan Frankle.
\newblock Revisiting" qualitatively characterizing neural network optimization
  problems".
\newblock \emph{arXiv preprint arXiv:2012.06898}, 2020.

\bibitem[Frankle and Carbin(2018)]{frankle2018lottery}
Jonathan Frankle and Michael Carbin.
\newblock The lottery ticket hypothesis: Finding sparse, trainable neural
  networks.
\newblock \emph{arXiv preprint arXiv:1803.03635}, 2018.

\bibitem[Frankle et~al.(2019)Frankle, Dziugaite, Roy, and
  Carbin]{frankle2019linear}
Jonathan Frankle, Gintare~Karolina Dziugaite, Daniel~M Roy, and Michael Carbin.
\newblock Linear mode connectivity and the lottery ticket hypothesis.
\newblock \emph{arXiv preprint arXiv:1912.05671}, 2019.

\bibitem[Garipov et~al.(2018)Garipov, Izmailov, Podoprikhin, Vetrov, and
  Wilson]{garipov2018loss}
Timur Garipov, Pavel Izmailov, Dmitrii Podoprikhin, Dmitry~P Vetrov, and
  Andrew~G Wilson.
\newblock Loss surfaces, mode connectivity, and fast ensembling of dnns.
\newblock In \emph{Advances in Neural Information Processing Systems}, pages
  8789--8798, 2018.

\bibitem[Goodfellow et~al.(2014)Goodfellow, Vinyals, and
  Saxe]{goodfellow2014qualitatively}
Ian~J Goodfellow, Oriol Vinyals, and Andrew~M Saxe.
\newblock Qualitatively characterizing neural network optimization problems.
\newblock \emph{arXiv preprint arXiv:1412.6544}, 2014.

\bibitem[Goyal et~al.(2017)Goyal, Doll{\'a}r, Girshick, Noordhuis, Wesolowski,
  Kyrola, Tulloch, Jia, and He]{goyal2017accurate}
Priya Goyal, Piotr Doll{\'a}r, Ross Girshick, Pieter Noordhuis, Lukasz
  Wesolowski, Aapo Kyrola, Andrew Tulloch, Yangqing Jia, and Kaiming He.
\newblock Accurate, large minibatch sgd: Training imagenet in 1 hour.
\newblock \emph{arXiv preprint arXiv:1706.02677}, 2017.

\bibitem[Gur-Ari et~al.(2018)Gur-Ari, Roberts, and Dyer]{gur2018gradient}
Guy Gur-Ari, Daniel~A Roberts, and Ethan Dyer.
\newblock Gradient descent happens in a tiny subspace.
\newblock \emph{arXiv preprint arXiv:1812.04754}, 2018.

\bibitem[He et~al.(2016)He, Zhang, Ren, and Sun]{he2016deep}
Kaiming He, Xiangyu Zhang, Shaoqing Ren, and Jian Sun.
\newblock Deep residual learning for image recognition.
\newblock In \emph{Proceedings of the IEEE conference on computer vision and
  pattern recognition}, pages 770--778, 2016.

\bibitem[Hinton et~al.(2012)Hinton, Srivastava, and Swersky]{hinton2012neural}
Geoffrey Hinton, Nitish Srivastava, and Kevin Swersky.
\newblock Neural networks for machine learning lecture 6a overview of
  mini-batch gradient descent.
\newblock \emph{Cited on}, 14\penalty0 (8), 2012.

\bibitem[Hochreiter and Schmidhuber(1997{\natexlab{a}})]{hochreiter1997flat}
Sepp Hochreiter and J{\"u}rgen Schmidhuber.
\newblock Flat minima.
\newblock \emph{Neural Computation}, 9\penalty0 (1):\penalty0 1--42,
  1997{\natexlab{a}}.

\bibitem[Hochreiter and Schmidhuber(1997{\natexlab{b}})]{hochreiter1997long}
Sepp Hochreiter and J{\"u}rgen Schmidhuber.
\newblock Long short-term memory.
\newblock \emph{Neural computation}, 9\penalty0 (8):\penalty0 1735--1780,
  1997{\natexlab{b}}.

\bibitem[Huang et~al.(2017)Huang, Liu, Van Der~Maaten, and
  Weinberger]{huang2017densely}
Gao Huang, Zhuang Liu, Laurens Van Der~Maaten, and Kilian~Q Weinberger.
\newblock Densely connected convolutional networks.
\newblock In \emph{Proceedings of the IEEE conference on computer vision and
  pattern recognition}, pages 4700--4708, 2017.

\bibitem[Ioffe and Szegedy(2015)]{ioffe2015batch}
Sergey Ioffe and Christian Szegedy.
\newblock Batch normalization: Accelerating deep network training by reducing
  internal covariate shift.
\newblock \emph{arXiv preprint arXiv:1502.03167}, 2015.

\bibitem[Jacot et~al.(2018)Jacot, Gabriel, and Hongler]{jacot2018neural}
Arthur Jacot, Franck Gabriel, and Cl{\'e}ment Hongler.
\newblock Neural tangent kernel: Convergence and generalization in neural
  networks.
\newblock In \emph{Advances in neural information processing systems}, pages
  8571--8580, 2018.

\bibitem[Jastrzebski et~al.(2020)Jastrzebski, Szymczak, Fort, Arpit, Tabor,
  Cho*, and Geras*]{Jastrzebski2020The}
Stanislaw Jastrzebski, Maciej Szymczak, Stanislav Fort, Devansh Arpit, Jacek
  Tabor, Kyunghyun Cho*, and Krzysztof Geras*.
\newblock The break-even point on optimization trajectories of deep neural
  networks.
\newblock In \emph{International Conference on Learning Representations}, 2020.
\newblock URL \url{https://openreview.net/forum?id=r1g87C4KwB}.

\bibitem[Kingma and Ba(2014)]{kingma2014adam}
Diederik~P Kingma and Jimmy Ba.
\newblock Adam: {A} method for stochastic optimization.
\newblock \emph{arXiv preprint arXiv:1412.6980}, 2014.

\bibitem[Krizhevsky et~al.(2012)Krizhevsky, Sutskever, and
  Hinton]{krizhevsky2012imagenet}
Alex Krizhevsky, Ilya Sutskever, and Geoffrey~E Hinton.
\newblock Imagenet classification with deep convolutional neural networks.
\newblock \emph{Advances in neural information processing systems},
  25:\penalty0 1097--1105, 2012.

\bibitem[Krizhevsky et~al.(2009)]{krizhevsky2009learning}
Alex Krizhevsky et~al.
\newblock Learning multiple layers of features from tiny images.
\newblock 2009.

\bibitem[Kuditipudi et~al.(2019)Kuditipudi, Wang, Lee, Zhang, Li, Hu, Ge, and
  Arora]{kuditipudi2019explaining}
Rohith Kuditipudi, Xiang Wang, Holden Lee, Yi~Zhang, Zhiyuan Li, Wei Hu, Rong
  Ge, and Sanjeev Arora.
\newblock Explaining landscape connectivity of low-cost solutions for
  multilayer nets.
\newblock In \emph{Advances in Neural Information Processing Systems}, pages
  14574--14583, 2019.

\bibitem[Kunin et~al.(2019)Kunin, Bloom, Goeva, and Seed]{pmlr-v97-kunin19a}
Daniel Kunin, Jonathan Bloom, Aleksandrina Goeva, and Cotton Seed.
\newblock Loss landscapes of regularized linear autoencoders.
\newblock In Kamalika Chaudhuri and Ruslan Salakhutdinov, editors,
  \emph{Proceedings of the 36th International Conference on Machine Learning},
  volume~97 of \emph{Proceedings of Machine Learning Research}, pages
  3560--3569. PMLR, 09--15 Jun 2019.
\newblock URL \url{http://proceedings.mlr.press/v97/kunin19a.html}.

\bibitem[LeCun et~al.(1989)LeCun, Boser, Denker, Henderson, Howard, Hubbard,
  and Jackel]{lecun1989backpropagation}
Yann LeCun, Bernhard Boser, John~S Denker, Donnie Henderson, Richard~E Howard,
  Wayne Hubbard, and Lawrence~D Jackel.
\newblock Backpropagation applied to handwritten zip code recognition.
\newblock \emph{Neural computation}, 1\penalty0 (4):\penalty0 541--551, 1989.

\bibitem[LeCun et~al.(2010)LeCun, Cortes, and Burges]{lecun2010mnist}
Yann LeCun, Corinna Cortes, and CJ~Burges.
\newblock {MNIST} handwritten digit database.
\newblock \emph{ATT Labs [Online]. Available:
  http://yann.lecun.com/exdb/mnist}, 2, 2010.

\bibitem[Lee et~al.(2019)Lee, Xiao, Schoenholz, Bahri, Sohl-Dickstein, and
  Pennington]{lee2019wide}
Jaehoon Lee, Lechao Xiao, Samuel~S Schoenholz, Yasaman Bahri, Jascha
  Sohl-Dickstein, and Jeffrey Pennington.
\newblock Wide neural networks of any depth evolve as linear models under
  gradient descent.
\newblock \emph{arXiv preprint arXiv:1902.06720}, 2019.

\bibitem[Lee(2006)]{lee2006riemannian}
John~M Lee.
\newblock \emph{Riemannian manifolds: an introduction to curvature}, volume
  176.
\newblock Springer Science \& Business Media, 2006.

\bibitem[Lewkowycz et~al.(2020)Lewkowycz, Bahri, Dyer, Sohl-Dickstein, and
  Gur-Ari]{lewkowycz2020large}
Aitor Lewkowycz, Yasaman Bahri, Ethan Dyer, Jascha Sohl-Dickstein, and Guy
  Gur-Ari.
\newblock The large learning rate phase of deep learning: the catapult
  mechanism.
\newblock \emph{arXiv preprint arXiv:2003.02218}, 2020.

\bibitem[Li et~al.(2018)Li, Farkhoor, Liu, and Yosinski]{li2018measuring}
Chunyuan Li, Heerad Farkhoor, Rosanne Liu, and Jason Yosinski.
\newblock Measuring the intrinsic dimension of objective landscapes.
\newblock \emph{arXiv preprint arXiv:1804.08838}, 2018.

\bibitem[Liu et~al.(2020)Liu, Zhu, and Belkin]{liu2020linearity}
Chaoyue Liu, Libin Zhu, and Mikhail Belkin.
\newblock On the linearity of large non-linear models: when and why the tangent
  kernel is constant.
\newblock \emph{Advances in Neural Information Processing Systems}, 33, 2020.

\bibitem[Liu et~al.(2019)Liu, Ott, Goyal, Du, Joshi, Chen, Levy, Lewis,
  Zettlemoyer, and Stoyanov]{liu2019roberta}
Yinhan Liu, Myle Ott, Naman Goyal, Jingfei Du, Mandar Joshi, Danqi Chen, Omer
  Levy, Mike Lewis, Luke Zettlemoyer, and Veselin Stoyanov.
\newblock Roberta: A robustly optimized bert pretraining approach.
\newblock \emph{arXiv preprint arXiv:1907.11692}, 2019.

\bibitem[Martens and Grosse(2015)]{martens2015optimizing}
James Martens and Roger Grosse.
\newblock Optimizing neural networks with kronecker-factored approximate
  curvature.
\newblock In \emph{International conference on machine learning}, pages
  2408--2417. PMLR, 2015.

\bibitem[Matthews et~al.(2018)Matthews, Rowland, Hron, Turner, and
  Ghahramani]{matthews2018gaussian}
Alexander G de~G Matthews, Mark Rowland, Jiri Hron, Richard~E Turner, and
  Zoubin Ghahramani.
\newblock Gaussian process behaviour in wide deep neural networks.
\newblock \emph{arXiv preprint arXiv:1804.11271}, 2018.

\bibitem[Nguyen(2019)]{nguyen2019connected}
Quynh Nguyen.
\newblock On connected sublevel sets in deep learning.
\newblock \emph{arXiv preprint arXiv:1901.07417}, 2019.

\bibitem[Papyan(2020)]{papyan2020traces}
Vardan Papyan.
\newblock Traces of class/cross-class structure pervade deep learning spectra.
\newblock \emph{Journal of Machine Learning Research}, 21\penalty0
  (252):\penalty0 1--64, 2020.

\bibitem[Poole et~al.(2016)Poole, Lahiri, Raghu, Sohl-Dickstein, and
  Ganguli]{poole2016exponential}
Ben Poole, Subhaneil Lahiri, Maithra Raghu, Jascha Sohl-Dickstein, and Surya
  Ganguli.
\newblock Exponential expressivity in deep neural networks through transient
  chaos.
\newblock \emph{arXiv preprint arXiv:1606.05340}, 2016.

\bibitem[Saxe et~al.(2019)Saxe, McClelland, and Ganguli]{saxe2019mathematical}
Andrew~M. Saxe, James~L. McClelland, and Surya Ganguli.
\newblock A mathematical theory of semantic development in deep neural
  networks.
\newblock \emph{Proceedings of the National Academy of Sciences}, 116\penalty0
  (23):\penalty0 11537--11546, 2019.

\bibitem[Schaul et~al.(2013)Schaul, Zhang, and LeCun]{schaul2013no}
Tom Schaul, Sixin Zhang, and Yann LeCun.
\newblock No more pesky learning rates.
\newblock In \emph{International Conference on Machine Learning}, pages
  343--351, 2013.

\bibitem[Shevchenko and Mondelli(2019)]{shevchenko2019landscape}
Alexander Shevchenko and Marco Mondelli.
\newblock Landscape connectivity and dropout stability of sgd solutions for
  over-parameterized neural networks.
\newblock \emph{arXiv preprint arXiv:1912.10095}, 2019.

\bibitem[Simonyan and Zisserman(2014)]{simonyan2014very}
Karen Simonyan and Andrew Zisserman.
\newblock Very deep convolutional networks for large-scale image recognition.
\newblock \emph{arXiv preprint arXiv:1409.1556}, 2014.

\bibitem[Sun(2019)]{sun2019optimization}
Ruoyu Sun.
\newblock Optimization for deep learning: theory and algorithms.
\newblock \emph{arXiv preprint arXiv:1912.08957}, 2019.

\bibitem[Vaswani et~al.(2017)Vaswani, Shazeer, Parmar, Uszkoreit, Jones, Gomez,
  Kaiser, and Polosukhin]{vaswani2017attention}
Ashish Vaswani, Noam Shazeer, Niki Parmar, Jakob Uszkoreit, Llion Jones,
  Aidan~N Gomez, {\L}ukasz Kaiser, and Illia Polosukhin.
\newblock Attention is all you need.
\newblock In \emph{Advances in neural information processing systems}, pages
  5998--6008, 2017.

\bibitem[Wolf et~al.(2020)Wolf, Debut, Sanh, Chaumond, Delangue, Moi, Cistac,
  Rault, Louf, Funtowicz, Davison, Shleifer, von Platen, Ma, Jernite, Plu, Xu,
  Scao, Gugger, Drame, Lhoest, and Rush]{wolf-etal-2020-transformers}
Thomas Wolf, Lysandre Debut, Victor Sanh, Julien Chaumond, Clement Delangue,
  Anthony Moi, Pierric Cistac, Tim Rault, Rémi Louf, Morgan Funtowicz, Joe
  Davison, Sam Shleifer, Patrick von Platen, Clara Ma, Yacine Jernite, Julien
  Plu, Canwen Xu, Teven~Le Scao, Sylvain Gugger, Mariama Drame, Quentin Lhoest,
  and Alexander~M. Rush.
\newblock Transformers: State-of-the-art natural language processing.
\newblock In \emph{Proceedings of the 2020 Conference on Empirical Methods in
  Natural Language Processing: System Demonstrations}, pages 38--45, Online,
  October 2020. Association for Computational Linguistics.
\newblock URL \url{https://www.aclweb.org/anthology/2020.emnlp-demos.6}.

\bibitem[Wu et~al.(2018)Wu, Ren, Liao, and Grosse]{wu2018understanding}
Yuhuai Wu, Mengye Ren, Renjie Liao, and Roger Grosse.
\newblock Understanding short-horizon bias in stochastic meta-optimization.
\newblock \emph{arXiv preprint arXiv:1803.02021}, 2018.

\bibitem[Xiao et~al.(2017)Xiao, Rasul, and Vollgraf]{xiao2017fashion}
Han Xiao, Kashif Rasul, and Roland Vollgraf.
\newblock Fashion-{MNIST}: {A} novel image dataset for benchmarking machine
  learning algorithms.
\newblock \emph{arXiv preprint arXiv:1708.07747}, 2017.

\bibitem[Zhang et~al.(2019{\natexlab{a}})Zhang, Li, Nado, Martens, Sachdeva,
  Dahl, Shallue, and Grosse]{zhang2019algorithmic}
Guodong Zhang, Lala Li, Zachary Nado, James Martens, Sushant Sachdeva, George
  Dahl, Chris Shallue, and Roger~B Grosse.
\newblock Which algorithmic choices matter at which batch sizes? insights from
  a noisy quadratic model.
\newblock In \emph{Advances in Neural Information Processing Systems}, pages
  8196--8207, 2019{\natexlab{a}}.

\bibitem[Zhang et~al.(2019{\natexlab{b}})Zhang, Dauphin, and
  Ma]{zhang2019fixup}
Hongyi Zhang, Yann~N Dauphin, and Tengyu Ma.
\newblock Fixup initialization: Residual learning without normalization.
\newblock \emph{arXiv preprint arXiv:1901.09321}, 2019{\natexlab{b}}.

\end{thebibliography}
